\documentclass[11pt]{article}
\usepackage[margin=1in]{geometry}
\usepackage{amsmath,amssymb,amsthm,mathtools}
\usepackage{booktabs,tabularx,array,multirow}
\usepackage{graphicx}
\usepackage{enumitem}
\usepackage{float}
\usepackage{placeins}
\usepackage{microtype}
\usepackage{xcolor}
\usepackage{algorithm}
\usepackage{algpseudocode}
\usepackage{authblk}
\usepackage[numbers,compress]{natbib}
\usepackage[hidelinks]{hyperref}
\usepackage{xr-hyper}
\usepackage{doi}

\hypersetup{
  pdftitle={Shared Physics Responses Recover Hidden Rankings in Neural Operator Libraries},
  pdfauthor={Hanbing Liang; Fujun Liu},
  pdfsubject={Physics-based ranking and selection in finite neural-operator libraries},
  pdfkeywords={neural operators, partial differential equations, model selection, ranking, certification}
}


\newtheorem{theorem}{Theorem}
\newtheorem{proposition}{Proposition}

\theoremstyle{definition}

\theoremstyle{remark}

\newcommand{\Hcal}{\mathcal H}
\newcommand{\Dcal}{\mathcal D}
\newcommand{\Acal}{\mathcal A}

\newcommand{\Rcal}{\mathcal R}
\newcommand{\norm}[1]{\left\lVert #1\right\rVert}

\newcommand{\metricnorm}[1]{\left\lVert #1\right\rVert_M}
\newcommand{\metricip}[2]{\left\langle #1,#2\right\rangle_M}

\title{Shared Physics Responses Recover Hidden Rankings in Neural Operator Libraries}
\author[1]{Hanbing Liang}
\author[1,*]{Fujun Liu}
\affil[1]{Nanophotonics and Biophotonics Key Laboratory of Jilin Province,
School of Physics, Changchun University of Science and Technology,
Changchun 130022, P.R. China}
\affil[*]{Correspondence: \href{mailto:fjliu@cust.edu.cn}{fjliu@cust.edu.cn}}
\date{}

\begin{document}
\maketitle

\begin{abstract}
Selecting the optimal neural-operator prediction during deployment is challenging when high-fidelity reference solutions are unavailable. We demonstrate that under a squared Hilbert-space loss, ranking a finite model library depends strictly on the low-dimensional span of candidate differences, allowing us to score all models simultaneously using a single anchor-based linearized response of the governing equation. This shared physical diagnostic accurately recovered over 99.6\% of pairwise preferences and 99.0\% of optimal checkpoints across diverse Fourier and convolutional operator libraries for fluid, reaction-diffusion, and wave dynamics. Furthermore, the corrected physical proxy frequently outperformed the best individual candidates, and we establish computable sufficient conditions that rigorously certify exact decisions for strongly monotone discretizations. By exploiting the local dynamical response rather than raw defect magnitude, this framework enables the reliable and highly efficient deployment of scientific surrogates without requiring ground-truth data.
\end{abstract}

% --- Begin inlined file: sections/01_introduction.tex ---
\section{Introduction}
\label{sec:introduction}

Partial differential equations (PDEs) govern fundamental physical phenomena, yet relying exclusively on high-fidelity numerical solvers imposes severe computational bottlenecks. Machine-learning surrogates mitigate this burden by learning mappings from problem inputs to solution fields across diverse families of physical systems~\citep{LiEtAl2021FNO,LuEtAl2021DeepONet,KovachkiEtAl2023,RaonicEtAl2023CNO}. While different architectures are continuously evaluated across public benchmarks~\citep{TakamotoEtAl2022PDEBench,GuptaBrandstetter2023,OhanaEtAl2024}, a practical scientific workflow often generates multiple plausible predictions from various models, training runs, or checkpoints for the exact same input. Because the optimal candidate frequently changes depending on the specific physical input and the scientific quantity of interest, practitioners face a fundamental deployment dilemma.

During benchmarking, an available high-fidelity solution effortlessly ranks these competing candidates. During active deployment, however, this reference solution is inherently unavailable, and computing it merely to select a surrogate would entirely forfeit the computational advantage of using neural operators. The central question is therefore whether the governing physics can reliably recover the hidden ranking of a finite candidate library without first computing the missing reference.

Addressing this question requires recognizing how a finite candidate library alters the necessary information landscape. Under a fixed squared Hilbert-space task loss, subtracting the losses of two candidates eliminates every component of the unknown solution that cannot distinguish between them. All pairwise preferences subsequently depend exclusively on the span of candidate differences, whose dimension is at most the library size minus one. This mathematical reduction establishes that choosing among a finite set requires significantly less information than reconstructing an arbitrary physical solution, meaning a physical diagnostic only needs to recover this low-dimensional decision coordinate.

Existing diagnostic tools target related but fundamentally distinct objectives. Ground-truth-free model selection and validation errors typically rely on historical examples or ensemble statistics rather than reconstructing the specific ordering induced by the governing physics for the current input~\citep{WeiEtAl2023PINNSelection,GargEtAl2022,GuilloryEtAl2021}. Furthermore, while the raw equation residual measures a physical defect, its magnitude alone fails to capture how that defect propagates into the requested output. Physics-informed correction methods use the governing equations to improve individual predictions~\citep{RaissiEtAl2019,KarniadakisEtAl2021,LiEtAl2024PINO,CaoEtAl2023,Jha2024,HuangPerdikaris2026PhysicsCorrect}, while a posteriori estimators translate residuals into global error bounds~\citep{PrudhommeOden1999,OdenPrudhomme2001,BeckerRannacher2001,GilesSuli2002,HesthavenEtAl2016,FanaskovEtAl2024}, and complementary verification methods place rigorous bounds over the continuous domain~\citep{EirasEtAl2024}. In contrast, our objective is the relative ordering of a fixed finite library. By leveraging the candidate-difference quotient, our framework allows a single anchor-derived response to be reused across the entire library, conditioning the scores directly on the current physical input and evaluating all contrasts simultaneously.

We formulate this paradigm as reference-free finite-library ranking, where ``reference-free'' specifically denotes that the scoring process utilizes no deployment reference solution or reference-derived labels. If the hidden high-fidelity solution were accessible, the chosen scientific loss would cleanly induce pairwise preferences and identify the best candidate. Our goal is to recover those precise decisions relying solely on the input, the candidate predictions, and the governing PDE system. This operational challenge spans two distinct deployment scales: query-time routing to select an optimal candidate for a single input, and panel-level aggregation to freeze a specific checkpoint for repeated future use.

To solve this ranking problem, our approach distills these decisions into a single shared physics calculation. The candidate library first defines a shared library anchor, at which we evaluate the complete physical defect before propagating it once through a linearized physics solve. Mapping this corrected state to the scientific task yields a shared task proxy, denoted by $q$, which subsequently estimates the hidden reference ordering by measuring its distance to every candidate. Although the numerical implementation computes a full-state correction, our theory explicitly identifies the smaller projection of this response that dictates library decisions. Unlike candidate-wise correction strategies, this shared framework avoids solving a separate physics problem for every prediction while remaining deeply task-aware.

We theoretically prove that this shared comparison coordinate is both sufficient and minimal among bounded linear observations for globally exact ordering, linking it to the physics response through an exact conditional error identity. Because ranking reliability intrinsically depends on the candidate margin, we demonstrate that a decision is trustworthy when the remaining physical uncertainty is strictly bounded by the separation between candidates. For a strongly monotone reaction-diffusion operator, this geometric principle translates into computable sufficient conditions that rigorously certify same-grid pairwise or top-1 decisions. We validate this framework comprehensively across separately trained Fourier and convolutional neural-operator libraries for Burgers and reaction-diffusion equations under controlled input shifts. Notably, in an operator-mismatched two-dimensional compressible-flow probe, the shared proxy recovered all 10/10 candidate comparisons although it outperformed the best neural-operator prediction in only 4/10 cases, showing that accurate ranking can persist without superior reconstruction or exact operator--data alignment. By extending our analysis to Sine-Gordon dynamics and public PDEBench models, we systematically distinguish ranking from direct proxy use, confirming that governing-physics responses can successfully replace missing ground truth in scientific surrogate deployment.
% --- End inlined file: sections/01_introduction.tex ---
\section{Results}
\label{sec:results}

When several surrogate predictions disagree, deployment provides the candidate
fields and governing equation but not the high-fidelity solution needed to rank
them. A finite library changes what must be recovered from that missing
solution. Under a fixed squared Hilbert-space task loss, only target directions spanned by
candidate differences can change their relative errors. This comparison space
has dimension at most the library size minus one, even when the physical state
is very high-dimensional. Ranking can therefore require much less information
than reconstructing the hidden solution.

We exploit this reduction by constructing a library anchor $w$ from the
candidates and applying one linearized physics correction. Mapping the
corrected state to the scientific task gives a proxy $q$, which is compared
with every candidate output (Fig.~\ref{fig:decision-framework}a--c). These
shared scores select a prediction for the current input and can also be
averaged over an unlabeled panel to rank reusable models.

\begin{figure}[t]
  \centering
  \includegraphics[width=\linewidth]{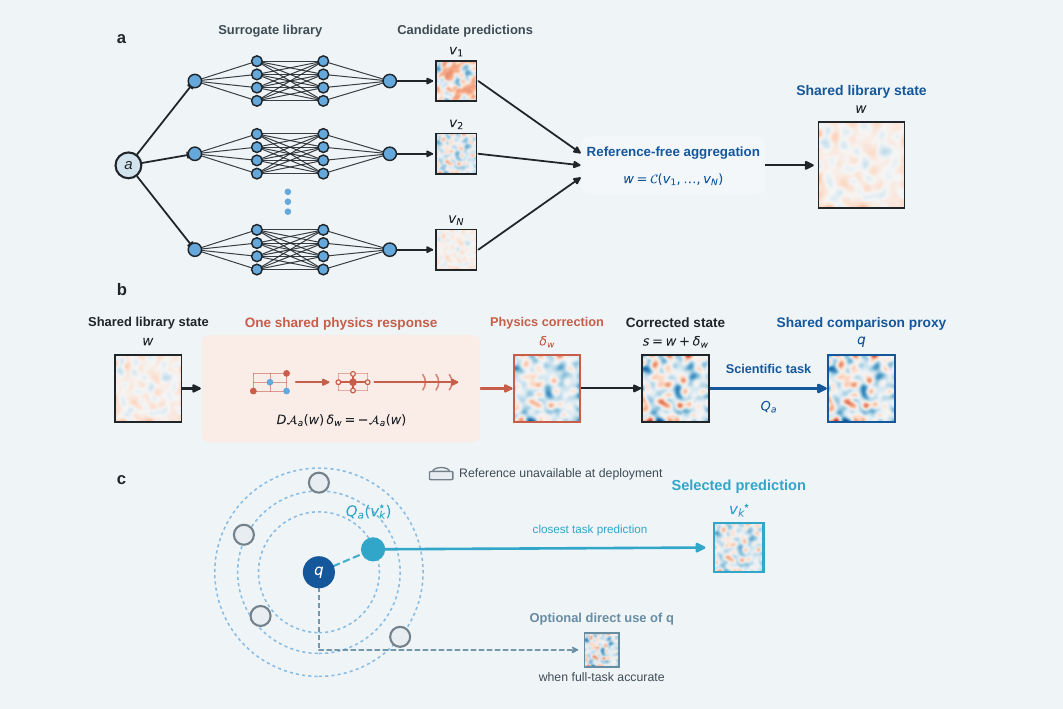}
  \caption{\textbf{One shared physics response turns a finite surrogate
  library into a reference-free decision.}
  \textbf{a}, Candidate states $v_1,\ldots,v_N$ are aggregated by a
  reference-free rule $\mathcal C$ into the library anchor $w$.
  \textbf{b}, One linearized physics response gives the corrected state
  $s=w+\delta_w$, which the task map $Q_a$ converts to the proxy $q$.
  \textbf{c}, Distances
  $S_i(q)=\lVert q-Q_a(v_i)\rVert_M^2$ rank the candidates and select
  $v_{k^\star}$; the dashed branch denotes optional direct use of $q$.}
  \label{fig:decision-framework}
\end{figure}

Theorem~\ref{thm:decision-quotient} below makes the information reduction
exact. The remaining questions are whether one physical response can recover
the required coordinate in learned surrogate libraries, which features of the
library make recovery reliable, and when the resulting choice can be certified
without revealing the reference.

% --- Begin inlined file: sections/02_results_decision_geometry.tex ---
\subsection{Ranking a finite library requires less information than solution recovery}
\label{sec:decision-geometry}

The framework treats ranking as its primary output rather than as a by-product
of reconstructing the hidden solution. This distinction is exact for a fixed
library: only the parts of the target that change the candidates' relative
errors can affect their ordering. We now formalize that smaller information
requirement.

Let $(\Hcal,\metricip{\cdot}{\cdot})$ be a real Hilbert task space,
let $y_1,\ldots,y_N\in\Hcal$ be fixed candidate outputs and let
$t\in\Hcal$ denote the unavailable task truth. Candidate $i$ has squared loss
\begin{equation}
  R_i(t)=\metricnorm{t-y_i}^{2}.
  \label{eq:decision-loss}
\end{equation}
For an ordered pair, define $G_{ij}(t)=R_j(t)-R_i(t)$, so that
$G_{ij}(t)>0$ favours candidate $i$. Expanding the two losses gives
\begin{equation}
  G_{ij}(t)
  =\metricnorm{y_j}^{2}-\metricnorm{y_i}^{2}
  +2\metricip{t}{y_i-y_j}.
  \label{eq:finite-library-gap}
\end{equation}
Thus the hidden target enters every pairwise decision only through inner
products with candidate differences. Define
\begin{equation}
  \Dcal=\operatorname{span}\{y_i-y_1:2\leq i\leq N\}.
  \label{eq:comparison-subspace}
\end{equation}
Targets that differ by an element of $\Dcal^{\perp_M}$ belong to the same
decision class; the quotient $\Hcal/\Dcal^{\perp_M}$ is canonically
represented by $P_{\Dcal}^{M}t\in\Dcal$.
The same collapse extends to general bounded linear observations.

\begin{theorem}[Finite-library decision quotient and minimal bounded-linear observation]
\label{thm:decision-quotient}
The coordinate $P_{\Dcal}^{M}t$ determines every exact pair gap, every
three-way pair label (candidate $i$, tie or candidate $j$) and the complete weak ordering, and
$\dim\Dcal\leq N-1$. It also determines top-1 after fixing a candidate-index
tie rule that depends only on the minimizer set.

Let $W$ be a real normed space and let $T:\Hcal\to W$ be bounded and linear.
The exact gap vector, the three-way
pair-label vector and the complete weak ordering are determined by $Tt$ for
every $t\in\Hcal$ if and only if
\begin{equation}
  \ker T\subseteq\Dcal^{\perp_M}.
  \label{eq:decision-kernel-condition}
\end{equation}
Let $\tau$ choose an element $\tau(S)\in S$ from every nonempty candidate subset
$S$ and define $\iota_\tau(t)=\tau(\arg\min_i R_i(t))$. For this fixed minimizer-set tie
rule, the same equivalence holds for globally exact recovery of
$\iota_\tau(t)$ when targets range over all of $\Hcal$. If $W$ is
finite-dimensional, any such observation satisfies
\begin{equation}
  \dim W\geq\dim\Dcal.
  \label{eq:decision-dimension-lower-bound}
\end{equation}
\end{theorem}

\begin{proof}
Equation~\eqref{eq:finite-library-gap} proves sufficiency. For necessity,
suppose that $h\in\ker T$ but $h\notin\Dcal^{\perp_M}$. At least one
candidate difference has nonzero inner product with $h$. Perturbing the
midpoint of that pair by $\pm\varepsilon h$ produces opposite strict signs
with the same observation, which rules out exact sign recovery and hence exact
gap or complete-order recovery. For top-1, set
$a_i=\metricip{h}{y_i}$. The $a_i$ are not all equal. For sufficiently large
positive $\alpha$, every minimizer at $t=\alpha h$ lies in the
maximum-projection candidate set; for sufficiently large negative $\alpha$,
every minimizer lies in the disjoint minimum-projection set. Both targets have
the same observation. Therefore $T|_{\Dcal}$ must be injective, which also
gives Eq.~\eqref{eq:decision-dimension-lower-bound}.
\end{proof}

Theorem~\ref{thm:decision-quotient} establishes minimality among bounded linear
observations for globally exact decisions over all targets. On a restricted physical solution
set, some candidate boundaries may never be exposed, and exact gaps, pair signs
and top-1 can require different amounts of information. The corresponding
exposure conditions, duplicate-candidate cases and a continuous-statistic
dimension result are given in Supplementary Section S1.

The theorem isolates the target information that a library decision can use.
A physical solver may still return a full-state response, but the library reads
only its comparison projection. The difficulty of a choice is then governed by
the margin between competing candidates. The empirical libraries contained
both ambiguous and well-separated decisions: among 1,920 input--library
combinations, 336 had a best--runner-up gap no larger than $10^{-6}$, whereas
the remaining gaps spanned a broad range (Supplementary Section S9).

Equation~\eqref{eq:finite-library-gap} also identifies the relevant notion of
error. A large state error orthogonal to $\Dcal$ is invisible to every
squared-loss decision, whereas a small error aligned with a near-tie candidate
difference can reverse a ranking. The next question is therefore whether one
reference-free physical response can recover this decision coordinate relative
to the margins found in actual libraries.
% --- End inlined file: sections/02_results_decision_geometry.tex ---
\FloatBarrier
% --- Begin inlined file: sections/04_results_decision_recovery.tex ---
\subsection{A shared physics response recovers hidden neural-operator choices}
\label{sec:decision-recovery}

Theorem~\ref{thm:decision-quotient} identifies what must be known to rank a
finite library. We next construct a shared physical response and show how its
error enters precisely that comparison coordinate. For a physical input $a$,
let $v_1,\ldots,v_N$ be the candidate states. A
reference-free library rule $\mathcal C$ defines an
anchor
\begin{equation}
  w=\mathcal C(v_1,\ldots,v_N).
  \label{eq:anchor}
\end{equation}
Let $\Acal_a(v)=0$ denote the assembled discrete physical problem, including
its initial and boundary conditions. With
$L_w=D\Acal_a(w)$, we compute one inexact shared response
\begin{equation}
  L_w\delta_w=-\Acal_a(w)+\ell,
  \label{eq:shared-response}
\end{equation}
where $\ell$ records the algebraic solve defect. The corrected state is
$s=w+\delta_w$. For the principal linear tasks, $q=Q_a(s)$ and
$y_i=Q_a(v_i)$; every candidate is scored against the same task proxy,
\begin{equation}
  S_i(q)=\metricnorm{q-y_i}^{2}.
  \label{eq:proxy-score}
\end{equation}
No numerical reference enters Eqs.~\eqref{eq:anchor}--\eqref{eq:proxy-score}.

The role of the shared solve is given by an exact conditional representation.
Let $u$ solve the governing equation, set $e=u-w$, and define
\begin{equation}
  \Rcal_A(e)=\Acal_a(w+e)-\Acal_a(w)-L_we.
  \label{eq:physics-remainder}
\end{equation}

\begin{proposition}[Conditional physics-to-decision error representation]
\label{prop:physics-decision-representation}
Let $X$, $Y$ and $\Hcal$ be real Hilbert spaces, let
$\Acal_a:\Omega\subset X\to Y$ and $Q_a:\Omega\to\Hcal$, and suppose
$\Acal_a(u)=0$. Assume that $\Acal_a$ and $Q_a$ are Fr\'echet differentiable
where used, that $L_w:X\to Y$ is a bounded isomorphism, and that the required
line segments lie in $\Omega$. Then
\begin{equation}
  \delta_w-e=L_w^{-1}\bigl[\ell+\Rcal_A(e)\bigr].
  \label{eq:shared-correction-error}
\end{equation}
If $Q_a$ is nonlinear, define
\begin{equation}
  q_{\rm lin}=Q_a(w)+DQ_a(w)\delta_w,
  \qquad
  q_{\rm eval}=Q_a(w+\delta_w),
  \label{eq:linear-evaluated-proxies}
\end{equation}
and
$\Rcal_Q(h)=Q_a(w+h)-Q_a(w)-DQ_a(w)h$. Then
\begin{align}
  P_{\Dcal}^{M}(q_{\rm lin}-t)
  &=P_{\Dcal}^{M}DQ_a(w)L_w^{-1}
    \bigl[\ell+\Rcal_A(e)\bigr]
    -P_{\Dcal}^{M}\Rcal_Q(e),
  \label{eq:q-linear-error}\\
  P_{\Dcal}^{M}(q_{\rm eval}-t)
  &=P_{\Dcal}^{M}DQ_a(w)L_w^{-1}
    \bigl[\ell+\Rcal_A(e)\bigr]
    +P_{\Dcal}^{M}\bigl[\Rcal_Q(\delta_w)-\Rcal_Q(e)\bigr],
  \label{eq:q-evaluated-error}
\end{align}
where $t=Q_a(u)$. If $Q_a$ is bounded and linear, the two proxies coincide and
\begin{equation}
  P_{\Dcal}^{M}(q-t)
  =P_{\Dcal}^{M}Q_aL_w^{-1}\bigl[\ell+\Rcal_A(e)\bigr].
  \label{eq:linear-task-projected-error}
\end{equation}
If both derivatives are locally H\"older continuous with the same exponent
$0<\alpha\leq1$ along the required segments, these identities yield local
bounds in $\|\ell\|$, $\|u-w\|^{1+\alpha}$ and, for an evaluated nonlinear task,
$\|\delta_w\|^{1+\alpha}$. For any resulting task proxy $q$,
\begin{equation}
  G_{ij}(t)-G_{ij}(q)
  =2\metricip{t-q}{y_i-y_j}
  \label{eq:task-gap-transfer}
\end{equation}
remains exact.
\end{proposition}

\begin{proof}
Substitute $\Acal_a(u)=0$ into the first-order expansion at $w$ and compare it
with Eq.~\eqref{eq:shared-response}; this gives
Eq.~\eqref{eq:shared-correction-error}. Expanding $Q_a$ at $w$ gives
Eqs.~\eqref{eq:q-linear-error}--\eqref{eq:q-evaluated-error}. The last identity
is Eq.~\eqref{eq:finite-library-gap} evaluated at $t$ and $q$.
\end{proof}

The unprojected identities and their complete algebraic derivation are given
in Supplementary Proposition 7.

These identities explain how the physical correction reaches the decision
coordinate. A computable reference-free certificate additionally requires a
way to bound the anchor-to-solution distance. Nonlinear $Q_a$ does not alter the
final squared-Hilbert comparison geometry; it adds a physics-to-task remainder.
A non-Hilbert or truth-normalized pair loss changes the comparison geometry
itself.

For linear tasks, the same structure has an adjoint interpretation. Define
\begin{equation}
  J_{ij}(x)=\metricnorm{Q_ax-y_j}^{2}-\metricnorm{Q_ax-y_i}^{2}.
\end{equation}
The quadratic terms in $Q_ax$ cancel. If
$L_w^*z_{ij}=2Q_a^*(y_i-y_j)$ and $\ell=0$, then the shared solve satisfies
\begin{equation}
  J_{ij}(w)-\langle z_{ij},\Acal_a(w)\rangle_Y
  =J_{ij}(w+\delta_w).
  \label{eq:one-primal-all-pair}
\end{equation}
One corrected primal state therefore realizes the full family of pairwise
first-order adjoint point estimates. The adjoint machinery is classical; the
finite-library structure is that the candidate set induces this affine family
of goals and one shared state evaluates all of them. Rigorous uncertainty still
requires localization or directional support bounds. For the quadratic
Burgers operator, the remainder is exactly bilinear:
\begin{equation}
  \Acal_a(w+e)=\Acal_a(w)+L_we+\mathbf B(e,e),
  \qquad
  \delta_w-e=L_w^{-1}\mathbf B(e,e)
  \label{eq:burgers-quadratic-remainder}
\end{equation}
for an exact solve. This gives a PDE-specific local quadratic realization of
Proposition~\ref{prop:physics-decision-representation}.

We tested this construction on eight libraries formed by crossing two
equations, two architectures and two seeded ensembles, evaluated under
controlled input shifts. One shared response selected a candidate within
$10^{-6}$ of the best library loss in 99.01\% of cases and recovered 99.59\% of
non-tied pairwise preferences (Fig.~\ref{fig:prospective-recovery}a,b).
Recovery remained high in all eight libraries and among the closest candidate
comparisons (Supplementary Section S9). The raw
residual recovered 67.83\% of non-tied pairs and 42.66\% of tolerant top-1
choices.

The ranking behaviour extended to Sine--Gordon. Across eight libraries and 240
inputs, the proxy recovered 95.57\% of exact top-1 decisions (Supplementary Table 6 and Supplementary Fig. 2).

\begin{figure}[t]
  \centering
  \includegraphics[width=\linewidth]{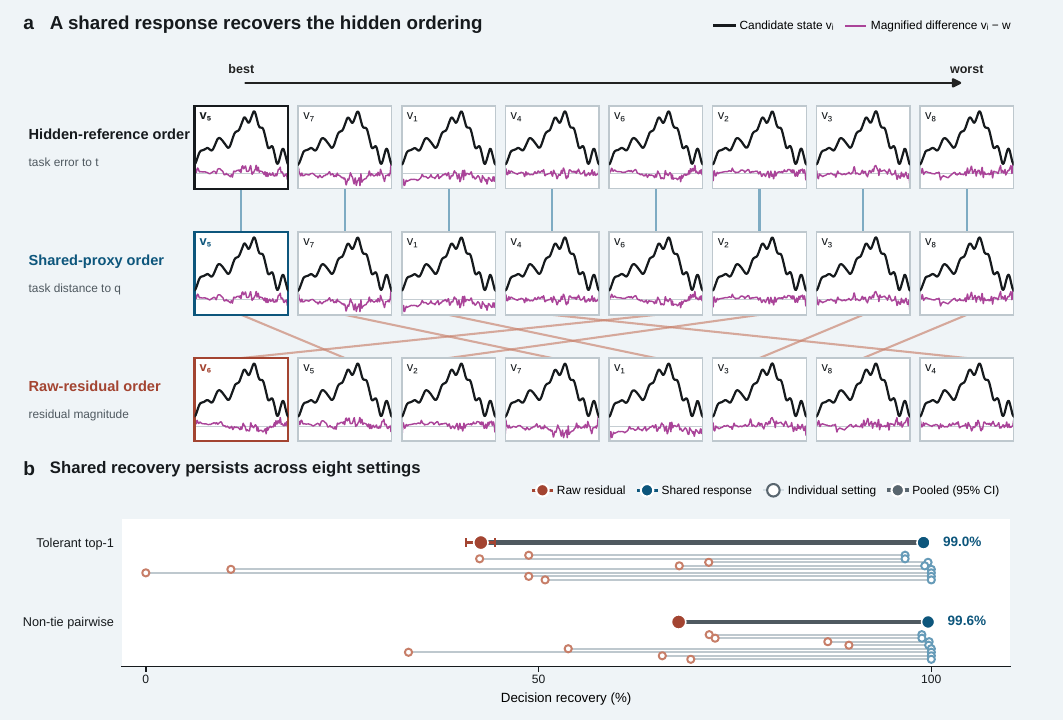}
  \caption{\textbf{One shared physics response recovers hidden
  neural-operator choices.} \textbf{a}, The same eight candidate states are
  shown in three horizontal lanes, ordered from best to worst by
  hidden-reference task error, distance to the shared proxy and
  assembled-defect energy. Lines connect candidate identities across lanes.
  Within each frame, the black curve shows $v_i$ on a common state scale,
  whereas the magenta curve shows $v_i-w$ about the grey zero line on a
  separate magnified scale shared by all candidates. \textbf{b}, Tolerant top-1
  and non-tie pairwise recovery in eight libraries. Thin segments join the
  raw-residual and shared-response results within each library; large markers
  show pooled recovery with input-clustered 95\% confidence intervals.}
  \label{fig:prospective-recovery}
\end{figure}

The gap geometry in Fig.~\ref{fig:ranking-quality}a shows what supplied the
signal. With the correct response, proxy gaps track hidden-reference gaps in
both sign and magnitude. Assigning the response to the wrong input disperses
this relation, while reversing its direction produces a predominantly reversed
gap pattern. The recovered preference therefore depended on the current
physical instance and on propagating its defect in the correct direction, not
simply on static library geometry.

\begin{figure}[!t]
  \centering
  \includegraphics[width=\linewidth]{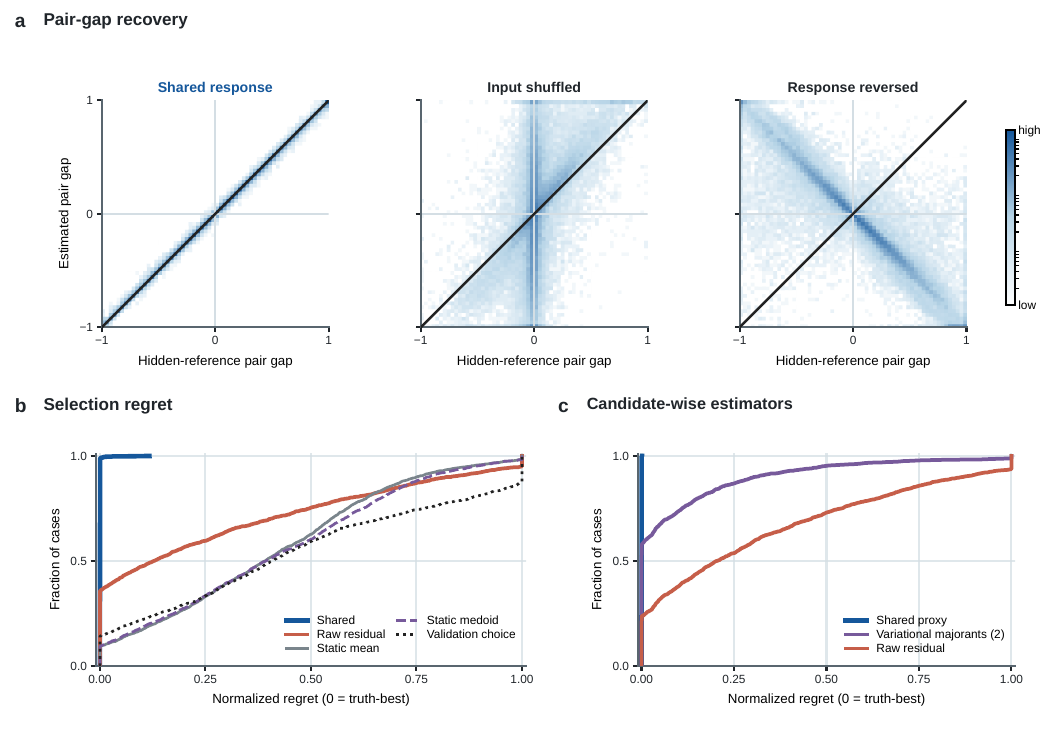}
  \caption{\textbf{Shared physics recovers candidate gaps and yields
  low-regret choices.} \textbf{a}, Hidden-reference pair gaps plotted against
  gaps estimated from the shared, input-shuffled and reversed responses
  (49,323 non-tie pairs; common normalized axes, bins and density scale).
  Matching signs indicate correct ordering; the diagonal indicates agreement
  in gap magnitude. \textbf{b}, Empirical cumulative distributions of
  span-normalized selection regret for 1,920 primary input--library cases.
  Zero denotes selection of the hidden-reference best candidate; curves nearer
  the upper left are better. \textbf{c}, The same regret measure for a separate
  reaction--diffusion population (1,920 cases), comparing the shared proxy,
  raw residual and two candidate-wise variational majorants; the two majorant
  curves nearly coincide.}
  \label{fig:ranking-quality}
\end{figure}

Top-1 recovery alone masks the severity of incorrect choices.
Across the 1,920 Burgers and reaction--diffusion input--library cases, the
shared proxy selected the exact winner in 1,899 cases; when it missed, its
normalized regret remained close to zero (Fig.~\ref{fig:ranking-quality}b).
Here regret is the selected candidate's excess hidden loss divided by the loss
span from the best to the worst library member. The comparison baselines select
by assembled-defect energy (the raw residual baseline), distance to the candidate mean or medoid, or
input-independent validation loss; all produced broader regret distributions.

We also compared the shared response with candidate-wise error estimators
(Fig.~\ref{fig:ranking-quality}c). A common-energy
dual-norm estimator and a candidate-adapted, operator-aware majorant both gave
valid a posteriori bounds and improved pairwise ranking to 82.7\% in the
reaction--diffusion cases. Both nevertheless remained below the shared proxy on
the same candidates. These estimators bound individual state errors, whereas
the shared proxy estimates the direction of a candidate comparison. In this
reaction--diffusion comparison, these
absolute-error estimators preserved the pairwise ordering less reliably than
the shared task proxy; detailed pairwise, top-1 and regret results are given in Supplementary Table 18.

Ranking candidates by the task-space magnitude of their own linearized
corrections produced nearly the same decisions as the shared construction.
In a separate 192-case control with 16 candidates per library, the shared
selector and the candidate-specific correction-magnitude selector chose the
same checkpoint in 191 cases and had nearly identical pairwise accuracy. Thus
one library-level response retained almost all of the decision information from
$N$ candidate-specific linearized error proxies while requiring only one
correction solve.
The difference in solve count was visible in the component timings: from
$N=2$ to $N=16$, shared-response time remained nearly constant, whereas
candidate-wise time grew approximately in proportion to library size. At
$N=16$, the candidate-wise component was 16.0 times slower for Burgers and 15.3
times slower for reaction--diffusion (Supplementary Table 12).
% --- End inlined file: sections/04_results_decision_recovery.tex ---
\FloatBarrier
% --- Begin inlined file: sections/05_results_anchor_replication.tex ---
\subsection{Reliable ranking depends on margins, library composition and admissibility}
\label{sec:anchor-replication}

High average accuracy can hide sensitivity to library composition. In a
reaction--diffusion CNO library, two extreme candidates pulled the arithmetic
anchor away from the main cluster and reduced exact top-1 recovery to 67 of 240
inputs (Fig.~\ref{fig:anchor-replication}a). For those inputs, a candidate-valued
medoid remained inside the main cluster and recovered all 240 choices, linking
the failure to the position of the anchor rather than to an uninformative
candidate library. Across 16 further libraries constructed from separate
training, validation and evaluation data, both anchors selected the exact winner
in 3,800 of 3,840 cases. Candidate composition therefore matters, but neither
anchor was uniformly better; the anchor rule must suit the intended candidate
family.

\begin{figure}[!t]
  \centering
  \includegraphics[width=\linewidth]{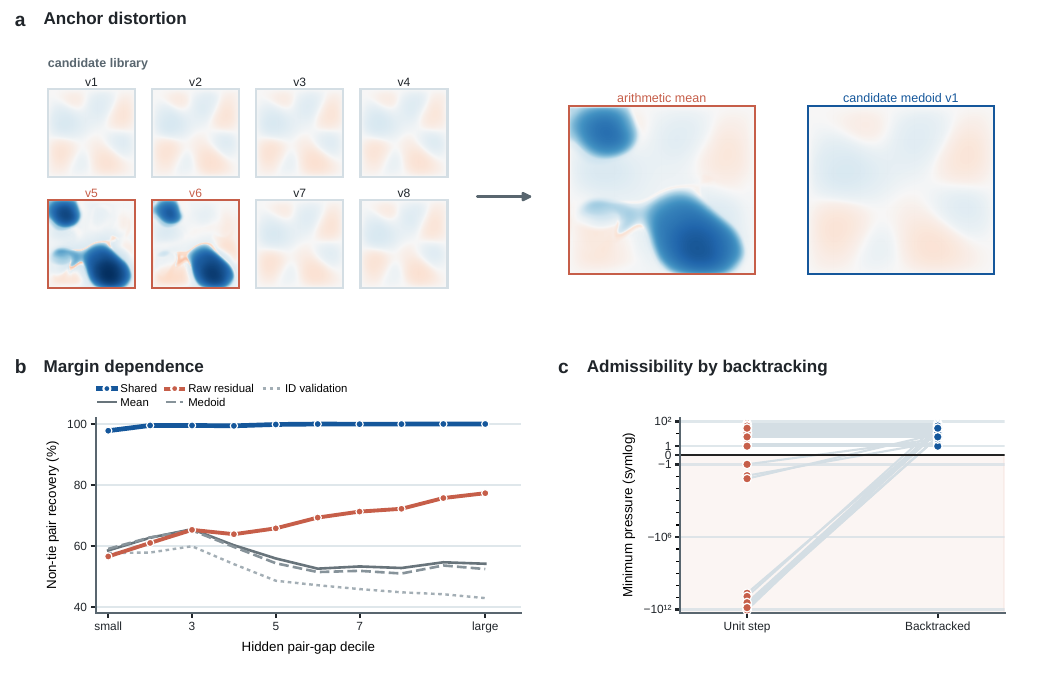}
  \caption{\textbf{Library composition, candidate separation and admissibility
  shape reliable recovery.}
  \textbf{a}, Eight candidate states, their arithmetic mean and medoid on the
  common signed-log scale
  $\operatorname{sgn}(u)\ln(1+|u|/0.05)$. Extreme states displace the mean,
  whereas the medoid remains candidate-valued. \textbf{b}, Non-tie pair
  recovery across hidden-gap deciles for the shared response, raw residual,
  validation choice, mean and medoid (49,323 pairs). \textbf{c}, Paired minimum
  pressures before and after backtracking for 35 shock cases; the horizontal
  zero line marks physical admissibility.}
  \label{fig:anchor-replication}
\end{figure}

Reliable ranking therefore requires the shared anchor to produce an informative
physical response. A second condition is geometric: proxy error
must be smaller than the separation between competing candidates. If only
\begin{equation}
  \metricnorm{P_{\Dcal}^{M}(t-q)}\leq\varepsilon
  \label{eq:projected-proxy-radius}
\end{equation}
is known, then Eq.~\eqref{eq:task-gap-transfer} gives
\begin{equation}
  |G_{ij}(t)-G_{ij}(q)|
  \leq2\varepsilon\metricnorm{y_i-y_j}.
  \label{eq:sharp-margin-bound}
\end{equation}
This bound is sharp under ball uncertainty: for a nonzero pair direction,
equality is attained when the projected error is parallel or antiparallel to
that direction. Near-tie decisions are therefore intrinsically fragile. More
generally, for any fixed $k>0$, an $O(\rho^k)$ proxy error as $\rho\downarrow0$,
with a constant independent of $\rho$, does not guarantee correct ranking when
a nonzero true best--runner-up gap is also $\Theta(\rho^k)$ on the same family;
the directional reversal
construction is given in Supplementary Section S1.
The 49,323 pair decisions follow this margin dependence. Shared
recovery was already 97.8\% in the smallest hidden-gap decile and approached
100\% as gaps widened. Raw residual improved more slowly, while validation and
static-centre baselines did not acquire the same high-margin recovery
(Fig.~\ref{fig:anchor-replication}b). The same margin logic applies to algebraic
solve error, which must be controlled on the scale of the decision being
resolved (Supplementary Section S1).

The shared response has a second boundary in the one-dimensional compressible-flow
shock probe. An undamped update could leave the physically admissible set,
whereas positivity and residual backtracking increased the number of feasible
cases from 22 of 35 to 35 of 35
(Fig.~\ref{fig:anchor-replication}c). The global certificate developed below
provides a different reliability route when strong monotonicity is available.
% --- End inlined file: sections/05_results_anchor_replication.tex ---
\FloatBarrier
% --- Begin inlined file: sections/06_results_comparison_geometry.tex ---
\subsection{Accurate ranking does not require accurate solution reconstruction}
\label{sec:comparison-geometry}

The shared proxy need not reproduce the hidden solution everywhere to rank the
candidates correctly. It only needs to be accurate in directions that
distinguish one candidate from another. This separation explains why strong
ranking can coexist with an imperfect corrected prediction, and why a small
full-task error can still flip a near-tied choice.

The finite-library quotient gives an exact orthogonal decomposition of the
task error,
\begin{equation}
  t-q
  =P_{\Dcal}^{M}(t-q)+P_{\Dcal^{\perp_M}}^{M}(t-q).
  \label{eq:decision-error-decomposition}
\end{equation}
Only the first term changes squared-loss comparisons. Consequently, when
$\Dcal^{\perp_M}$ is non-trivial, correct comparison does not imply a uniformly
accurate task proxy: the orthogonal component can be arbitrarily large
without changing any candidate decision. Conversely, whenever a candidate pair
is distinct, an arbitrarily small full-task error can cross an arbitrarily small
candidate-bisector margin and reverse a decision. Full-task accuracy and
decision accuracy are therefore neither equivalent nor ordered without a
margin condition.

This distinction concerns the final loss geometry, not whether the
state-to-task map is linear. If $Q_a$ is nonlinear but $t=Q_a(u)$, the proxy
$q$ and candidate outputs $y_i=Q_a(v_i)$ all lie in one fixed Hilbert task space,
Eq.~\eqref{eq:task-gap-transfer} remains exact. Nonlinear $Q_a$ changes the
remainders that connect a physical state correction to $q$, as in
Proposition~\ref{prop:physics-decision-representation}. By contrast, a
truth-normalized nRMSE, a general $L^p$ loss or a nonlinear
quantity-of-interest loss need not have affine pair contrasts. Such losses
require task-specific structure or a local tangent and remainder analysis.

The tested systems exhibited several regimes allowed by this geometry. Across
the public PDEBench population, the correction reduced full-trajectory error
by a factor of 9.86 and comparison-subspace error by a factor of 16.1. The
larger contraction in the comparison subspace shows preferential improvement
in decision-relevant directions
(Supplementary Fig. 5). In Sine--Gordon, the
median weighted full-trajectory error norm fell by a factor of 1,835.7. Here the
correction entered a full-proxy regime: $q$ was accurate not only in the
comparison subspace but also across the complete task. The correction paths and
task definitions are given in Supplementary Fig. 4
and Supplementary Section S12. The task loss of $q$ was
also no greater than that of the best candidate in all 1,920 primary
Burgers/reaction--diffusion cases and all 99 inputs under the fixed PDEBench
terminal squared-$L^2$ task. In these populations, ranking and direct
improvement coincided.

Notably, the ranking signal remained accurate under operator--data mismatch. In the two-dimensional compressible-flow probe, the proxy recovered the correct candidate comparison in all ten cases and reduced comparison-subspace error by a median factor of 22.2, although direct proxy use improved on the best candidate in only four cases (Supplementary Fig. 5). This behaviour directly realizes the separation permitted by the exact geometry:
\begin{equation*}
  \text{correct comparison}\;\not\Rightarrow\;
  \text{best corrected state}.
\end{equation*}

Direct use of a corrected state therefore requires accuracy in the full task
space, whereas selection depends only on the candidate-difference coordinate.
% --- End inlined file: sections/06_results_comparison_geometry.tex ---
\FloatBarrier
% --- Begin inlined file: sections/07_results_deployment_timescales.tex ---
\subsection{Shared scores support routing and reusable model selection}
\label{sec:deployment-timescales}

Once a hidden-reference ordering can be estimated, it can be used at different
points in a deployment workflow. At query time, the candidates for the current
input can be rescored against the task-matched proxy and the selected candidate
returned. At a slower timescale, scores can be averaged over an unlabeled
selection panel to rank the candidate models and choose one checkpoint for
repeated use. These
two uses share the candidate scores but answer different questions: routing
preserves input- and task-specific choice, whereas panel-level selection trades
that flexibility for a single reusable model
(Fig.~\ref{fig:deployment-timescales}).

\begin{figure}[!t]
  \centering
  \includegraphics[width=\linewidth]{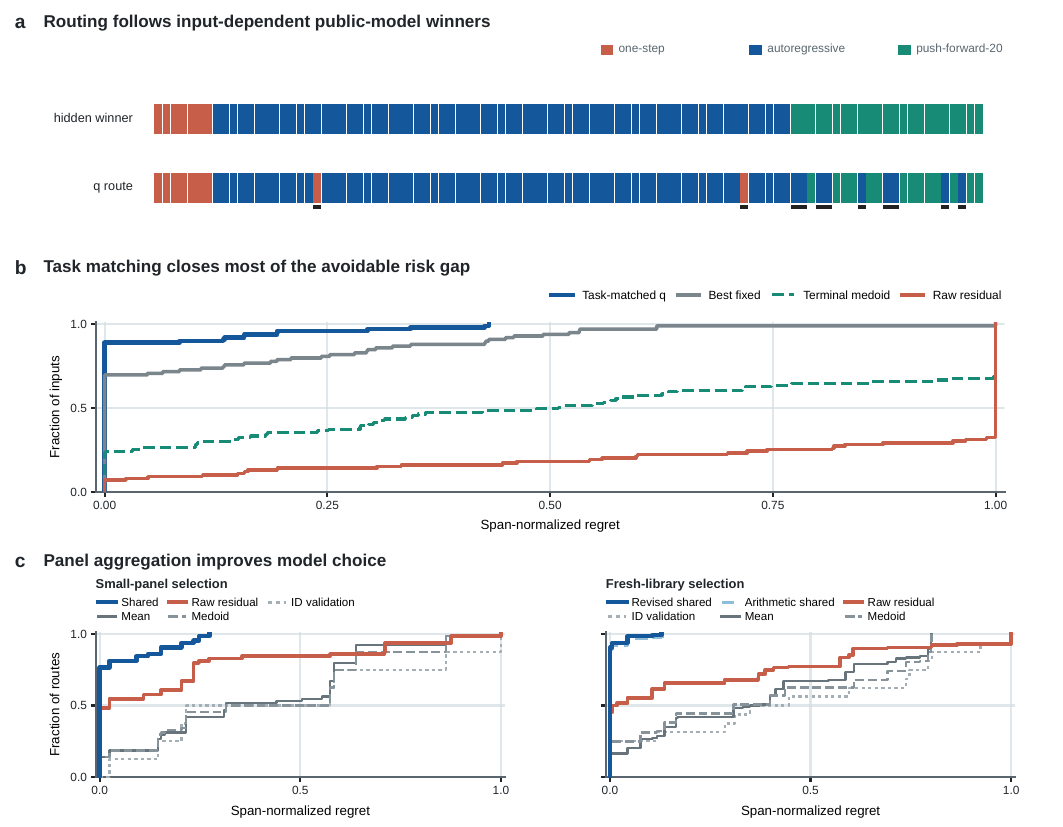}
  \caption{\textbf{Routing and panel-level model selection use comparison at different
  timescales.} \textbf{a}, Hidden-reference winner and task-matched proxy route
  for each of 99 inputs; colours distinguish checkpoints and black ticks mark
  mismatches. \textbf{b}, Empirical cumulative distributions of span-normalized
  selection regret for the same inputs. Regret is excess hidden loss divided by
  the three-model loss span; curves nearer the upper left are better.
  \textbf{c}, Empirical cumulative distributions of the same regret for
  panel-level selection. The left facet contains eight libraries with eight
  panel evaluations per library at $B=24$; the right contains 16 additional
  libraries with eight evaluations per library at $B=144$. Methods are
  identified in the panel legends.}
  \label{fig:deployment-timescales}
\end{figure}

Three public U-Net checkpoints for the two-field PDEBench diffusion--reaction
system provide a test beyond the trained libraries above
\citep{RonnebergerEtAl2015,TakamotoEtAl2022PDEBench}. These released
checkpoints were selected within the benchmark's validation workflow. On the terminal
two-field task, one checkpoint was best for every input and the shared score
recovered every choice, whereas raw residual recovered none of the top choices
(Supplementary Table 7). Because that
ordering was constant across the 99 inputs, the terminal task tests recovery on
one fixed public library rather than input-adaptive routing. We therefore
rescored the same trajectories under a restricted-quadrant high-band spectral
task, for which the preferred checkpoint became input dependent. This spectral
distance lies outside the exact squared-Hilbert identity and probes task
transfer beyond the exact comparison theorem.

For this high-band spectral objective, the preferred checkpoint changed across
inputs (Fig.~\ref{fig:deployment-timescales}a). Rescoring the same proxy in this
task recovered 88 of 99 choices, compared with 69 for retaining the best fixed
terminal model, 23 for medoid centrality and 7 for the raw residual; it closed
86.7\% of the fixed-to-oracle risk gap
(Fig.~\ref{fig:deployment-timescales}b). Because the terminal and spectral tasks
use the same models, inputs and proxy trajectories, the winner crossover
isolates task dependence rather than a change in the candidate library.
Additional task objectives are compared in Supplementary
Section S12.

Panel-level selection asks whether the instance-wise signal survives
aggregation. For an unlabeled panel of size $B$, candidate scores are averaged
to select one reusable checkpoint. Small panels were not always sufficient. At
$B=24$, performance varied sharply among libraries, including one
Burgers--CNO library in which only one of eight panels selected a top-two model.
Pooled across the eight libraries, shared scoring produced lower normalized regret
than the raw residual baseline, validation-loss selection and mean or medoid centrality
(Fig.~\ref{fig:deployment-timescales}c). This pooled advantage came from the
reaction--diffusion libraries and was not observed in either Burgers
architecture group.
Candidate-specific correction made the same selection and incurred the same
regret as the shared response in all 64 panel evaluations.
Strong query-wise ranking therefore need not translate into reliable selection
from a small selection panel.

Across 16 additional libraries evaluated with a selection panel of $B=144$, the
shared score chose a model within the study tolerance of the best in 122 of 128
panels and a top-two model in all 128. It also
outperformed the raw residual baseline, validation and candidate-only centrality
baselines (Fig.~\ref{fig:deployment-timescales}c), with normalized regret
concentrated near zero.
% --- End inlined file: sections/07_results_deployment_timescales.tex ---
\FloatBarrier
% --- Begin inlined file: sections/08_results_rd_certificate.tex ---
\subsection{Strong monotonicity certifies decisions against a same-grid operator reference}
\label{sec:rd-certificate}

The preceding results estimate decisions without access to a numerical
reference. For one zero-Dirichlet, strongly monotone cubic
reaction--diffusion discretization, the residual at any returned state supplies
a global error enclosure and hence a computable decision guarantee. This result
does not require the local linearization assumptions of
Proposition~\ref{prop:physics-decision-representation}.

Let $V_h$ contain the interior grid degrees of freedom with weighted inner
product
\begin{equation}
  \langle v,z\rangle_h=h^2\sum_k v_kz_k.
  \label{eq:rd-inner-product-main}
\end{equation}
Let $K_h$ be self-adjoint positive definite and define
\begin{equation}
  A_h(v)=K_hv+\lambda v+\gamma v^{\odot3}-f_h,
  \qquad \lambda,\gamma>0.
  \label{eq:rd-operator-main}
\end{equation}
Let $Q_h:V_h\to Y$ be linear into a finite-dimensional Hilbert task space. We
identify $Y$ with the task space $\Hcal$ above and
$\langle\cdot,\cdot\rangle_Y=\metricip{\cdot}{\cdot}$; the state-space metric
$M_s$ below is distinct from the fixed task metric $M$. Define the weighted
adjoint by
\begin{equation}
  \langle Q_hv,z\rangle_Y=\langle v,Q_h^*z\rangle_h.
  \label{eq:rd-weighted-adjoint-main}
\end{equation}
For any returned state $s\in V_h$, define
\begin{equation}
  M_s=K_h+\lambda I+
  \frac{3\gamma}{4}\operatorname{diag}(s^{\odot2}),
  \label{eq:rd-proxy-metric-main}
\end{equation}
and write $\|z\|_{M_s}^2=\langle M_sz,z\rangle_h$ and
$\|\rho\|_{M_s^{-1}}^2=\langle \rho,M_s^{-1}\rho\rangle_h$.

\begin{theorem}[Global strongly monotone reaction--diffusion decision certificate]
\label{thm:rd-decision-certificate}
Equation $A_h(u_h)=0$ has a unique solution, $M_s$ is self-adjoint positive
definite, and
\begin{equation}
  \|s-u_h\|_{M_s}
  \leq\|A_h(s)\|_{M_s^{-1}}.
  \label{eq:rd-state-enclosure}
\end{equation}
For $t=Q_hu_h$, $q=Q_hs$ and $d_{ij}=y_i-y_j$,
\begin{equation}
  |G_{ij}(t)-G_{ij}(q)|
  \leq
  2\|A_h(s)\|_{M_s^{-1}}
  \|Q_h^*d_{ij}\|_{M_s^{-1}}
  =:B_{ij}.
  \label{eq:rd-pair-certificate}
\end{equation}
Hence $|G_{ij}(q)|>B_{ij}$ certifies a strict pair sign. If proxy winner $k$
satisfies
\begin{equation}
  R_j(q)-R_k(q)>B_{kj}
  \quad\text{for every }j\ne k,
  \label{eq:rd-strict-top1-main}
\end{equation}
then it is the unique same-grid operator winner. More generally, for a scalar
tolerance $\tau\geq0$, if
\begin{equation}
  R_k(q)-R_j(q)+B_{kj}\leq\tau
  \quad\text{for every }j\ne k,
  \label{eq:rd-tolerant-top1-main}
\end{equation}
then its same-grid regret is at most $\tau$. Failure of a sufficient condition
is abstention.

Let $\mathcal B_h=Q_h^*\Dcal$ and
$r=\dim\mathcal B_h\leq N-1$. One residual inverse action and $r$
comparison-basis inverse actions determine every pair bound through a Gram
matrix.

If a dataset target satisfies $t_{\rm data}=t+\xi$ and
$P_{\Dcal}^{M}\xi\in\mathcal U_{\rm disc}$ for a compact set
$\mathcal U_{\rm disc}\subset\Dcal$, define
$\sigma_{\mathcal U}(d)=\sup_{z\in\mathcal U}\metricip{z}{d}$. Then
\begin{align}
  G_{ij}(t_{\rm data})\in\bigl[&G_{ij}(q)-B_{ij}
  -2\sigma_{\mathcal U_{\rm disc}}(-d_{ij}),\notag\\
  &G_{ij}(q)+B_{ij}
  +2\sigma_{\mathcal U_{\rm disc}}(d_{ij})\bigr].
  \label{eq:rd-dataset-discrepancy}
\end{align}
In particular, $\|P_{\Dcal}^{M}\xi\|_M\leq\Delta_{\Dcal}$ adds
$2\Delta_{\Dcal}\|d_{ij}\|_M$ to $B_{ij}$.
\end{theorem}

\begin{proof}
The scalar inequality
\begin{equation*}
  (a^3-b^3)(a-b)\geq\frac34a^2(a-b)^2
\end{equation*}
is sharp and follows from
$a^2+ab+b^2-3a^2/4=(b+a/2)^2$. Applied componentwise, it gives
\begin{equation*}
  \langle A_h(s)-A_h(u_h),s-u_h\rangle_h
  \geq\|s-u_h\|_{M_s}^{2}.
\end{equation*}
Strict convexity and coercivity of the discrete energy give existence and
uniqueness. Cauchy--Schwarz in the $M_s$ metric yields
Eq.~\eqref{eq:rd-state-enclosure}; the squared-Hilbert gap identity and weighted
adjoint relation yield Eq.~\eqref{eq:rd-pair-certificate}. The pair, strict and
tolerant conditions follow by one-sided interval arithmetic. Expanding every
$Q_h^*d_{ij}$ in a basis of $Q_h^*\Dcal$ gives the Gram reduction. Finally,
support-function addition transfers the same-grid interval to
$t_{\rm data}=t+\xi$.
\end{proof}

The certificate uses directional structure rather than a separate inverse
action for every pair. A basis of $Q_h^*\Dcal$ builds one Gram matrix from at
most $N-1$ comparison-basis actions; every pair bound is then a quadratic form
in that basis. The complete discrete proof,
solve-defect inflation and weighted-adjoint implementation are given in
Supplementary Section S11.

Across 24 library--input combinations and 672 candidate pairs, every pairwise,
strict top-1 and $\tau=10^{-6}$ tolerant top-1 margin exceeded its bound; the
largest bound-to-margin ratio was 0.09494 (Fig.~\ref{fig:rd-certificate}a,b).

\begin{figure}[H]
  \centering
  \includegraphics[width=\linewidth]{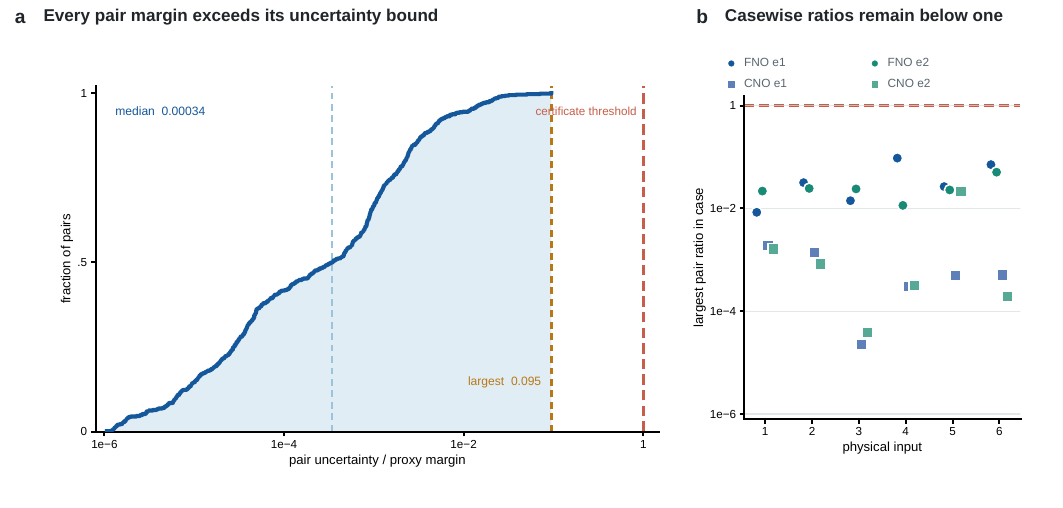}
  \caption{\textbf{Bound-to-margin ratios remain below the certification
  threshold.} \textbf{a}, Empirical cumulative distribution of 672 ratios
  $B_{ij}/|G_{ij}(q)|$; the vertical line marks the sufficient threshold of one.
  \textbf{b}, Largest ratio in each of 24 cases formed by six physical inputs
  and four libraries. All ratios are evaluated against the unique zero of the
  stated same-grid discrete operator.}
  \label{fig:rd-certificate}
\end{figure}

The theorem concerns the unique zero of Eq.~\eqref{eq:rd-operator-main} on the
stated grid. Agreement with a dataset or continuum target additionally requires
a discrepancy bound such as Eq.~\eqref{eq:rd-dataset-discrepancy}.
Cross-resolution results are given in Supplementary Section S10.
% --- End inlined file: sections/08_results_rd_certificate.tex ---
\FloatBarrier
% --- Begin inlined file: sections/10_discussion.tex ---
\section{Discussion}
\label{sec:discussion}

Neural operators are conventionally evaluated as individual predictors, yet a deployed scientific workflow frequently involves multiple plausible models. Once a finite surrogate library is established, the immediate operational challenge shifts from full-state reconstruction to candidate selection. Under a fixed squared Hilbert loss, the unknown solution influences candidate preferences exclusively through the span of candidate differences. This comparison subspace has a dimension of at most one less than the library size, regardless of the ambient state dimension. Recognizing this distinction is critical because it isolates the precise fraction of an unavailable reference required to make a reliable deployment decision.

The shared physics response translates this geometric reduction into a practical computational diagnostic. By propagating the complete constraint defect of a single library anchor through a full-state linearized solve, the method yields a corrected task output to score every candidate simultaneously. While defect correction, adjoint analysis, and a posteriori estimation supply the underlying numerical ingredients, the finite-library quotient identifies the specific family of relative decisions that a single response can resolve. In the reaction-diffusion census, candidate-wise variational estimators raised non-tie pair recovery from 59.8\% to 82.7\%, whereas the shared task proxy achieved 100\% recovery. Furthermore, candidate-specific correction-magnitude proxies produced almost identical choices to the shared calculation, confirming that library-level physics propagation preserves the necessary decision signals.

Despite this shared computational foundation, ranking and direct correction solve fundamentally different problems. Returning the corrected proxy demands accuracy across the entire task space, whereas candidate selection only requires accurate comparison coordinates and sufficient separation between predictions. In several smooth systems, the proxy proved highly effective for both purposes; however, the compressible-flow probe demonstrated that correct relative comparisons can persist even when the direct-use advantage diminishes. These scores subsequently support two distinct deployment modes. Query-time routing preserves the current input and task context, while panel-level aggregation identifies a single optimal model for repeated use. Both deployment strategies require only one post-inference physics response for the entire library, shifting the computational burden from a one-per-candidate scaling to a single shared solve. The overall operational cost will therefore depend primarily on the governing operator, the chosen discretization, and the broader serving workflow.

The reliable application of this shared response is naturally bounded by three operational constraints. First, small candidate margins inherently amplify both proxy and algebraic errors. Second, library composition directly influences the anchor state from which the physical response is computed. Because neither the arithmetic mean nor the candidate medoid demonstrated universal superiority across the tested datasets, anchor construction remains intrinsically library dependent. Third, shock-dominated dynamics may necessitate numerical damping to maintain physical admissibility. Consequently, panel-level aggregation heavily relies on whether a small, unlabeled evaluation panel accurately represents the eventual deployment population, a vulnerability highlighted by the retained Burgers-CNO failure case.

Moving beyond empirical proxy performance, the framework can yield rigorous theoretical guarantees for well-structured systems. For the strongly monotone reaction-diffusion discretization, the residual provides an exact-arithmetic bound for the unique zero of the corresponding same-grid operator, thereby certifying pairwise or top-1 decisions whenever the required margins are present. Transferring these guarantees to a dataset or continuum target will require an additional operator-data discrepancy bound. Future extensions of this work include integrating mixed and evolving candidate libraries, accommodating non-Hilbert objectives, utilizing multifidelity data, and conducting more rigorous stress tests under severe shocks or simulator mismatches. Across all such settings, the central pursuit remains unchanged: identifying precisely which features of a hidden physical solution are strictly necessary to distinguish among available machine learning predictions.
% --- End inlined file: sections/10_discussion.tex ---
% --- Begin inlined file: sections/03_method.tex ---
\section{Methods}
\label{sec:method}

\subsection*{Finite libraries and decision targets}

Let $a$ denote a deployment input and let
$\{\mathcal M_i\}_{i=1}^{N}$ be the finite candidate library. The
high-fidelity reference is unavailable when the library is queried. Candidate
$i$ returns a state trajectory
$v_i=\mathcal M_i(a)$ and a task
output
\begin{equation}
  y_i=Q_a(v_i)\in\Hcal,
\end{equation}
where $(\Hcal,\metricip{\cdot}{\cdot})$ is a fixed Hilbert space after all outputs
have been lifted to a common grid.  If $u_a$ is the numerical reference and
$t=Q_a(u_a)$, the hidden squared task loss is
$R_i(t)=\metricnorm{t-y_i}^{2}$.  The selector is specified without access to the
reference $t$. Throughout, ``reference-free'' means that scoring uses no
deployment reference solution or reference-derived label; the governing
operator, problem data and task remain available. In the Burgers and
reaction--diffusion comparisons, each library
contains candidates from one architecture; the two-dimensional compressible-flow analysis instead
compares one public FNO with one public U-Net. The score depends on the lifted
outputs rather than the model parameters, so architecture matching is an
experimental choice rather than a mathematical requirement.
Every experimental state-to-task map is linear after the fixed common-grid
lift: it is either the identity map or a terminal-state trace. The derivative
notation below retains the framework's extension to smooth nonlinear task maps.
The theorem-backed analyses use these maps with squared-Hilbert losses. Some
task-transfer distances applied to the same outputs are nonlinear and are
reported only as empirical transfer tests; they do not inherit the exact
squared-Hilbert comparison identity.

For an unordered pair $i<j$, define the diagnostic and truth gaps
\begin{equation}
  G_{ij}^{q}=S_j-S_i,
  \qquad
  G_{ij}^{t}=R_j(t)-R_i(t).
  \label{eq:gap-convention}
\end{equation}
Thus $G_{ij}>0$ means that candidate $i$ is preferred. A
per-case tolerance $\tau$ maps each gap to one of three labels: $i$ preferred,
tie, or $j$ preferred.  A strict pair comparison sets $\tau=0$ and tests the
resulting ordering for every unordered pair.  ``Non-tie pairwise accuracy'' instead
conditions on non-ties of the hidden-reference label, whereas ``all-pair three-way
accuracy'' retains the tie class.  Exact top-1 requires the same minimizing
checkpoint; tolerant top-1 accepts a selected checkpoint whose hidden loss is
within the chosen distance of the minimum. For the eight Burgers and
reaction--diffusion libraries, this distance is $10^{-6}$. These comparisons
use non-tie pairwise and tolerant top-1 recovery. Their all-pair sensitivity
analysis uses a $10^{-6}$
hidden-reference tolerance, a $1.01\times10^{-8}$ propagated-score numerical
tolerance (the fixed $10^{-10}$ absolute floor plus a $10^{-8}$ unit-scale
relative floor), and a $10^{-12}$ raw-score tolerance; a matched-tolerance
sensitivity analysis applies $10^{-6}$ to both scores. The 16-library comparison
uses strict pairwise and exact top-1, and
the Sine--Gordon study uses strict pairwise and exact top-1 with numerically
unresolved outcomes counted as errors.  The public PDEBench study also uses strict
pairwise and exact top-1 for one fixed library.

\subsection*{Study design and statistical units}

We crossed two PDEs, two architectures and two independently seeded ensembles
to form eight candidate libraries. Each library contained the
epoch-20 and epoch-40 checkpoints from four training runs, giving 64 checkpoints
from 32 runs.  Eligibility required finite metrics, in-distribution relative
$L^2$ error at most one, exact hard initial or boundary conditions where
applicable, and completion of the full eight-candidate library. Each PDE
contributed 240 controlled-shift inputs, giving 480 distinct physical inputs and
1,920 input--library cases; none overlapped with training, validation or
development data. Candidate eligibility, the anchor rule, task metric, tie
tolerances, selectors and evaluation inputs were specified before reference
evaluation. Within each input--library case, every candidate was scored against
the same library anchor and metric. Confidence intervals use 10,000
input-clustered bootstrap replicates.

Each eight-candidate library yields 28 dependent pair rows per input. Confidence
intervals therefore resample inputs while preserving the coupling of
architectures and ensembles evaluated on the same input. For the 16-library
comparison, resampling likewise preserves matched libraries and shared inputs.
Its hierarchical estimator first averages each metric over inputs
within each library and then gives each of the 16 libraries equal weight; it
therefore differs from the pooled 3,800/3,840 case count. For Sine--Gordon,
the cluster is the shared input identifier across architectures, libraries, and pairs.  A
Sine--Gordon pair was unresolved when the absolute fine-grid loss gap did not
exceed the sum of the candidates' coarse-to-fine loss changes; the top-1 rule
applied the analogous check to the best--runner-up gap. These outcomes were
counted as errors. The public PDEBench comparison contains one fixed
three-candidate library and is therefore summarized by exact counts.

\subsection*{Shared correction and candidate scores}

Let $\Acal_a(v)$ be the assembled discrete or continuous constraint defect,
including initial and boundary components rather than only the interior PDE
residual; components enforced exactly appear as zero rows. A
deterministic, reference-free library-anchor rule $\mathcal C$ produces
\begin{equation}
  w=\mathcal C(v_1,\ldots,v_N).
\end{equation}
The eight Burgers and reaction--diffusion libraries used the arithmetic
candidate mean. The 16-library, Sine--Gordon and public PDEBench comparisons
used the geometric medoid: the candidate
state minimizing the sum of distances to all candidates in the fixed anchor
metric.  The two choices define distinct variants of the framework.  Burgers and
reaction--diffusion medoid distances use the Euclidean norm over all stored state
entries on the diagnostic grid.  Sine--Gordon uses trapezoidal time--space $L^2$
distance over displacement and velocity, with velocity weight one. PDEBench uses
the corresponding physical time--space distance over both released fields. Exact
medoid ties follow the fixed candidate order. The anchor rule is part of the
method specification and is chosen without the deployment reference. The
composition study below shows that no single anchor was uniformly preferable
across the tested libraries; the response must be used in a library regime for
which its anchor and linearization remain informative.

At $w$, form $r_w=\Acal_a(w)$ and the state-dependent Jacobian
$L_w=D\Acal_a(w)$. The shared correction satisfies
\begin{equation}
  L_w\delta_w=-r_w+\ell,
  \label{eq:shared-correction}
\end{equation}
where $\ell$ is the algebraic solve defect and vanishes for an exact solve.
The corresponding corrected state, linearized task estimate and evaluated task
estimate are
\begin{equation}
  s=w+\delta_w,
  \qquad
  q_{\rm lin}=Q_a(w)+DQ_a(w)\delta_w,
  \qquad
  q_{\rm eval}=Q_a(s),
  \label{eq:q-general}
\end{equation}
respectively. Every task map used here is linear, so
$q\equiv q_{\rm lin}=q_{\rm eval}$ exactly; we retain the symbol $q$ for
the task proxy and $s$ for the full corrected state or
trajectory. For a nonlinear $Q_a$, the linearized and evaluated proxies remain
distinct. The
propagated score is Eq.~\eqref{eq:proxy-score}.  Equation~\eqref{eq:shared-correction}
is solved once per input and library; the resulting $q$ is shared by all
$N$ candidate comparisons. This implementation is a full-state linearized
solve, not an explicitly reduced-order solver; the finite-library quotient
specifies which component of its output is consumed by ranking.
Figure~\ref{fig:decision-framework} summarizes this
shared physics-to-decision path, and Fig.~\ref{fig:prospective-recovery}
reports its recovery and comparisons.

The raw residual baseline evaluates each candidate's assembled defect energy,
\begin{equation}
  S_i^{\mathrm{raw}}=\norm{\Acal_a(v_i)}_{Y,a}^{2},
\end{equation}
using the fixed residual-product norm.  This baseline uses the same candidates,
inputs and hidden references as the propagated method, while scoring
the defect directly instead of forming a task-space estimate.  Its exact block,
time, and quadrature weights are defined in Supplementary Section S6.  A higher-cost
control solves
$D\Acal_a(v_i)\delta_i=-\Acal_a(v_i)$ for every candidate and assigns
\begin{equation}
  q_i^{\mathrm{cand}}=Q_a(v_i)+DQ_a(v_i)\delta_i,
  \qquad
  S_i^{\mathrm{cand}}=\metricnorm{q_i^{\mathrm{cand}}-y_i}^{2}.
  \label{eq:candidate-specific-control}
\end{equation}
Thus the control uses one candidate-derived correction per candidate rather than
one common correction for the library.  It ranks candidates by increasing
$S_i^{\mathrm{cand}}$ and selects
$\widehat i_{\mathrm{cand}}=\arg\min_i S_i^{\mathrm{cand}}$; evaluation uses the
same pair and top-1 conventions defined above.  This score is the task-space norm
of the candidate-specific correction---because $y_i=Q_a(v_i)$, it reduces to
$\lVert DQ_a(v_i)\delta_i\rVert_M^2$---and serves as a candidate-specific
linearized error proxy. Comparing it with the shared score tests whether one
library-level correction retains the decisions of $N$ separate corrections.

Candidate-only baselines rank candidates by their in-distribution validation
loss or by squared task-space distance to the candidate mean or geometric
medoid. The validation ordering is independent of the deployment input, whereas
the two centrality scores depend on the current candidate outputs but do not use
the governing equation. For any selector evaluated against a finite set of
hidden losses, normalized regret is the selected excess loss divided by the
best-to-worst loss span, with a denominator floor of $10^{-8}$.

The component-cost comparison used nested $N=2,4,8,16$ libraries, two
architectures, three input strata and ten repeated CPU measurements. Timers
began after candidate outputs were available and covered only the shared or
candidate-specific physics calculation. The reported ratio therefore compares
one shared response with $N$ separate responses after candidate inference; the
complete timing design is reported in
Supplementary Section S9.

\begin{algorithm}[t]
\caption{Shared physics ranking for one deployment input}
\label{alg:shared-ranking}
\begin{algorithmic}[1]
\Require Candidate states $v_1,\ldots,v_N$; complete constraint $\Acal_a$;
linear task map $Q_a$; anchor rule $\mathcal C$, lift, and metric $M$.
\State Lift all candidates to their shared state space.
\State $w\gets\mathcal C(v_1,\ldots,v_N)$.
\State Evaluate $r_w\gets\Acal_a(w)$, retaining initial/boundary rows.
\State Compute $\delta_w$ from $D\Acal_a(w)\delta_w\simeq-r_w$ once and record
$\ell\gets D\Acal_a(w)\delta_w+r_w$.
\State Form $s\gets w+\delta_w$ and
$q\gets Q_a(w)+DQ_a(w)\delta_w=Q_a(s)$.
\For{$i=1,\ldots,N$}
  \State $y_i\gets Q_a(v_i)$; $S_i\gets\metricnorm{q-y_i}^{2}$.
\EndFor
\State \Return the ranked candidates, pair labels, or $\arg\min_i S_i$.
\end{algorithmic}
\end{algorithm}

\subsection*{Viscous Burgers}
On the periodic line, the equation is
$u_t-\nu u_{xx}+u u_x=f$ with prescribed initial state.  The candidate output is
the trajectory and the task output is its terminal state.  The complete defect
is
\begin{equation}
  \Acal_a(w)=\bigl(w(0)-u_0,\;w_t-\nu w_{xx}+w w_x-f\bigr),
\end{equation}
and
$L_wh=(h(0),h_t-\nu h_{xx}+\partial_x(wh))$.  The implementation uses periodic
Fourier differentiation \citep{Trefethen2000}, a time-marching tangent solve,
and a physical $L^2$ terminal metric after Fourier lifting to the common grid.

\subsection*{Cubic reaction--diffusion}
On $\Omega=(0,L)^2$ with zero Dirichlet boundary,
\begin{equation}
  -\nabla\!\cdot(a\nabla u)+\lambda u+\gamma u^3=f,
  \qquad u|_{\partial\Omega}=0.
\end{equation}
States include boundary nodes.  Boundary entries of $\Acal_a$ are the Dirichlet
defect, and interior entries are a conservative variable-coefficient
finite-difference defect.  The correction on the boundary is fixed exactly by the
Dirichlet rows; its known contribution is transferred to the interior right-hand
side before solving the boundary-eliminated system.  The interior Jacobian contains
the reaction factor $\lambda+3\gamma w^2$.
Candidate and corrected fields are compared after common interpolation in a fixed
tensor-product physical quadrature metric.  The task output is the complete
stationary solution field on this common grid, including boundary nodes.  Thus
reaction--diffusion tests the identity-task case, while the Burgers and
Sine--Gordon terminal tasks exercise task compression.

\subsection*{Reaction--diffusion decision certificate}

The certificate calculations use four reaction--diffusion groups, each with
eight candidates, and two inputs from each of three shift strata. This gives 24
library--input combinations and 672 unordered pairs. The
bounds depend on the candidates, forcing, coefficients, operator, anchor rule
and tolerance.

All certificate operators act on the 128-by-128 interior space after verifying
exact zero boundary values.  With the identity task and its $h^2$-weighted inner
product, the corrected proxy is $s=w+\delta_w$ and $q=s$. For each $s$, we
formed $M_s$ from Eq.~\eqref{eq:rd-proxy-metric-main}, computed the residual dual norm,
and solved a basis of $\mathcal B_h=Q_h^*\Dcal$, requiring
$r=\dim\mathcal B_h\leq N-1$ spanning right-hand sides. Approximate
inverse actions were inflated by their algebraic solve defects and an analytic
coercivity lower bound. Certificate quantities were evaluated in float64; the
theorem itself is stated in exact arithmetic.

\subsection*{Candidate-wise residual-estimator controls}

The residual-estimator comparison used all 240 reaction--diffusion inputs, four
eight-candidate libraries and two state realizations: the stored candidate
fields transferred to the comparison grid and the same checkpoints rerun on
that grid. This produced 1,920 input--library--scenario cases.
The controls compared the raw product residual, a common-energy dual norm based
on $H_h=K_h+\lambda I$, a candidate-adapted $L^2$ majorant and the shared
corrected proxy.  Each variational selector applies an inverse operator to all
eight candidate residuals; the shared score uses the single correction already
defined above.  Component timings begin after candidate outputs are available
and exclude inference, data movement, reference computation and end-to-end
serving. The complete estimator definitions and results appear in
Supplementary Section S13.

\subsection*{Sine--Gordon}
The third PDE is the periodic phase-space system
\begin{equation}
  u_t=v,\qquad v_t=c^2u_{xx}-\sin u+f.
\end{equation}
Its complete defect is defined by the initial displacement/velocity errors and all
steps of the fixed velocity-Verlet map \citep{HairerLubichWanner2006}.  The exact
Jacobian of that discrete map is block lower triangular, so
Eq.~\eqref{eq:shared-correction} is solved by a
forward sweep.  The task is terminal displacement and velocity, compared in the
physical phase-space $L^2$ metric with velocity weight one after periodic
Fourier lifting.

\subsection*{PDEBench two-field diffusion--reaction}
The released PDEBench models predict trajectories for the homogeneous-Neumann
system
\begin{equation}
  u_t=u-u^3-k-v+D_u\Delta u,
  \qquad
  v_t=u-v+D_v\Delta v,
\end{equation}
with $(D_u,D_v,k)=(10^{-3},5\times10^{-3},5\times10^{-3})$.  The candidate
trajectories begin with ten observed frames on the released $128\times128$
cell-centred grid.  Every checkpoint receives the same 20-channel history and
predicts one two-field next frame.  Irrespective of its training curriculum, each
checkpoint is evaluated by the same 91-step autoregressive rollout: the oldest
frame is dropped and the prediction appended until the 101-frame trajectory is
complete.  From the last observed frame onward, the complete defect
contains the known initial row and every trapezoidal-map defect.  Its exact
linearization is solved forward in time with a preconditioned sparse solve.  The
raw and propagated diagnostics therefore use identical candidate trajectories and
time horizons.  The task is the terminal two-field state under cell-area-weighted
squared $L^2$ loss.

The task-transfer analysis uses the same 99 inputs, three public checkpoints
and 91 predicted future frames. The already computed trajectory $s$ is scored
under ten objective-specific distances. These
comprised terminal and trajectory squared $L^2$, RMSE, a normalized-RMSE plug-in,
conservation, maximum and boundary errors, and three fixed restricted-quadrant
spectral bands. Only
the two squared-$L^2$ tasks lie directly within the exact Hilbert comparison
identity.  The normalized score uses the root-mean-square magnitude of its own
reference argument, so candidate-to-truth and candidate-to-proxy evaluations do
not share one fixed denominator. The ten metrics therefore provide correlated
views of the same 99 inputs.

\subsection*{Compressible-flow boundary studies}

The one-dimensional shock probe used 35 samples from an official PDEBench
held-out file and three released checkpoints: two FNOs with different history
contracts and one U-Net. Candidate states were transformed to conservative
variables and combined by an arithmetic mean. The comparison operator was a separately
implemented fixed-substep Rusanov--SSPRK2 map rather than the adaptive HLLC data
generator, so the dataset truth has a nonzero operator defect. A reference-free
backtracking rule tested multipliers $0$, $0.25$, $0.5$, $0.75$ and $1$ and chose
the largest value that preserved positive density and pressure without
increasing that defect. Inadmissible unit steps remained errors in strict
decision counts. Supplementary Section S15 gives the complete
results.

The two-dimensional probe used one public FNO and one public push-forward-20
U-Net for periodic PDEBench compressible-flow data at $64\times64$ resolution.
We used ten trajectories. Both candidates received the same ten observed frames and were
rolled autoregressively for eleven further frames.  Primitive variables were
converted to density, momenta and total energy before constructing the
arithmetic mean.  A clean-room periodic Rusanov finite-volume map with
density-weighted viscous fluxes and SSPRK2 time integration defined the comparison
operator; finite differences supplied Jacobian actions.  Because this operator
does not replay the dataset generator exactly, these results characterize an
operator-mismatched regime. The partial path
used $\alpha\in\{0,0.25,0.5,0.75,1,1.25\}$ with reference-free admissibility and
defect checks. Additional medoid results are given in Supplementary Section S14.

\subsection*{Numerical references used for evaluation}

Reference losses were computed from independent high-resolution solutions. For
the eight Burgers and reaction--diffusion libraries, Burgers references used integrating-factor RK4 on
independent 768- and 1,536-point Fourier grids; reaction--diffusion used 511- and
767-point interior grids. The 16-library comparison increased these pairs to
1,536/3,072 and 767/1,023, respectively.  Reaction--diffusion references used
damped Newton iteration with requested relative and absolute tolerances
$10^{-11}$ and $10^{-13}$, and a preconditioned linear solve with requested
relative tolerance $10^{-12}$.  Sine--Gordon references used spectral
velocity-Verlet at 384 spatial points with 6,144 time steps and at 768 points with
12,288 steps. Refinement, residual, replay, and Jacobian-closure
checks are detailed in Supplementary Section S6. For PDEBench, terminal states in
the released numerical trajectories provided the reference states.

These realizations share the decision object but use different discretizations,
anchor rules, and evaluation metrics.

\subsection*{Panel-level model selection}

For an unlabeled selection panel $A$, aggregate the same per-input score,
\begin{equation}
  \bar S_i(A)=\frac{1}{|A|}\sum_{a\in A}S_i(a),
  \qquad \widehat i(A)=\arg\min_i\bar S_i(A).
  \label{eq:panel-aggregate}
\end{equation}
For squared-Hilbert tasks, the instance-wise comparison identity averages
exactly over the panel. With
$G_{ij,a}^{t}$ and $G_{ij,a}^{q}$ denoting the truth and diagnostic gaps for input
$a$,
\begin{equation}
 \frac{1}{|A|}\sum_{a\in A}(G_{ij,a}^{t}-G_{ij,a}^{q})
 =\frac{2}{|A|}\sum_{a\in A}
   \metricip{t_a-q_a}{y_i(a)-y_j(a)}.
 \label{eq:panel-gap-average}
\end{equation}
Thus panel-level selection averages the same decision coordinates over the
inputs. On an evaluation set $D$, only $\mathcal M_{\widehat i(A)}$ is run. Its
mean task loss $R(\widehat i)$ is compared with the best fixed checkpoint
$i^\star=\arg\min_i R(i)$ through absolute regret and
\begin{equation}
  \operatorname{NRegret}(\widehat i)=
  \frac{R(\widehat i)-R(i^\star)}
       {\max\{R(i^{\mathrm{worst}})-R(i^\star),10^{-8}\}}.
  \label{eq:nregret}
\end{equation}
Uncertainty estimates treat inputs as the sampling units and preserve the
dependence among candidate pairs from the same input.

\subsection*{Use of generative language tools}

Generative language tools assisted manuscript organization, prose revision and
code review. The authors verified the reported mathematics and numerical
results and take responsibility for the final manuscript. No experimental data
were generated by these tools.
% --- End inlined file: sections/03_method.tex ---

\section*{Data availability}
The numerical data underlying Figs.~2--6 and all quantitative Supplementary
figures and tables are provided in the accompanying Source Data file. The
processed data and supporting files used in the analyses are available at
\url{https://github.com/ToughClimb/OpPropRank-open-source}; the version used in
this study is
\href{https://github.com/ToughClimb/OpPropRank-open-source/commit/dafb2a06d0815a6678174fc9e1aadc9055b7e69b}{commit \texttt{dafb2a0}}.
The larger archive of candidate predictions, reference solutions and selection
results is available in the
\href{https://github.com/ToughClimb/OpPropRank-open-source/releases/tag/data-v1.0.0}{\texttt{data-v1.0.0} release}.
The PDEBench data and pretrained models used in this study are available under
CC BY 4.0 from \href{https://doi.org/10.18419/DARUS-2986}{DaRUS-2986} and
\href{https://doi.org/10.18419/DARUS-2987}{DaRUS-2987}, respectively.

\section*{Code availability}
The code and tests needed to reproduce the numerical analyses are available at
\url{https://github.com/ToughClimb/OpPropRank-open-source} (commit
\href{https://github.com/ToughClimb/OpPropRank-open-source/commit/dafb2a06d0815a6678174fc9e1aadc9055b7e69b}{\texttt{dafb2a0}}).
The accompanying source package also contains the scripts used to generate all
six main figures.

\section*{Author contributions}
H.L. conceived the study, developed the theory and methodology, implemented the
software, designed and performed the computational experiments, analysed the
data, prepared the figures and drafted the manuscript. F.L. supervised the
research, provided project guidance and reviewed and edited the manuscript.
Both authors reviewed and approved the final manuscript.

\section*{Funding}
This work was supported by the Science and Technology Development Program of
Jilin Province under Grant No.~20250102032JC.

\section*{Ethics declarations}
No ethical approval was required because this study used computational
experiments and publicly available or generated data and did not involve human
participants or animals.

\section*{Competing interests}
The authors declare no competing interests.


\begingroup
\small
\begin{thebibliography}{26}
\providecommand{\natexlab}[1]{#1}
\providecommand{\url}[1]{\texttt{#1}}
\expandafter\ifx\csname urlstyle\endcsname\relax
  \providecommand{\doi}[1]{doi: #1}\else
  \providecommand{\doi}{doi: \begingroup \urlstyle{rm}\Url}\fi

\bibitem[Li et~al.(2021)Li, Kovachki, Azizzadenesheli, Liu, Bhattacharya,
  Stuart, and Anandkumar]{LiEtAl2021FNO}
Zongyi Li, Nikola Kovachki, Kamyar Azizzadenesheli, Burigede Liu, Kaushik
  Bhattacharya, Andrew Stuart, and Anima Anandkumar.
\newblock Fourier neural operator for parametric partial differential
  equations.
\newblock In \emph{International Conference on Learning Representations}, 2021.
\newblock URL \url{https://openreview.net/forum?id=c8P9NQVtmnO}.

\bibitem[Lu et~al.(2021)Lu, Jin, Pang, Zhang, and
  Karniadakis]{LuEtAl2021DeepONet}
Lu~Lu, Pengzhan Jin, Guofei Pang, Zhongqiang Zhang, and George~Em Karniadakis.
\newblock Learning nonlinear operators via {DeepONet} based on the universal
  approximation theorem of operators.
\newblock \emph{Nature Machine Intelligence}, 3:\penalty0 218--229, 2021.
\newblock \doi{10.1038/s42256-021-00302-5}.

\bibitem[Kovachki et~al.(2023)Kovachki, Li, Liu, Azizzadenesheli, Bhattacharya,
  Stuart, and Anandkumar]{KovachkiEtAl2023}
Nikola Kovachki, Zongyi Li, Burigede Liu, Kamyar Azizzadenesheli, Kaushik
  Bhattacharya, Andrew Stuart, and Anima Anandkumar.
\newblock Neural operator: Learning maps between function spaces with
  applications to {PDE}s.
\newblock \emph{Journal of Machine Learning Research}, 24\penalty0
  (89):\penalty0 1--97, 2023.
\newblock URL \url{https://www.jmlr.org/papers/v24/21-1524.html}.

\bibitem[Raoni{\'c} et~al.(2023)Raoni{\'c}, Molinaro, De~Ryck, Rohner,
  Bartolucci, Alaifari, Mishra, and de~B{\'e}zenac]{RaonicEtAl2023CNO}
Bogdan Raoni{\'c}, Roberto Molinaro, Tim De~Ryck, Tobias Rohner, Francesca
  Bartolucci, Rima Alaifari, Siddhartha Mishra, and Emmanuel de~B{\'e}zenac.
\newblock Convolutional neural operators for robust and accurate learning of
  {PDE}s.
\newblock In \emph{Advances in Neural Information Processing Systems},
  volume~36, 2023.
\newblock URL
  \url{https://proceedings.neurips.cc/paper_files/paper/2023/hash/f3c1951b34f7f55ffaecada7fde6bd5a-Abstract-Conference.html}.

\bibitem[Takamoto et~al.(2022)Takamoto, Praditia, Leiteritz, MacKinlay,
  Alesiani, Pfl{\"u}ger, and Niepert]{TakamotoEtAl2022PDEBench}
Makoto Takamoto, Timothy Praditia, Raphael Leiteritz, Dan MacKinlay, Francesco
  Alesiani, Dirk Pfl{\"u}ger, and Mathias Niepert.
\newblock {PDEBench}: An extensive benchmark for scientific machine learning.
\newblock In \emph{36th Conference on Neural Information Processing Systems
  Datasets and Benchmarks Track}, 2022.
\newblock URL
  \url{https://proceedings.neurips.cc/paper_files/paper/2022/hash/0a9747136d411fb83f0cf81820d44afb-Abstract-Datasets_and_Benchmarks.html}.

\bibitem[Gupta and Brandstetter(2023)]{GuptaBrandstetter2023}
Jayesh~K. Gupta and Johannes Brandstetter.
\newblock Towards multi-spatiotemporal-scale generalized {PDE} modeling.
\newblock \emph{Transactions on Machine Learning Research}, 2023.
\newblock URL \url{https://openreview.net/forum?id=dPSTDbGtBY}.

\bibitem[Ohana et~al.(2024)Ohana, McCabe, Meyer, Morel, Agocs, Beneitez,
  Berger, Burkhart, Dalziel, Fielding, Fortunato, Goldberg, Hirashima, Jiang,
  Kerswell, Maddu, Miller, Mukhopadhyay, Nixon, Shen, Watteaux,
  R{\'e}galdo-Saint~Blancard, Rozet, Parker, Cranmer, and Ho]{OhanaEtAl2024}
Ruben Ohana, Michael McCabe, Lucas Meyer, Rudy Morel, Fruzsina~J. Agocs, Miguel
  Beneitez, Marsha Berger, Blakesley Burkhart, Stuart~B. Dalziel, Drummond~B.
  Fielding, Daniel Fortunato, Jared~A. Goldberg, Keiya Hirashima, Yan-Fei
  Jiang, Rich~R. Kerswell, Suryanarayana Maddu, Jonah Miller, Payel
  Mukhopadhyay, Stefan~S. Nixon, Jeff Shen, Romain Watteaux, Bruno
  R{\'e}galdo-Saint~Blancard, Fran{\c c}ois Rozet, Liam~H. Parker, Miles
  Cranmer, and Shirley Ho.
\newblock The well: a large-scale collection of diverse physics simulations for
  machine learning.
\newblock In \emph{Advances in Neural Information Processing Systems},
  volume~37, pages 44989--45037. Curran Associates, Inc., 2024.
\newblock \doi{10.52202/079017-1430}.
\newblock URL
  \url{https://proceedings.neurips.cc/paper_files/paper/2024/hash/4f9a5acd91ac76569f2fe291b1f4772b-Abstract-Datasets_and_Benchmarks_Track.html}.
\newblock Datasets and Benchmarks Track.

\bibitem[Wei et~al.(2023)Wei, Wong, Sung, Gupta, Ooi, Chiu, Dao, and
  Ong]{WeiEtAl2023PINNSelection}
Zhao Wei, Jian~Cheng Wong, Nicholas Sung, Abhishek Gupta, Chin~Chun Ooi,
  Pao-Hsiung Chiu, My~Ha Dao, and Yew~Soon Ong.
\newblock How to select physics-informed neural networks in the absence of
  ground truth: A pareto front-based strategy.
\newblock In \emph{ICML 2023 Workshop on the Synergy of Scientific and Machine
  Learning Modelling}, 2023.
\newblock URL \url{https://icml.cc/virtual/2023/28612}.
\newblock Workshop paper.

\bibitem[Garg et~al.(2022)Garg, Balakrishnan, Lipton, Neyshabur, and
  Sedghi]{GargEtAl2022}
Saurabh Garg, Sivaraman Balakrishnan, Zachary~Chase Lipton, Behnam Neyshabur,
  and Hanie Sedghi.
\newblock Leveraging unlabeled data to predict out-of-distribution performance.
\newblock In \emph{International Conference on Learning Representations}, 2022.
\newblock URL \url{https://openreview.net/forum?id=o_HsiMPYh_x}.

\bibitem[Guillory et~al.(2021)Guillory, Shankar, Ebrahimi, Darrell, and
  Schmidt]{GuilloryEtAl2021}
Devin Guillory, Vaishaal Shankar, Sayna Ebrahimi, Trevor Darrell, and Ludwig
  Schmidt.
\newblock Predicting with confidence on unseen distributions.
\newblock In \emph{Proceedings of the IEEE/CVF International Conference on
  Computer Vision (ICCV)}, pages 1114--1124, 2021.
\newblock \doi{10.1109/ICCV48922.2021.00117}.

\bibitem[Raissi et~al.(2019)Raissi, Perdikaris, and
  Karniadakis]{RaissiEtAl2019}
M.~Raissi, P.~Perdikaris, and George~E. Karniadakis.
\newblock Physics-informed neural networks: A deep learning framework for
  solving forward and inverse problems involving nonlinear partial differential
  equations.
\newblock \emph{Journal of Computational Physics}, 378:\penalty0 686--707,
  2019.
\newblock \doi{10.1016/j.jcp.2018.10.045}.

\bibitem[Karniadakis et~al.(2021)Karniadakis, Kevrekidis, Lu, Perdikaris, Wang,
  and Yang]{KarniadakisEtAl2021}
George~Em Karniadakis, Ioannis~G. Kevrekidis, Lu~Lu, Paris Perdikaris, Sifan
  Wang, and Liu Yang.
\newblock Physics-informed machine learning.
\newblock \emph{Nature Reviews Physics}, 3:\penalty0 422--440, 2021.
\newblock \doi{10.1038/s42254-021-00314-5}.

\bibitem[Li et~al.(2024)Li, Zheng, Kovachki, Jin, Chen, Liu, Azizzadenesheli,
  and Anandkumar]{LiEtAl2024PINO}
Zongyi Li, Hongkai Zheng, Nikola Kovachki, David Jin, Haoxuan Chen, Burigede
  Liu, Kamyar Azizzadenesheli, and Anima Anandkumar.
\newblock Physics-informed neural operator for learning partial differential
  equations.
\newblock \emph{ACM/IMS Journal of Data Science}, 1\penalty0 (3):\penalty0
  1--27, 2024.
\newblock \doi{10.1145/3648506}.

\bibitem[Cao et~al.(2023)Cao, O'Leary-Roseberry, Jha, Oden, and
  Ghattas]{CaoEtAl2023}
Lianghao Cao, Thomas O'Leary-Roseberry, Prashant~K. Jha, J.~Tinsley Oden, and
  Omar Ghattas.
\newblock Residual-based error correction for neural operator accelerated
  infinite-dimensional bayesian inverse problems.
\newblock \emph{Journal of Computational Physics}, 486:\penalty0 112104, 2023.
\newblock \doi{10.1016/j.jcp.2023.112104}.

\bibitem[Jha(2024)]{Jha2024}
Prashant~K. Jha.
\newblock Residual-based error corrector operator to enhance accuracy and
  reliability of neural operator surrogates of nonlinear variational
  boundary-value problems.
\newblock \emph{Computer Methods in Applied Mechanics and Engineering},
  419:\penalty0 116595, 2024.
\newblock \doi{10.1016/j.cma.2023.116595}.

\bibitem[Huang and Perdikaris(2026)]{HuangPerdikaris2026PhysicsCorrect}
Xinquan Huang and Paris Perdikaris.
\newblock {PhysicsCorrect}: A training-free approach for stable neural {PDE}
  simulations.
\newblock \emph{Proceedings of the AAAI Conference on Artificial Intelligence},
  40\penalty0 (26):\penalty0 22057--22065, 2026.
\newblock \doi{10.1609/aaai.v40i26.39360}.

\bibitem[Prudhomme and Oden(1999)]{PrudhommeOden1999}
Serge Prudhomme and J.~Tinsley Oden.
\newblock On goal-oriented error estimation for elliptic problems: Application
  to the control of pointwise errors.
\newblock \emph{Computer Methods in Applied Mechanics and Engineering},
  176\penalty0 (1--4):\penalty0 313--331, 1999.
\newblock \doi{10.1016/S0045-7825(98)00343-0}.

\bibitem[Oden and Prudhomme(2001)]{OdenPrudhomme2001}
J.~Tinsley Oden and Serge Prudhomme.
\newblock Goal-oriented error estimation and adaptivity for the finite element
  method.
\newblock \emph{Computers \& Mathematics with Applications}, 41\penalty0
  (5--6):\penalty0 735--756, 2001.
\newblock \doi{10.1016/S0898-1221(00)00317-5}.

\bibitem[Becker and Rannacher(2001)]{BeckerRannacher2001}
Roland Becker and Rolf Rannacher.
\newblock An optimal control approach to a posteriori error estimation in
  finite element methods.
\newblock \emph{Acta Numerica}, 10:\penalty0 1--102, 2001.
\newblock \doi{10.1017/S0962492901000010}.

\bibitem[Giles and S{\"u}li(2002)]{GilesSuli2002}
Michael~B. Giles and Endre S{\"u}li.
\newblock Adjoint methods for {PDE}s: A posteriori error analysis and
  postprocessing by duality.
\newblock \emph{Acta Numerica}, 11:\penalty0 145--236, 2002.
\newblock \doi{10.1017/S096249290200003X}.

\bibitem[Hesthaven et~al.(2016)Hesthaven, Rozza, and Stamm]{HesthavenEtAl2016}
Jan~S. Hesthaven, Gianluigi Rozza, and Benjamin Stamm.
\newblock \emph{Certified Reduced Basis Methods for Parametrized Partial
  Differential Equations}.
\newblock SpringerBriefs in Mathematics. Springer, 2016.
\newblock \doi{10.1007/978-3-319-22470-1}.

\bibitem[Fanaskov et~al.(2024)Fanaskov, Rudikov, and
  Oseledets]{FanaskovEtAl2024}
Vladimir Fanaskov, Alexander Rudikov, and Ivan Oseledets.
\newblock Neural functional a posteriori error estimates.
\newblock \emph{arXiv preprint arXiv:2402.05585}, 2024.
\newblock \doi{10.48550/arXiv.2402.05585}.

\bibitem[Eiras et~al.(2024)Eiras, Bibi, Bunel, Dvijotham, Torr, and
  Kumar]{EirasEtAl2024}
Francisco Eiras, Adel Bibi, Rudy~R. Bunel, Krishnamurthy Dvijotham, Philip
  Torr, and M.~Pawan Kumar.
\newblock Efficient error certification for physics-informed neural networks.
\newblock In \emph{Proceedings of the 41st International Conference on Machine
  Learning}, volume 235 of \emph{Proceedings of Machine Learning Research},
  pages 12318--12347. PMLR, 2024.
\newblock URL \url{https://proceedings.mlr.press/v235/eiras24a.html}.

\bibitem[Ronneberger et~al.(2015)Ronneberger, Fischer, and
  Brox]{RonnebergerEtAl2015}
Olaf Ronneberger, Philipp Fischer, and Thomas Brox.
\newblock {U-Net}: Convolutional networks for biomedical image segmentation.
\newblock In \emph{Medical Image Computing and Computer-Assisted Intervention
  -- MICCAI 2015}, volume 9351 of \emph{Lecture Notes in Computer Science},
  pages 234--241. Springer, 2015.
\newblock \doi{10.1007/978-3-319-24574-4_28}.

\bibitem[Trefethen(2000)]{Trefethen2000}
Lloyd~N. Trefethen.
\newblock \emph{Spectral Methods in {MATLAB}}, volume~10 of \emph{Software,
  Environments, and Tools}.
\newblock Society for Industrial and Applied Mathematics, Philadelphia, PA,
  2000.
\newblock \doi{10.1137/1.9780898719598}.

\bibitem[Hairer et~al.(2006)Hairer, Lubich, and Wanner]{HairerLubichWanner2006}
Ernst Hairer, Christian Lubich, and Gerhard Wanner.
\newblock \emph{Geometric Numerical Integration: Structure-Preserving
  Algorithms for Ordinary Differential Equations}, volume~31 of \emph{Springer
  Series in Computational Mathematics}.
\newblock Springer, Berlin, Heidelberg, 2 edition, 2006.
\newblock \doi{10.1007/3-540-30666-8}.

\end{thebibliography}
\endgroup
\end{document}

\maketitle
\tableofcontents

% --- Begin inlined file: supplement_sections/s1_comparison_factorization.tex ---
\section{Comparison geometry: expanded statements}
\label{sec:s-comparison-geometry}

Let $(\Hcal,\metricip{\cdot}{\cdot})$ be a real Hilbert task space and let
$y_1,\ldots,y_N\in\Hcal$ be fixed. For any target $z\in\Hcal$, set
$R_i(z)=\metricnorm{z-y_i}^{2}$,
$G_{ij}(z)=R_j(z)-R_i(z)$, and
\begin{equation}
 \Dcal=\operatorname{span}\{y_i-y_1:2\le i\le N\}.
\end{equation}

\subsection{Global decision quotient and bounded-linear minimality}

\begin{theorem}[Finite-library decision quotient]
\label{thm:s-minimal-information}
For every $z\in\Hcal$,
\begin{equation}
  G_{ij}(z)=\metricnorm{y_j}^{2}-\metricnorm{y_i}^{2}
  +2\metricip{z}{y_i-y_j}.
  \label{eq:s-exact-gap}
\end{equation}
Consequently, $P_{\Dcal}^{M}z$ determines the exact gap vector, all ternary
pair signs and the complete weak ordering, and $\dim\Dcal\leq N-1$. It also
determines top-1 after a tie rule has been fixed as a function only of the
minimizer set.

Let $W$ be a real normed space and let $\mathcal T:\Hcal\to W$ be bounded and
linear. The exact gap vector, the
ternary sign vector and the complete weak ordering are determined by
$\mathcal Tz$ for every $z\in\Hcal$ if and only if
\begin{equation}
  \ker\mathcal T\subseteq\Dcal^{\perp_M}.
  \label{eq:s-global-kernel-condition}
\end{equation}
Fix a map $\tau$ that chooses one index from every nonempty subset of
$\{1,\ldots,N\}$ and set
$\iota_\tau(z)=\tau(\operatorname*{arg\,min}_iR_i(z))$. The same kernel
condition is necessary and sufficient for $\mathcal Tz$ to determine
$\iota_\tau(z)$ globally. If $W$ is finite-dimensional, each of these exact
recovery properties implies $\dim W\geq\dim\Dcal$.
\end{theorem}

\begin{proof}
Expansion of the squared norms gives Eq.~\eqref{eq:s-exact-gap}. Every
$y_i-y_j$ lies in $\Dcal$, which proves projection sufficiency and the rank
bound. If $h\in\ker\mathcal T$ but
$h\notin\Dcal^{\perp_M}$, choose $i,j$ with
$\metricip{h}{y_i-y_j}\ne0$. The targets
$(y_i+y_j)/2\pm\varepsilon h$ have the same observation but opposite strict
pair signs. Thus exact signs, and therefore exact gaps and complete weak
ordering, require Eq.~\eqref{eq:s-global-kernel-condition}. Conversely, that
condition makes the difference of any two targets with equal observation
orthogonal to every candidate difference, so all gaps agree.

For top-1 necessity, set $a_i=\metricip{h}{y_i}$. These numbers are not all
equal. As $\alpha\to+\infty$, every minimizer of $R_i(\alpha h)$ belongs to
the maximum-$a_i$ set; as $\alpha\to-\infty$, every minimizer belongs to the
disjoint minimum-$a_i$ set. The fixed tie rule returns different indices from
the two disjoint sets, even though both observations are zero. Duplicate
candidate locations do not add a candidate-difference direction and do not
alter this argument. Finally, the kernel condition makes
$\mathcal T|_{\Dcal}$ injective, so a finite-dimensional codomain must have
dimension at least $\dim\Dcal$.
\end{proof}

This theorem makes the finite-candidate qualification essential. If the library
changes, so does $\Dcal$ and therefore the task information that must be
recovered. Its top-1 conclusion is global over the whole task space, not over
an arbitrary physical solution manifold.

The margin order cannot be omitted. For any fixed $k>0$ and
$0<\rho<1$, take $\Hcal=\Rnum$, $y_1=-1$, $y_2=1$,
$t=-\rho^k/4$ and $q=\rho^k/4$. Then
$\lvert t-q\rvert=\rho^k/2$, while
$G_{12}(t)=\rho^k$ and $G_{12}(q)=-\rho^k$. Thus a proxy error of order
$\rho^k$ can reverse a nonzero truth gap of the same order.

\subsection{Restricted target classes and exposure}

Let $V\subset\Hcal$ be a closed subspace, let
$\mathcal S\subset z_0+V$, and define
\begin{equation}
  \Dcal_V=\operatorname{span}\{P_V^M(y_i-y_1):2\leq i\leq N\}.
\end{equation}
Call $h\in V$ decision-active when
$h\notin\Dcal_V^{\perp_M}$. The class $\mathcal S$ is gap-exposing,
sign-exposing or top-1-exposing in $V$ if, for every decision-active $h$, it
contains two points on one affine line parallel to $h$ whose gap vectors,
ternary sign vectors or fixed tie-broken winners, respectively, differ.

\begin{proposition}[Restricted-domain sufficiency and conditional minimality]
\label{prop:s-restricted-exposure}
Let $W$ be a real normed space and let $\mathcal T:\Hcal\to W$ be bounded and
linear. The condition
\begin{equation}
  \ker(\mathcal T|_V)\subseteq\Dcal_V^{\perp_M}\cap V
  \label{eq:s-restricted-kernel-condition}
\end{equation}
is sufficient for $\mathcal Tz$ to determine every exact gap, sign and
tie-broken top-1 decision on $z_0+V$, and hence on $\mathcal S$. It is
necessary for the corresponding target on $\mathcal S$ when $\mathcal S$ is
gap-, sign- or top-1-exposing, respectively. A complete affine class
$\mathcal S=z_0+V$ satisfies all three exposure properties.
\end{proposition}

\begin{proof}
If two points of $z_0+V$ have the same observation, their difference lies in
$\ker(\mathcal T|_V)$. Under Eq.~\eqref{eq:s-restricted-kernel-condition},
that difference is orthogonal to every projected candidate difference and
hence leaves all gaps unchanged. Conversely, if a decision-active
$h\in\ker(\mathcal T|_V)$ exists, the appropriate exposure property supplies
two targets separated along $h$ with the same observation but different target
objects. A complete affine class contains every such line. Along the whole
line, a nonconstant affine gap crosses its bisector, and sufficiently large
positive and negative parameters expose disjoint maximum- and
minimum-projection candidate sets.
\end{proof}

Exposure cannot be omitted. For $\Hcal=\Rnum$, candidates $y_1=0$ and
$y_2=10$, and $\mathcal S=(-1,1)$, candidate 1 is always uniquely optimal and
the pair sign is constant. A zero-dimensional statistic therefore determines
sign and top-1 on $\mathcal S$, although the exact gap varies and
$\Dcal=\Rnum$.

\subsection{A continuous-statistic dimension bound for exact gaps}

\begin{proposition}[Continuous exact-gap statistics require comparison dimension]
\label{prop:s-continuous-dimension}
Let $r=\dim\Dcal$ and let $\Phi:\Hcal\to\Rnum^m$ be continuous. Suppose an
arbitrary decoder $F$ satisfies
\begin{equation}
  \bigl(R_i(z)-R_1(z)\bigr)_{i=2}^{N}=F(\Phi(z))
  \quad\text{for every }z\in\Hcal.
\end{equation}
Then $m\geq r$. No continuity, linearity or measurability is required of $F$.
\end{proposition}

\begin{proof}
The relative-gap map is injective on $\Dcal$: if two points of $\Dcal$ have
the same relative gaps, their difference lies in both $\Dcal$ and
$\Dcal^{\perp_M}$ and is zero. Hence $\Phi|_{\Dcal}$ is a continuous
injection $\Rnum^r\to\Rnum^m$. If $m<r$, compose this map with the standard
embedding of $\Rnum^m$ into $\Rnum^r$. Invariance of domain would make the
image open in $\Rnum^r$, but it lies in a proper linear subspace and has empty
interior, a contradiction.
\end{proof}

The kernel condition is only an identifiability statement. Stable recovery of
the comparison coordinate requires the extended observability constant
\begin{equation}
  \kappa_{\Dcal}(\mathcal T)
  =\inf\left\{C\geq0:
  \metricnorm{P_{\Dcal}^{M}h}\leq C\norm{\mathcal Th}_W
  \ \text{for every }h\in\Hcal\right\},
  \label{eq:s-observability-constant}
\end{equation}
where the infimum of the empty set is $+\infty$.

\begin{theorem}[Stable comparison observability]
\label{thm:s-stable-comparison-observability}
Let $W$ be a real Hilbert space and let $\mathcal T:\Hcal\to W$ be bounded and
linear.  The following are
equivalent:
\begin{enumerate}[label=(\roman*)]
  \item there is a bounded $B:W\to\Dcal$ satisfying
  $B\mathcal T=P_{\Dcal}^{M}$;
  \item there is $0\leq\kappa<\infty$ such that
  \begin{equation}
    \metricnorm{P_{\Dcal}^{M}h}
    \leq\kappa\norm{\mathcal Th}_{W}
    \quad\text{for every }h\in\Hcal.
    \label{eq:s-stable-comparison-observability}
  \end{equation}
\end{enumerate}
Such a factorization exists if and only if
$\kappa_{\Dcal}(\mathcal T)<\infty$. Whenever these equivalent conditions hold,
\begin{equation}
  \min_{B\mathcal T=P_{\Dcal}^{M}}\norm B
  =\kappa_{\Dcal}(\mathcal T).
\end{equation}
\end{theorem}

\begin{proof}
A factorization gives the inequality with $\kappa=\norm{B}$.  Conversely,
define $B_0(\mathcal Th)=P_{\Dcal}^{M}h$ on
$\operatorname{Ran}\mathcal T$.  Equation~\eqref{eq:s-stable-comparison-observability}
makes this map well defined and bounded, so it extends to
$\overline{\operatorname{Ran}\mathcal T}$.  Composing the extension with the
orthogonal projection of $W$ onto that closed subspace defines $B$ on $W$.
The constructions preserve the optimal norm, proving the final statement. No
closed-range assumption is needed: the factor is first extended to
$\overline{\operatorname{Ran}\mathcal T}$ and then to $W$ by its orthogonal
projection.
\end{proof}

\begin{theorem}[Exact deterministic minimax amplification]
\label{thm:s-minimax-observability}
Let $W$ be a real Hilbert space, let $\mathcal T:\Hcal\to W$ be bounded and
linear, and let $\delta>0$. Observe
\begin{equation}
  y=\mathcal Tx+n,
  \qquad x\in\Hcal,
  \qquad n\in W,
  \qquad \norm n_W\leq\delta.
\end{equation}
Over all set maps $\widehat z:W\to\Dcal$, without linearity, continuity or
measurability restrictions,
\begin{equation}
  \inf_{\widehat z:W\to\Dcal}
  \sup_{\substack{x\in\Hcal\\\norm n_W\leq\delta}}
  \metricnorm{\widehat z(\mathcal Tx+n)-P_{\Dcal}^{M}x}
  =\kappa_{\Dcal}(\mathcal T)\delta.
  \label{eq:s-exact-minimax}
\end{equation}
\end{theorem}

\begin{proof}
If $\kappa_{\Dcal}(\mathcal T)<\infty$, the optimal bounded factor $B$ from
Theorem~\ref{thm:s-stable-comparison-observability} gives the estimator
$\widehat z(y)=By$ and the upper bound
$\norm{Bn}\leq\kappa_{\Dcal}(\mathcal T)\delta$.

For the lower bound, take any $h$ with $\mathcal Th\ne0$ and rescale it so
$\norm{\mathcal Th}_W=2\delta$. The two admissible state--noise pairs
\begin{equation*}
  (x_+,n_+)=\left(\frac h2,-\frac{\mathcal Th}{2}\right),
  \qquad
  (x_-,n_-)=\left(-\frac h2,\frac{\mathcal Th}{2}\right)
\end{equation*}
produce the same observation. Any estimator therefore incurs error at least
$\metricnorm{P_{\Dcal}^{M}h}/2$ on one pair, which equals
$\delta\metricnorm{P_{\Dcal}^{M}h}/\norm{\mathcal Th}_W$. Taking the
supremum proves the finite case even when it is not attained. If an invisible
$h$ has nonzero decision projection, scaling $h$ makes the risk infinite. If
no such direction exists but the finite ratios are unbounded, the same
two-point construction yields arbitrarily large risk. The case
$\kappa_{\Dcal}=0$ follows from the upper bound and nonnegativity.
\end{proof}

Equation~\eqref{eq:s-exact-minimax} concerns the unrestricted state class and
the full deterministic $W$-noise ball. A restricted solution manifold,
non-spherical noise, a stochastic risk or a prior-constrained estimator can
have a smaller minimax constant.

\section{Stable task factorization and nonlinear decision radii}
\label{sec:s-factorization}

Let $X,Y,Z$ be Hilbert spaces, let $\mathcal U\subset X$ be open, and let
$\Acal:\mathcal U\to Y$ and $Q:\mathcal U\to Z$.  A selected solution branch
satisfies $\Acal(u)=0$.

\begin{theorem}[Stable task factorization]
\label{thm:s-factorization}
Fix $v\in\mathcal U$, assume $\Acal$ and $Q$ are Fr\'echet differentiable at
$v$, and let $L=D\Acal(v)$ and $J=DQ(v)$.  The following are equivalent:
\begin{enumerate}[label=(\roman*)]
  \item there is a bounded $T:Y\to Z$ such that $J=TL$;
  \item there is $0\leq\kappa<\infty$ such that
  $\norm{Jh}_Z\le\kappa\norm{Lh}_Y$ for every $h\in X$.
\end{enumerate}
The least admissible $\kappa$ equals
\begin{equation}
  \kappa_*(J\mid L)=
  \inf\left\{\kappa\geq0:
  \norm{Jh}_Z\leq\kappa\norm{Lh}_Y\ \text{for every }h\in X\right\},
\end{equation}
with the infimum of the empty set defined as $+\infty$. Whenever it is finite,
it also equals
\begin{equation}
  \kappa_*(J\mid L)
  =\sup_{Lh\ne0}\frac{\norm{Jh}_Z}{\norm{Lh}_Y},
\end{equation}
with value zero when $L=0=J$. A factorization can be chosen with
$\norm{T}=\kappa_*(J\mid L)$.
\end{theorem}

\begin{proof}
If $J=TL$, the inequality holds with $\kappa=\norm T$.  Conversely, the inequality
implies $\ker L\subseteq\ker J$, so $T_0(Lh)=Jh$ is well defined on
$\operatorname{Ran}L$ and has norm $\kappa_*$.  It extends continuously to
$\overline{\operatorname{Ran}L}$.  Let
$P:Y\to\overline{\operatorname{Ran}L}$ be the orthogonal projection and compose
the extension with $P$.  The resulting $T$ satisfies
$J=TL$ and $\norm T=\norm{T_0}=\kappa_*$.
\end{proof}

For a scalar task, $J\in X^*$ and the factorization is equivalent, after Riesz
identification, to solvability of the task adjoint equation
$L^*z=J$.  State identifiability can therefore fail while a particular comparison
direction remains observable.

\begin{proposition}[Exact shared-response and task-map identities]
\label{prop:s-shared-response-identities}
Let $w,u\in\mathcal U$, suppose $\Acal(u)=0$, and assume that
$L_w=D\Acal(w):X\to Y$ is a bounded isomorphism. Let an inexact shared response
satisfy
\begin{equation}
  L_w\delta_w=-\Acal(w)+\ell,
\end{equation}
set $e=u-w$, and define
\begin{equation}
  \Rcal_A(e)=\Acal(w+e)-\Acal(w)-L_we.
\end{equation}
Then
\begin{equation}
  \delta_w-e=L_w^{-1}\bigl[\ell+\Rcal_A(e)\bigr].
  \label{eq:s-shared-response-error}
\end{equation}
If $Q$ is Fr\'echet differentiable at $w$ and $w+\delta_w\in\mathcal U$, set
$J_w=DQ(w)$, define
$\Rcal_Q(h)=Q(w+h)-Q(w)-J_wh$ whenever $w+h\in\mathcal U$, set $t=Q(u)$, and
\begin{equation}
  q_{\rm lin}=Q(w)+J_w\delta_w,
  \qquad
  q_{\rm eval}=Q(w+\delta_w).
\end{equation}
Then
\begin{align}
  q_{\rm lin}-t
  &=J_wL_w^{-1}\bigl[\ell+\Rcal_A(e)\bigr]-\Rcal_Q(e),
  \label{eq:s-q-linear-error}\\
  q_{\rm eval}-t
  &=J_wL_w^{-1}\bigl[\ell+\Rcal_A(e)\bigr]
    +\Rcal_Q(\delta_w)-\Rcal_Q(e).
  \label{eq:s-q-evaluated-error}
\end{align}
When the task space is the Hilbert space used for candidate comparison,
applying $P_{\Dcal}^{M}$ gives the projected identities in the main-text
physics-to-decision proposition. If
$Q$ is bounded and linear, the two proxies coincide and both task remainders
vanish.
\end{proposition}

\begin{proof}
The expansion
$0=\Acal(u)=\Acal(w)+L_we+\Rcal_A(e)$ and the response equation give
$L_w(\delta_w-e)=\ell+\Rcal_A(e)$, proving
Eq.~\eqref{eq:s-shared-response-error}. The identities for $q_{\rm lin}$ and
$q_{\rm eval}$ follow by subtracting
$t=Q(w)+J_we+\Rcal_Q(e)$ and substituting the state identity.
\end{proof}

\begin{theorem}[Nonlinear task--residual identity]
\label{thm:s-nonlinear-identity}
Assume $\Acal,Q\in C^1(\mathcal U)$ and the segment $[v,u]\subset\mathcal U$.
For any bounded $T_v:Y\to Z$, let
$e=u-v$ and $\widehat{\Delta Q}_v=-T_v\Acal(v)$.  Then
\begin{equation}
 \widehat{\Delta Q}_v-[Q(u)-Q(v)]
 =\int_0^1 [T_vD\Acal(v+se)-DQ(v+se)]e\,ds.
 \label{eq:s-nonlinear-identity}
\end{equation}
Suppose, on the segment,
\begin{align}
 \norm{D\Acal(x)-D\Acal(v)}&\le M_A\norm{x-v}^{\alpha},\\
 \norm{DQ(x)-DQ(v)}&\le M_Q\norm{x-v}^{\alpha},
\end{align}
where $0<\alpha\le1$, and let
$\eta_v=\norm{T_vD\Acal(v)-DQ(v)}$.  If $\norm e\le R$, then
\begin{equation}
 \norm{\widehat{\Delta Q}_v-[Q(u)-Q(v)]}
 \le \eta_vR+
 \frac{\norm{T_v}M_A+M_Q}{1+\alpha}R^{1+\alpha}.
 \label{eq:s-holder-radius}
\end{equation}
\end{theorem}

\begin{proof}
The fundamental theorem of calculus gives
$-\Acal(v)=\int_0^1D\Acal(v+se)e\,ds$ and
$Q(u)-Q(v)=\int_0^1DQ(v+se)e\,ds$.  Subtraction proves
Eq.~\eqref{eq:s-nonlinear-identity}.  Add and subtract the derivatives at $v$,
apply the stated bounds, and integrate $s^\alpha$ over $[0,1]$.
\end{proof}

For exact factorization, $\eta_v=0$.  In a $C^{1,1}$ setting, the task correction
is therefore second-order accurate once a localization engine supplies $R$.  Such
an engine may be Newton--Kantorovich, coercivity, a verified interval method, or a
problem-specific energy estimate. In this paper, a computable radius is supplied
for the strongly monotone reaction--diffusion operator in Supplementary
Section~\ref{sec:s-rd-certificate}.

\subsection{A computable local Newton-to-decision route}

The preceding identity becomes a certificate only when the anchor-to-root
distance is itself enclosed. The following standard contraction argument records
one sufficient route and its decision consequence.

\begin{theorem}[Conditional local Newton decision enclosure]
\label{thm:s-local-newton-decision}
Let the closed ball $\overline B(v,R_0)$ be contained in $\mathcal U$. Assume
$\Acal\in C^1$ there, $L=D\Acal(v)$ is invertible with
$\norm{L^{-1}}\leq\beta$, and
\begin{equation}
  \norm{D\Acal(x)-D\Acal(y)}
  \leq M\norm{x-y}_X
  \quad\text{for all }x,y\in\overline B(v,R_0).
\end{equation}
Set $\delta=-L^{-1}\Acal(v)$,
$\Rcal_{A,v}(h)=\Acal(v+h)-\Acal(v)-Lh$, $\eta=\norm\delta_X$ and
$a=\beta M/2$. If $a=0$, set $R=\eta$. If $a>0$, assume
$4a\eta<1$ and set
\begin{equation}
  R=\frac{1-\sqrt{1-4a\eta}}{2a}.
  \label{eq:s-local-newton-radius}
\end{equation}
Suppose in either case that $R\leq R_0$. Then $\Acal$ has a unique zero
$u_\star$ in $\overline B(v,R)$,
\begin{equation}
  \norm{u_\star-v}_X\leq R,
  \qquad
  \norm{D\Acal(x)^{-1}}
  \leq\frac{\beta}{1-\beta MR}
  \quad(x\in\overline B(v,R)),
  \label{eq:s-persistent-invertibility}
\end{equation}
where the denominator is positive.

Let $Q:X\to\Hcal$ be bounded and linear, let
$q=Q(v+\delta)$, $t_\star=Q(u_\star)$, and let $\Dcal$ be the comparison
subspace of a fixed candidate library. The set
\begin{equation}
  \mathcal U_R=
  \overline{\left\{
  -P_{\Dcal}^{M}QL^{-1}\Rcal_{A,v}(h):
  \norm h_X\leq R
  \right\}}\subset\Dcal
  \label{eq:s-local-projected-set}
\end{equation}
is compact and contains $P_{\Dcal}^{M}(t_\star-q)$. In particular,
\begin{equation}
  \metricnorm{P_{\Dcal}^{M}(t_\star-q)}
  \leq
  \norm{P_{\Dcal}^{M}QL^{-1}}\frac{M}{2}R^2.
  \label{eq:s-local-decision-ball}
\end{equation}
The support-function conditions in Supplementary
Section~\ref{sec:s-support-calculus} therefore give computable pair, top-1,
regret and abstention decisions whenever the stated constants and supports are
available.
\end{theorem}

\begin{proof}
For $h\in\overline B(0,R)$ define
$\Phi(h)=\delta-L^{-1}\Rcal_{A,v}(h)$. Taylor's theorem gives
$\norm{\Rcal_{A,v}(h)}\leq(M/2)\norm h_X^2$, so
\begin{equation}
  \norm{\Phi(h)}_X\leq\eta+aR^2=R.
\end{equation}
When $a=0$, the remainder vanishes on the ball and $\Phi$ is constant. When
$a>0$, the derivative bound on the convex ball gives
\begin{equation}
  \norm{\Phi(h)-\Phi(k)}_X
  \leq\beta MR\norm{h-k}_X,
  \qquad
  \beta MR=1-\sqrt{1-4a\eta}<1.
\end{equation}
Thus Banach's theorem gives a unique fixed point $h_\star$ in the ball, and
$\Acal(v+h_\star)=0$. For any $x$ in the ball,
$\norm{L^{-1}(D\Acal(x)-L)}\leq\beta MR<1$; the Banach perturbation lemma
proves Eq.~\eqref{eq:s-persistent-invertibility}.

The exact correction identity gives
$t_\star-q=-QL^{-1}\Rcal_{A,v}(h_\star)$. Hence its comparison projection
belongs to the set before closure in Eq.~\eqref{eq:s-local-projected-set}.
Because $\Dcal$ is finite-dimensional, the closure of this bounded set is
compact. The Taylor bound gives Eq.~\eqref{eq:s-local-decision-ball}.
\end{proof}

This theorem identifies the zero produced in the certified ball. Applying it
to a separately designated physical truth requires proof that the designated
branch is that local zero. Residual smallness alone cannot provide that branch
identification.

\begin{corollary}[Pairwise decision interval]
Let $t,q\in\Hcal$ and $\varepsilon\geq0$.  If
$\metricnorm{P_{\Dcal}^{M}(t-q)}\le\varepsilon$, then
\begin{equation}
 G_{ij}(t)\in
 [G_{ij}(q)-2\varepsilon\metricnorm{y_i-y_j},
  G_{ij}(q)+2\varepsilon\metricnorm{y_i-y_j}].
\end{equation}
An interval lying strictly on one side of zero certifies that pair.  If, for one
index $\widehat i\in\{1,\ldots,N\}$, every interval for
$G_{\widehat i j}(t)$ with $j\ne\widehat i$ is contained in $(0,\infty)$, then
$\widehat i$ is certified top-1.
\end{corollary}

\begin{proof}
The squared-gap expansion gives
$G_{ij}(t)-G_{ij}(q)=2\metricip{t-q}{y_i-y_j}$.  Since
$y_i-y_j\in\Dcal$, projection and Cauchy--Schwarz bound its absolute value by
$2\varepsilon\metricnorm{y_i-y_j}$, which gives the interval.  Strict
separation from zero fixes a pair sign.  If every such sign favours one
candidate against every competitor, that candidate is top-1.
\end{proof}

\subsection{A conditional finite-sample bound}

The panel-selection studies use clustered bootstrap rather than the following
bounded-i.i.d. result, which gives a simple finite-sample comparison.

\begin{proposition}[Uniform score-mean concentration]
Let $m\geq1$ and $N\geq2$ be integers, let $B_0\geq0$ and
$\delta\in(0,1)$, and let
$a_1,\ldots,a_m$ be independent and identically distributed with law $\mu$.
For each $1\leq i\leq N$, let $S_i$ be measurable and satisfy
$S_i(a)\in[0,B_0]$ for $\mu$-almost every $a$.  Then, with probability at least
$1-\delta$,
\begin{equation}
 \max_{1\le i\le N}
 \left|\frac1m\sum_{k=1}^m S_i(a_k)-\mathbb E_{a\sim\mu} S_i(a)\right|
 \le B_0\sqrt{\frac{\log(2N/\delta)}{2m}}.
\end{equation}
If the population score gap between the best and second-best candidate exceeds
twice this radius, empirical panel selection recovers the population minimizer.
\end{proposition}

\begin{proof}
Apply Hoeffding's inequality to each of the $N$ bounded score functions with
failure probability $\delta/N$ and take a union bound.  On the resulting event,
the empirical difference between the best and any competitor has the same sign
whenever the population difference exceeds twice the radius.
\end{proof}

Direct application to the tested $B=24$ and $B=144$ settings would require a
useful common bound $B_0$ and identically sampled inputs. Their stratified,
dependent sampling designs require the clustered analysis reported here instead.
% --- End inlined file: supplement_sections/s1_comparison_factorization.tex ---
% --- Begin inlined file: supplement_sections/s2_burgers.tex ---
\section{Periodic viscous Burgers: functional setting and proofs}
\label{sec:s-burgers}

Let $\mathbb T=\Rnum/(2\pi\mathbb Z)$, $I=(0,T)$, and $\nu>0$.  For input
$a=(u_0,f)\in H^1(\mathbb T)\times L^2(I;L^2(\mathbb T))$, consider
\begin{equation}
 u_t-\nu u_{xx}+u u_x=f,\qquad u(0)=u_0.
 \label{eq:s-burgers}
\end{equation}
Set
\begin{align}
 X&=H^1(I;L^2(\mathbb T))\cap L^2(I;H^2(\mathbb T)),\\
 Y&=H^1(\mathbb T)\times L^2(I;L^2(\mathbb T)),\\
 \Acal_a(w)&=\bigl(w(0)-u_0,\;w_t-\nu w_{xx}+w w_x-f\bigr),\\
 \mathbf B(p,z)&=\left(0,\frac12\partial_x(pz)\right).
\end{align}
Then $\Acal_a(w)=\mathcal L_0w+\mathbf B(w,w)-(u_0,f)$ and
\begin{equation}
 L_vh=D\Acal_a(v)h
 =\bigl(h(0),h_t-\nu h_{xx}+\partial_x(vh)\bigr).
 \label{eq:s-linearized-burgers}
\end{equation}

\begin{lemma}[Trace and bilinear bounds]
\label{lem:s-trace-bilinear}
There are constants $c_{\rm tr}(T)$ and $c_B(\nu,T)$ such that
\begin{equation}
 X\hookrightarrow C([0,T];H^1(\mathbb T)),\qquad
 \norm w_{C_tH_x^1}\le c_{\rm tr}\norm w_X,
 \label{eq:s-trace}
\end{equation}
and
\begin{equation}
 \norm{\mathbf B(p,z)}_Y\le c_B\norm p_X\norm z_X.
 \label{eq:s-bilinear}
\end{equation}
Consequently $\Acal_a:X\to Y$ is analytic.
\end{lemma}

\begin{proof}
For smooth $w$, let $A=I-\partial_{xx}$.  Then
\[
\frac{d}{dt}\norm{w(t)}_{H^1}^2=2\ip{w_t(t)}{Aw(t)}_{L^2}.
\]
Choose $t_0$ so that
$\norm{w(t_0)}_{H^1}^2\le T^{-1}\norm w_{L_t^2H_x^1}^2$.  Integration and
Cauchy--Schwarz give
\[
\sup_t\norm{w(t)}_{H^1}^2
\le T^{-1}\norm w_{L_t^2H_x^1}^2
+2\norm{w_t}_{L_t^2L_x^2}\norm{Aw}_{L_t^2L_x^2}
\le C_T\norm w_X^2.
\]
Density yields Eq.~\eqref{eq:s-trace}.  In one dimension,
$H^1(\mathbb T)\hookrightarrow L^\infty(\mathbb T)$, hence
\begin{align*}
\norm{\partial_x(pz)}_{L_t^2L_x^2}
&\le \norm p_{L_t^\infty L_x^\infty}\norm{z_x}_{L_t^2L_x^2}
+\norm z_{L_t^\infty L_x^\infty}\norm{p_x}_{L_t^2L_x^2}\\
&\le C\norm p_X\norm z_X.
\end{align*}
The first component of $\mathbf B$ is zero.
\end{proof}

\begin{theorem}[Strong well-posedness]
\label{thm:s-burgers-wp}
For every $u_0\in H^1(\mathbb T)$ and $f\in L^2(I;L^2(\mathbb T))$,
Eq.~\eqref{eq:s-burgers} has a unique
\begin{equation}
 u\in X\cap C([0,T];H^1(\mathbb T)).
\end{equation}
On bounded input sets,
$\norm u_{C_tH_x^1}+\norm u_{L_t^2H_x^2}+\norm{u_t}_{L_t^2L_x^2}$
is bounded by a finite function of $\nu,T,\norm{u_0}_{H^1}$, and
$\norm f_{L_t^2L_x^2}$.
\end{theorem}

\begin{proof}
For a smooth Galerkin solution, testing by $u$ and using periodicity gives
\begin{equation}
 \frac{d}{dt}\norm u_{L^2}^2+2\nu\norm{u_x}_{L^2}^2
 \le \norm f_{L^2}^2+\norm u_{L^2}^2.
 \label{eq:s-l2-energy}
\end{equation}
Gronwall controls $u$ in $L_t^\infty L_x^2$ and $u_x$ in $L_t^2L_x^2$.
Testing by $-u_{xx}$ and writing
$y=\norm{u_x}_{L^2}^2$, $z=\norm{u_{xx}}_{L^2}^2$ gives
\[
\frac12y'+\nu z=-(f,u_{xx})-\frac12\int_{\mathbb T}u_x^3\,dx.
\]
The mean of $u_x$ is zero.  The one-dimensional Gagliardo--Nirenberg inequality
gives
\[
 \norm{u_x}_{L^3}^3
 \le C\norm{u_{xx}}_{L^2}^{1/2}\norm{u_x}_{L^2}^{5/2}
 =Cz^{1/4}y^{5/4}.
\]
Young's inequality with conjugate exponents $4$ and $4/3$ then gives
\begin{equation}
 y'+\nu z\le C_\nu y^{5/3}+C_\nu\norm f_{L^2}^2.
 \label{eq:s-h1-energy}
\end{equation}
Since Eq.~\eqref{eq:s-l2-energy} controls $\int_0^T y$, the coefficient
$C_\nu y^{2/3}$ in $y^{5/3}=y^{2/3}y$ is uniformly integrable:
$\int_0^T y^{2/3}\le T^{1/3}(\int_0^T y)^{2/3}$.  Gronwall applied to the
Galerkin inequality yields a uniform $L_t^\infty H_x^1$ bound; integrating
Eq.~\eqref{eq:s-h1-energy} yields the $L_t^2H_x^2$ bound.  The equation then gives
$u_t\in L_t^2L_x^2$ because $u\in L_t^\infty L_x^\infty$ and
$u_x\in L_t^2L_x^2$.

Fourier--Galerkin approximations therefore exist on $[0,T]$ with uniform $X$ and
$L_t^\infty H_x^1$ bounds.  Weak compactness and Aubin--Lions give, along a
subsequence, strong convergence in $L_t^2H_x^1$, which is sufficient to pass to
$u u_x$. Indeed, the one-dimensional embedding $H^1\hookrightarrow L^\infty$
and the uniform $L_t^\infty H_x^1$ bound give
\[
 \norm{u_n u_{n,x}-u u_x}_{L_t^2L_x^2}
 \le \norm{u_n-u}_{L_t^2L_x^\infty}
       \norm{u_{n,x}}_{L_t^\infty L_x^2}
 +\norm{u}_{L_t^\infty L_x^\infty}
       \norm{u_{n,x}-u_x}_{L_t^2L_x^2}\longrightarrow0.
\]
The trace bound supplies $u(0)=u_0$.

For uniqueness, if $w=u_1-u_2$ and $m=(u_1+u_2)/2$, then
$w_t-\nu w_{xx}+\partial_x(mw)=0$, $w(0)=0$.  Testing by $w$ gives
\[
\frac12\frac{d}{dt}\norm w_{L^2}^2+\nu\norm{w_x}_{L^2}^2
=(mw,w_x)
\le\frac\nu2\norm{w_x}_{L^2}^2+C_\nu\norm m_{H^1}^2\norm w_{L^2}^2.
\]
Because $m\in C_tH_x^1$, Gronwall yields $w=0$.
\end{proof}

\begin{theorem}[Linearized isomorphism]
\label{thm:s-linear-iso}
For every fixed $v\in X$, $L_v:X\to Y$ in
Eq.~\eqref{eq:s-linearized-burgers} is a bounded isomorphism.  If
$L_vh=(h_0,g)$ and $V=\norm v_{C_tH_x^1}$, then
\begin{equation}
 \norm h_X\le C_{\nu,T,V}
 (\norm{h_0}_{H^1}+\norm g_{L_t^2L_x^2}).
 \label{eq:s-linear-estimate}
\end{equation}
In particular $\beta_v:=\norm{L_v^{-1}}_{\mathcal L(Y,X)}<\infty$.
\end{theorem}

\begin{proof}
Boundedness follows from Lemma~\ref{lem:s-trace-bilinear}.  Testing
$h_t-\nu h_{xx}+\partial_x(vh)=g$ by $h$ gives
\begin{equation}
 \frac{d}{dt}\norm h_{L^2}^2+\nu\norm{h_x}_{L^2}^2
 \le C_\nu(1+V^2)\norm h_{L^2}^2+C\norm g_{L^2}^2.
\end{equation}
Testing by $-h_{xx}$ and using
$\norm{\partial_x(vh)}_{L^2}\le CV\norm h_{H^1}$ gives
\begin{equation}
 \frac{d}{dt}\norm{h_x}_{L^2}^2+\nu\norm{h_{xx}}_{L^2}^2
 \le C_\nu V^2\norm h_{H^1}^2+C_\nu\norm g_{L^2}^2.
\end{equation}
After addition and Gronwall,
\[
\norm h_{L_t^\infty H_x^1}+\norm h_{L_t^2H_x^2}
\le C_{\nu,T,V}(\norm{h_0}_{H^1}+\norm g_{L_t^2L_x^2}).
\]
The equation controls $h_t$ and proves Eq.~\eqref{eq:s-linear-estimate}.
Fourier--Galerkin approximation gives existence, and the same estimate with zero
data gives uniqueness.
\end{proof}

\subsection{Burgers correction identities, conditional localization, and sharpness}
\label{sec:s-burgers-correction}

Let $u$ be the strong solution, $v\in X$, $e=u-v$, $r_v=\Acal_a(v)$, and
$m=(u+v)/2$.  Define
\begin{equation}
 \delta_v=-L_v^{-1}r_v,\qquad
 \delta_u=-L_u^{-1}r_v,\qquad
 \delta_m=-L_m^{-1}r_v.
\end{equation}

\begin{theorem}[Exact secant and plug-in identities]
\label{thm:s-secant}
The following identities hold:
\begin{align}
 -r_v&=L_me, & \delta_m&=e,\\
 \delta_v-e&=L_v^{-1}\mathbf B(e,e), &
 \delta_u-e&=-L_u^{-1}\mathbf B(e,e).
 \label{eq:s-plugin-signs}
\end{align}
Consequently
\begin{align}
 \norm{\delta_v-e}_X&\le\beta_vc_B\norm e_X^2,\\
 \norm{\delta_u-e}_X&\le\beta_uc_B\norm e_X^2,\\
 \norm{L_v^{-1}-L_u^{-1}}&\le2c_B\beta_v\beta_u\norm e_X.
\end{align}
\end{theorem}

\begin{proof}
The first two identities follow directly by expanding the quadratic Burgers
operator about $v$, $u$, and their midpoint. Since $m=u-e/2$,
$L_m=L_u-\mathbf B(e,\cdot)$, which gives the negative sign in the
truth-linearized formula. The estimates follow from
Eq.~\eqref{eq:s-bilinear}.  Finally use the resolvent identity
$L_v^{-1}-L_u^{-1}=L_v^{-1}(L_u-L_v)L_u^{-1}$ and
$L_u-L_v=2\mathbf B(e,\cdot)$.
\end{proof}

\begin{theorem}[Computable Newton localization]
\label{thm:s-newton-cert}
Suppose $\bar\beta_v\ge\norm{L_v^{-1}}$, set
$\eta_v=\norm{\delta_v}_X$ and $\alpha_v=\bar\beta_vc_B$, and assume
$4\alpha_v\eta_v<1$.  Define
\begin{equation}
 R_v=\frac{1-\sqrt{1-4\alpha_v\eta_v}}{2\alpha_v}.
 \label{eq:s-newton-radius}
\end{equation}
Then $\norm{u-v}_X\le R_v$.  For a bounded linear task $C:X\to Z$,
\begin{equation}
 \left|\norm{C\delta_v}_Z-\norm{C(u-v)}_Z\right|
 \le\norm C\,\alpha_vR_v^2.
 \label{eq:s-newton-score-radius}
\end{equation}
\end{theorem}

\begin{proof}
For $w=u-v$, the quadratic equation is equivalent to the fixed point
\[
 w=\delta_v-L_v^{-1}\mathbf B(w,w)=:\Phi(w).
\]
On the closed radius-$R_v$ ball,
$\norm{\Phi(w)}\le\eta_v+\alpha_vR_v^2=R_v$, while
\[
\norm{\Phi(w)-\Phi(z)}
\le2\alpha_vR_v\norm{w-z},
\qquad
2\alpha_vR_v=1-\sqrt{1-4\alpha_v\eta_v}<1.
\]
Banach's theorem produces a unique root in the ball.  Global uniqueness from
Theorem~\ref{thm:s-burgers-wp} identifies it with $u-v$.  The secant estimate and
the reverse triangle inequality prove Eq.~\eqref{eq:s-newton-score-radius}.
\end{proof}

Candidate-wise pair intervals, top-1 separation, and regret would follow by
applying Eq.~\eqref{eq:s-newton-score-radius} candidate by candidate. The
implemented experiments do not compute a verified $\bar\beta_v$ or instantiate
these conditional localization radii. The reaction--diffusion construction in
Supplementary Section~\ref{sec:s-rd-certificate} is the paper's only evaluated
decision-certificate chain; this Burgers theorem is not a source of the
hidden-reference accuracy claim.

\subsubsection{Self-midpoint correction}

Set $\widetilde m_v=v+\delta_v/2$ and
$\delta_v^{\rm mid}=-L_{\widetilde m_v}^{-1}r_v$.

\begin{theorem}[Cubic self-midpoint error]
\label{thm:s-midpoint}
Under Theorem~\ref{thm:s-newton-cert}, $L_{\widetilde m_v}$ is invertible and
\begin{equation}
 \norm{L_{\widetilde m_v}^{-1}}
 \le\bar\beta_v^{\rm mid}:=
 \frac{\bar\beta_v}{1-\alpha_v\eta_v}.
\end{equation}
Moreover,
\begin{equation}
 \delta_v^{\rm mid}-e
 =-L_{\widetilde m_v}^{-1}
 \mathbf B\!\left(L_v^{-1}\mathbf B(e,e),e\right),
 \label{eq:s-midpoint-identity}
\end{equation}
and hence
\begin{equation}
 \norm{\delta_v^{\rm mid}-e}_X
 \le\bar\beta_v^{\rm mid}\bar\beta_vc_B^2\norm e_X^3.
\end{equation}
\end{theorem}

\begin{proof}
$L_{\widetilde m_v}-L_v=\mathbf B(\delta_v,\cdot)$, so the Banach perturbation
lemma applies because $\bar\beta_vc_B\eta_v=\alpha_v\eta_v<1/4$.  Let
$q_v=\delta_v-e=L_v^{-1}\mathbf B(e,e)$.  Since the exact midpoint is
$m=v+e/2$, $\widetilde m_v-m=q_v/2$ and
$L_m-L_{\widetilde m_v}=-\mathbf B(q_v,\cdot)$.  Using $-r_v=L_me$ gives
Eq.~\eqref{eq:s-midpoint-identity}; the norm bound follows twice from
Eq.~\eqref{eq:s-bilinear}.
\end{proof}

Iterating the fixed-residual midpoint map
$c_{n+1}=-L_{v+c_n/2}^{-1}r_v$, $c_0=\delta_v$, gives the exact recursion
\begin{equation}
 c_{n+1}-e=-L_{v+c_n/2}^{-1}\mathbf B(c_n-e,e).
\end{equation}
Thus, under uniform inverse bounds, $c_n-e=O(\norm e^{n+2})$.  Cubic order is a
property of one self-midpoint step, not an oracle-budget optimum; repeated Newton
iterations can converge faster.

\subsubsection{Order-sharp rank reversals}

Take zero data so $u=0$ and terminal task $C_Tw=w(T)$.  Choose
$\vartheta\in C^\infty([0,T])$ with $\vartheta(0)=0$ and
$\vartheta(t)>0$ for $t>0$, and let
$h(t,x)=\vartheta(t)(\sin x-\sin2x)$.  Set
$q=L_0^{-1}\mathbf B(h,h)$.  Direct Fourier
calculation yields
\begin{equation}
 \ip{h(T)}{q(T)}_{L^2}
 =\frac{\pi\vartheta(T)}2(I_1-I_2)>0,
 \quad
 I_k=\int_0^T e^{-\nu k^2(T-s)}\vartheta(s)^2\,ds.
\end{equation}
If $H=\norm{h(T)}_{L^2}$ and
$\kappa=\ip{h(T)}{q(T)}/H$, then
\begin{align}
 \norm{C_T\delta_{-dh}}&=dH+\kappa d^2+O(d^3),\\
 \norm{C_T\delta_{dh}}&=dH-\kappa d^2+O(d^3).
\end{align}
Choosing $d_1=\epsilon-c\epsilon^2$ and
$d_2=\epsilon+c\epsilon^2$ with $0<c<\kappa/H$ makes the first true error smaller
but the first diagnostic score larger.  Hence the second-order norm-ranking margin
is order-sharp for this decoder.

For the self-midpoint decoder, let
$h_k(t,x)=\vartheta(t)\sin(kx)$ and
\begin{equation}
 q_k=L_0^{-1}\mathbf B(h_k,h_k),\qquad
 r_k=-L_0^{-1}\mathbf B(q_k,h_k).
\end{equation}
Writing $q_k=Q_k(t)\sin(2kx)$ gives
\begin{equation}
 Q_k(t)=\frac{k}{2}\int_0^t e^{-4\nu k^2(t-s)}\vartheta(s)^2\,ds>0.
\end{equation}
The $k$-mode coefficient of $r_k(T)$ is
\begin{equation}
 P_k(T)=\frac{k}{4}\int_0^T e^{-\nu k^2(T-s)}\vartheta(s)Q_k(s)\,ds>0,
\end{equation}
and $P_k(T)=O(k^{-2})$.  With
$\gamma_k=\sqrt\pi P_k(T)$ and
$H_{\rm mode}:=\norm{h_k(T)}_{L^2}=\sqrt\pi\,\vartheta(T)$, the resolvent identity
and Banach perturbation lemma give, for each fixed $k$,
$L_{-dh_k}^{-1}=L_0^{-1}+O_{\mathcal L(Y,X)}(d)$ and
$L_{\widetilde m_{-dh_k}}^{-1}=L_0^{-1}+O_{\mathcal L(Y,X)}(d)$.  Substitution in
Eq.~\eqref{eq:s-midpoint-identity} yields
\begin{equation}
 \delta_{-dh_k}^{\rm mid}=dh_k+d^3r_k+O_X(d^4),
 \qquad
 \norm{C_T\delta_{-dh_k}^{\rm mid}}=dH_{\rm mode}+\gamma_kd^3+O(d^4).
\end{equation}
Because $\gamma_k\to0$, first fix $K$ such that $\gamma_K<\gamma_1$, then choose
$0<c<(\gamma_1-\gamma_K)/(2H_{\rm mode})$.  For
$d_1=\epsilon-c\epsilon^3$, $d_2=\epsilon+c\epsilon^3$ and candidates
$v_1=-d_1h_1$, $v_2=-d_2h_K$, the true score of $v_1$ is smaller, whereas
the diagnostic-score difference has leading coefficient
$\gamma_1-\gamma_K-2cH_{\rm mode}>0$.  Thus, after fixing $K$ and $c$ and then
letting $\epsilon\downarrow0$, a cubic-gap rank reversal occurs.  This is a
sharpness example for the fixed decoder, not a PDE-wide impossibility independent
of residual representation or computational budget.
% --- End inlined file: supplement_sections/s2_burgers.tex ---
% --- Begin inlined file: supplement_sections/s3_extensions.tex ---
\section{Sharp decision calculus and theory-only extensions}
\label{sec:s-sharp-decision-calculus}

\subsection{Support functions, robust selection and regret}
\label{sec:s-support-calculus}

Let $q\in\Hcal$ be a proxy, and write $d_{ij}=y_i-y_j$ for each candidate
pair. Suppose that the unknown target satisfies
\begin{equation}
  P_{\Dcal}^{M}(t-q)\in\mathcal U,
  \label{eq:s-projected-uncertainty}
\end{equation}
where $\mathcal U\subset\Dcal$ is nonempty and compact. For
$d\in\Dcal$, write
\begin{equation}
  \sigma_{\mathcal U}(d)
  =\sup_{z\in\mathcal U}\metricip{z}{d}.
\end{equation}

\begin{proposition}[Exact projected-uncertainty calculus]
\label{prop:s-support-calculus}
For every pair $(i,j)$, the attainable gap set and its interval hull are
\begin{align}
 \{G_{ij}(t):P_{\Dcal}^{M}(t-q)\in\mathcal U\}
 &=G_{ij}(q)+2\{\metricip{z}{d_{ij}}:z\in\mathcal U\},
 \label{eq:s-exact-gap-set}\\
 \operatorname{hull}\{G_{ij}(t)\}
 &=\left[G_{ij}(q)-2\sigma_{\mathcal U}(-d_{ij}),
 G_{ij}(q)+2\sigma_{\mathcal U}(d_{ij})\right].
 \label{eq:s-exact-gap-hull}
\end{align}
Thus candidate $i$ defeats $j$ strictly for every admissible target if and
only if
\begin{equation}
  G_{ij}(q)>2\sigma_{\mathcal U}(-d_{ij}).
  \label{eq:s-robust-pair}
\end{equation}
Candidate $k$ is the unique minimizer for every admissible target if and only
if Eq.~\eqref{eq:s-robust-pair} holds with $i=k$ for every $j\ne k$.

The exact worst-case regret of deploying candidate $k$ is
\begin{equation}
 \sup_{P_{\Dcal}^{M}(t-q)\in\mathcal U}
 \left[R_k(t)-\min_jR_j(t)\right]
 =\max_j\left\{-G_{kj}(q)+2\sigma_{\mathcal U}(-d_{kj})\right\},
 \label{eq:s-exact-robust-regret}
\end{equation}
where the term $j=k$ is zero. A bound at most $\tau$ is therefore necessary
and sufficient for uniform $\tau$-regret over the stated uncertainty set;
failure of the strict or tolerant condition is an abstention.
\end{proposition}

\begin{proof}
The exact squared-gap identity gives
$G_{ij}(t)=G_{ij}(q)+2\metricip{P_{\Dcal}^{M}(t-q)}{d_{ij}}$.
Taking the image, minimum and maximum of this continuous linear functional on
the compact set proves Eqs.~\eqref{eq:s-exact-gap-set} and
\eqref{eq:s-exact-gap-hull}. The strict pair and top-1 statements follow from
the lower endpoint. Finally,
\begin{equation*}
 R_k(t)-\min_jR_j(t)=\max_j[-G_{kj}(t)].
\end{equation*}
The supremum commutes with the finite maximum, and each scalar supremum is the
corresponding support function, proving Eq.~\eqref{eq:s-exact-robust-regret}.
\end{proof}

The attainable set in Eq.~\eqref{eq:s-exact-gap-set} need not fill its interval
hull when $\mathcal U$ is non-convex. The endpoints and all robust conditions
are nevertheless exact. For the ball
$\mathcal U=\{z\in\Dcal:\metricnorm z\le\varepsilon\}$,
$\sigma_{\mathcal U}(d)=\varepsilon\metricnorm d$, and equality is attained
in every nonzero pair direction. This proves sharpness of the familiar
$2\varepsilon\metricnorm{d_{ij}}$ margin radius.

If independent uncertainty sources satisfy
$z_r\in\mathcal U_r$ and add in the comparison coordinate, then their total
set is the Minkowski sum $\mathcal U_1+\cdots+\mathcal U_p$ and
\begin{equation}
 \sigma_{\mathcal U_1+\cdots+\mathcal U_p}(d)
 =\sum_{r=1}^{p}\sigma_{\mathcal U_r}(d).
 \label{eq:s-minkowski-support}
\end{equation}
Physics, algebraic-solve, discretization, branch and operator--data discrepancy
sets can therefore be combined without pretending that any missing component
has been bounded.

\subsection{Bregman and smooth-loss coordinates}

Let $\Phi$ be Fr\'echet differentiable on an open convex subset of a real
Hilbert space and define
\begin{equation}
 D_\Phi(x,y)=\Phi(x)-\Phi(y)-\ip{\nabla\Phi(y)}{x-y}.
\end{equation}
For the first-type loss $R_i(t)=D_\Phi(t,y_i)$,
\begin{align}
 G_{ij}(t)
 &=c_{ij}^{(1)}+\ip{t}{\nabla\Phi(y_i)-\nabla\Phi(y_j)},
 \label{eq:s-bregman-first}\\
 c_{ij}^{(1)}
 &=\Phi(y_i)-\Phi(y_j)
 +\ip{\nabla\Phi(y_j)}{y_j}
 -\ip{\nabla\Phi(y_i)}{y_i}.
\end{align}
Hence the exact primal comparison coordinate is the projection onto
$\operatorname{span}\{\nabla\Phi(y_i)-\nabla\Phi(y_1)\}_{i=2}^N$.
For the second-type loss $R_i(t)=D_\Phi(y_i,t)$,
\begin{equation}
 G_{ij}(t)=\Phi(y_j)-\Phi(y_i)
 -\ip{\nabla\Phi(t)}{y_j-y_i},
 \label{eq:s-bregman-second}
\end{equation}
so the decision coordinate is the projection of the dual variable
$\nabla\Phi(t)$ onto the candidate-difference span. These are exact bisector
coordinates, not claims about every non-Hilbert metric.

The continuous-statistic dimension argument requires an open coordinate
domain. For Eq.~\eqref{eq:s-bregman-first}, it applies when the primal target
domain contains a relatively open set in an affine translate of the active
gradient-difference span. For Eq.~\eqref{eq:s-bregman-second}, it applies only
when the attainable dual domain $\nabla\Phi(\operatorname{dom}\Phi)$ contains
a relatively open set in an affine translate of the active candidate span.
On restricted sets, sign and top-1 minimality still require the corresponding
exposure condition.

More generally, for a twice differentiable loss $\ell(t,y)$, define
$g_{ij}(t)=\ell(t,y_j)-\ell(t,y_i)$. Around $t_0$,
\begin{equation}
 g_{ij}(t_0+h)=g_{ij}(t_0)
 +\ip{h}{\nabla_1\ell(t_0,y_j)-\nabla_1\ell(t_0,y_i)}
 +r_{ij}(h).
\end{equation}
If the gradient difference is $L_{ij}$-Lipschitz on the segment, then
$|r_{ij}(h)|\le(L_{ij}/2)\norm h^2$. The span of these gradient differences
is therefore a local tangent comparison space; unlike the squared-Hilbert and
Bregman formulas above, it is generally not a global exact quotient.

\subsection{Anchor sensitivity, composition and multiple anchors}

Let $\Acal\in C^2$ and define the Newton map
\begin{equation}
 \mathcal N(w)=w-D\Acal(w)^{-1}\Acal(w)
\end{equation}
where the derivative is invertible. With
$\delta_w=-D\Acal(w)^{-1}\Acal(w)$, direct differentiation gives
\begin{equation}
 D\mathcal N(w)[h]
 =-D\Acal(w)^{-1}D^2\Acal(w)[h,\delta_w].
 \label{eq:s-anchor-derivative}
\end{equation}
Consequently, if along the segment between anchors $w$ and $w'$ the inverse
norm, second-derivative norm and correction norm are bounded by
$\beta$, $K$ and $\eta$, respectively, then
\begin{equation}
 \norm{\mathcal N(w')-\mathcal N(w)}
 \le\beta K\eta\norm{w'-w}.
 \label{eq:s-anchor-sensitivity}
\end{equation}
This local estimate explains why anchor changes can be strongly damped near a
regular root but does not select a universally valid anchor.

For an arithmetic state anchor
$w_N=N^{-1}\sum_{i=1}^Nv_i$, augmentation by $v_{N+1}$ gives the exact identity
\begin{equation}
 w_{N+1}-w_N=\frac{v_{N+1}-w_N}{N+1}.
\end{equation}
For any two old task candidates, the resulting corrected proxies
$q_N=Q\mathcal N(w_N)$ and $q_{N+1}=Q\mathcal N(w_{N+1})$ obey
\begin{equation}
 G_{ij}(q_{N+1})-G_{ij}(q_N)
 =2\metricip{q_{N+1}-q_N}{d_{ij}}.
 \label{eq:s-library-augmentation-gap}
\end{equation}
Equations~\eqref{eq:s-anchor-sensitivity}--\eqref{eq:s-library-augmentation-gap}
are composition-sensitivity identities, not a guarantee for the arithmetic
mean or for a medoid.

Finally, suppose $K$ truth-free anchors produce proxies $q_a$ and certified
projected uncertainty sets $\mathcal U_a$ satisfying
$P_{\Dcal}^{M}(t-q_a)\in\mathcal U_a$. Each anchor yields the interval
$I_{ij}^{(a)}$ from Eq.~\eqref{eq:s-exact-gap-hull}, and
\begin{equation}
 G_{ij}(t)\in\bigcap_{a=1}^{K}I_{ij}^{(a)}.
\end{equation}
An intersection lying strictly on one side of zero certifies the pair; one
candidate favoured by all of its intersected pair intervals is certified
top-1. Empty intersections expose inconsistent bounds or implementations.
This multi-anchor construction is theory-only and was not evaluated in the
reported studies.

\section{Operator extensions under finite training observations}
\label{sec:s-operator-extensions}

This section explains why deployment selection is not resolved merely by matching a
finite training record.  Let $K$ be a compact subset of a normed input space and
$\mathfrak G=C(K;X)$.  Suppose training information has the form
\begin{equation}
 \mathcal O(G)=\Psi(O_1G(a_1),\ldots,O_MG(a_M)),
\end{equation}
where $O_j:X\to E_j$ are bounded linear observations and $\Psi$ is arbitrary.
Fix $a_\star\in K$ and define
\begin{equation}
 N_\star=\bigcap_{j:a_j=a_\star}\ker O_j,
\end{equation}
with $N_\star=X$ if $a_\star$ was not observed.

\begin{theorem}[Finite-observation operator lifting]
For $G_0\in C(K;X)$ and $h\in N_\star$, there exists $G_h\in C(K;X)$ such that
\begin{equation}
 \mathcal O(G_h)=\mathcal O(G_0),\qquad
 G_h(a_\star)=G_0(a_\star)+h,
 \qquad
 \norm{G_h-G_0}_{\infty}\le\norm h_X.
\end{equation}
\end{theorem}

\begin{proof}
Let $F_\star=\{a_j:a_j\ne a_\star\}$.  If it is empty, set $\chi\equiv1$.
Otherwise let
$d_\star=\min_{b\in F_\star}\norm{a_\star-b}>0$ and
\[
\chi(a)=\max\{0,1-2\norm{a-a_\star}/d_\star\}.
\]
Set $G_h(a)=G_0(a)+\chi(a)h$.  At observations away from $a_\star$, $\chi=0$;
at $a_\star$, every relevant $O_jh=0$.  Thus all features entering $\Psi$ are
unchanged.  The remaining statements follow from $\chi(a_\star)=1$ and
$0\le\chi\le1$.
\end{proof}

This topological statement explains why finite training observations alone do
not determine deployment behaviour: they permit indistinguishable extensions in
hidden directions.

\subsection{Pairwise adjoints and smooth losses}

For a bounded linear task $C:X\to Z$ and squared loss, define the pairwise truth contrast
\begin{equation}
 J_{ij}(x)=\norm{C(x-v_j)}_Z^2-\norm{C(x-v_i)}_Z^2.
\end{equation}
Here $C$ includes the fixed common-grid lift and the identity or terminal trace,
so the Methods notation is $y_i=Cv_i$ and $t=Cu$ in this linear setting.
The quadratic terms in $Cx$ cancel:
\begin{equation}
 J_{ij}(x)=2\ip{Cx}{C(v_i-v_j)}_Z
 +\norm{Cv_j}_Z^2-\norm{Cv_i}_Z^2.
 \label{eq:s-affine-contrast}
\end{equation}
Thus the unknown truth enters an affine functional even though the individual
losses are quadratic.

\begin{proposition}[Direct pairwise adjoint estimator]
Let $\Acal\in C^1(\mathcal U;Y)$, let $w,u\in\mathcal U$ with
$[w,u]\subset\mathcal U$, and assume $\Acal(u)=0$.  Suppose an adjoint state
$z_{ij,w}\in Y$ satisfies
\begin{equation}
 D\Acal(w)^*z_{ij,w}=2C^*C(v_i-v_j).
\end{equation}
Let an inexact shared correction satisfy
\begin{equation}
 D\Acal(w)\delta_w=-\Acal(w)+\ell_w.
\end{equation}
Define
\begin{equation}
 \widehat\Delta_{ij}=J_{ij}(w)-\ip{z_{ij,w}}{\Acal(w)}_Y
 +\ip{z_{ij,w}}{\ell_w}_Y.
\end{equation}
If
$\norm{D\Acal(x)-D\Acal(w)}\le M\norm{x-w}$ on the segment to $u$, then
\begin{equation}
 |\widehat\Delta_{ij}-J_{ij}(u)|
 \le\norm{z_{ij,w}}_Y\norm{\ell_w}_Y
 +\frac12M\norm{z_{ij,w}}_Y\norm{u-w}_X^2.
\end{equation}
At $w=(v_i+v_j)/2$, $J_{ij}(w)=0$.
\end{proposition}

\begin{proof}
Apply Theorem~\ref{thm:s-nonlinear-identity} to the affine scalar task
Eq.~\eqref{eq:s-affine-contrast}.  Its derivative is constant, the adjoint equation
is exact factorization, and $\int_0^1s\,ds=1/2$. The algebraic solve defect
contributes $\ip{z_{ij,w}}{\ell_w}_Y$, bounded by Cauchy--Schwarz. The
midpoint statement follows by symmetry.
\end{proof}

For the inexact linearized correction
$L_w\delta_w=-\Acal(w)+\ell_w$ and a linear task $C$,
this point estimator is exactly the shared-$q$ gap rather than a competing
diagnostic:
\begin{align}
 J_{ij}(w)-\ip{z_{ij,w}}{\Acal(w)}_Y
 +\ip{z_{ij,w}}{\ell_w}_Y
 &=J_{ij}(w)+\ip{L_w^*z_{ij,w}}{\delta_w}_X \\
 &=J_{ij}(w)+2\ip{C(v_i-v_j)}{C\delta_w}_Z
 =J_{ij}(w+\delta_w).
\end{align}
The last expression is $G_{ij}(q)$ for $q=C(w+\delta_w)$.  Thus one primal
correction computes every linear-task pairwise point estimate simultaneously.
The conservative adjoint envelope reported later is a different construction:
it bounds remainders and was too loose to select models.

For a general smooth loss $\ell:Z\times Z\to\Rnum$, set
$J_{ij}(x)=\ell(Q(x),y_j)-\ell(Q(x),y_i)$.  The same factorization and residual
identity apply to $DJ_{ij}$, but $J_{ij}$ need not be affine and its second
derivative contributes to the remainder.  Consequently the $N-1$ linear
comparison-subspace reduction is special to squared Hilbert distance; the broader
task-factorization layer remains valid.

\subsection{Decoder bias, residual gauge, and computational budget}

Suppose a task decoder along candidate paths $v_{\rho,\theta}=u-\rho\theta$ has
the expansion
\begin{equation}
 p_{\rho,\theta}=\rho C\theta+\rho^kH(\theta)+o(\rho^k),
 \qquad k\ge2.
\end{equation}
If $C\theta\ne0$, its norm expands as
\begin{equation}
 \norm{p_{\rho,\theta}}
 =\rho\norm{C\theta}
 +\rho^k\ip{C\theta/\norm{C\theta}}{H(\theta)}+o(\rho^k).
 \label{eq:s-radial-bias}
\end{equation}
Only the radial component of the leading vector error affects norm ranking at this
order.

\begin{proposition}[Directional rank reversal]
Suppose $\norm{C\theta_1}=\norm{C\theta_2}=a>0$ and the radial coefficients
$b(\theta)$ in Eq.~\eqref{eq:s-radial-bias} satisfy
$b(\theta_1)>b(\theta_2)$.  For any
$0<c<[b(\theta_1)-b(\theta_2)]/(2a)$ and sufficiently small $\rho$, candidates at
distances $\rho-c\rho^k$ along $\theta_1$ and $\rho+c\rho^k$ along $\theta_2$ have
the correct true norm ordering but the reversed proxy ordering.  The true norm gap
is $\Theta(\rho^k)$.
\end{proposition}

\begin{proof}
Insert the two perturbed radii into Eq.~\eqref{eq:s-radial-bias}.  The proxy
coefficients at order $\rho^k$ are $b(\theta_1)-ac$ and
$b(\theta_2)+ac$, while the true norms differ by $2ac\rho^k$ in the opposite
direction.
\end{proof}

This order is decoder-dependent.  For the scalar equation
$\Acal(x)=x+ax^2$ with root zero, a one-step Newton correction from $v=-\rho$ is
$\rho+a\rho^2+O(\rho^3)$.  Under the locally invertible residual reparameterization
$\Phi(y)=y-ay^2$ and $\widetilde\Acal=\Phi\circ\Acal$, the same root and candidate
give $\rho+O(\rho^3)$.  In general,
\begin{equation}
 D^2\widetilde\Acal(u)[h,h]
 =D\Phi(0)D^2\Acal(u)[h,h]
 +D^2\Phi(0)[D\Acal(u)h,D\Acal(u)h].
\end{equation}
High-order rank curvature is therefore not an invariant of the PDE zero set; it
depends on the residual coordinate and decoder.  We use the natural complete PDE
residual and make no coordinate-free optimality claim.

Similarly, without a physics-oracle budget there is no fixed rank barrier: repeated
Newton steps approach the nonlinear solve.  The empirical method fixes one shared
linearized solve per input.  The midpoint hierarchy in
Sections~\ref{sec:s-burgers}--\ref{sec:s-burgers-correction} analyzes
alternative fixed decoders and their computational requirements.

\subsection{Polynomial constraints, nonsmooth regimes, and branches}

If $\Acal$ is a polynomial of degree at most $d$, then
$s\mapsto D\Acal(v+se)$ has degree at most $d-1$.  A $q$-point Gauss--Legendre rule
on $[0,1]$ is exact when $d\le2q$, giving
\begin{equation}
 \Acal(u)-\Acal(v)
 =\left[\sum_{j=1}^{q}\omega_jD\Acal(v+c_je)\right]e.
\end{equation}
The Burgers midpoint identity is the degree-two, one-node case; cubic and quartic
reaction terms require two nodes.  Because the nodes depend on unknown $e$, this is
an algorithmic high-order template rather than a free certificate.

Smooth tangent theory also has hard boundaries.  For the translated step
$u_s(x)=\mathbf1_{x<s}$,
$\norm{u_{s+h}-u_s}_{L^2}=|h|^{1/2}$, so the solution map is not locally Lipschitz
into $L^2$.  In $L^1$ the difference quotient approaches a measure rather than an
$L^1$ tangent.  For physics-informed approximation of conservation laws after
shock formation, strong residual losses may therefore be unsuitable; weak
dual-norm residuals and entropy-admissibility conditions are natural alternatives
\citep{ChaumetGiesselmann2024}.

Finally, residual zero does not identify a chosen branch.  For
$\Acal(x)=x^2-1$, designated truth $u=1$, and candidate $v=-1$, the candidate
residual vanishes although $|u-v|=2$.  A general certificate must either establish
branch uniqueness for the task or include a branch-ambiguity radius.  The three
experimental PDEs use fixed unique numerical branches in the tested regimes.
% --- End inlined file: supplement_sections/s3_extensions.tex ---
% --- Begin inlined file: supplement_sections/s4_numerical_protocol.tex ---
\section{PDE diagnostics and numerical contracts}
\label{sec:s-pde-contracts}

\subsection{Assembled constraint defects}

For Burgers, the state trajectory is sampled on a periodic Fourier grid.  The
defect product contains the full initial mismatch and every time-step PDE defect;
the shared tangent equation is integrated by a fixed integrating-factor RK4
scheme.  The common task grid contains 384 points in the two-PDE studies.  The task
is the terminal field with physical periodic $L^2$ quadrature.

For reaction--diffusion, arrays retain explicit boundary nodes.  Boundary residual
entries equal the predicted Dirichlet values, so the defect includes boundary as
well as interior error.  Interior entries use conservative face coefficients for
$-\nabla\cdot(a\nabla u)$ and the full $\lambda u+\gamma u^3-f$ reaction defect.
The correction on the boundary is imposed exactly; after transferring its
contribution to the interior right-hand side, the Jacobian system with reaction
factor $\lambda+3\gamma w^2$ is solved by preconditioned conjugate gradients.  The
common grid has 128 interior points per direction and uses physical tensor-product
quadrature including boundary nodes.

For Sine--Gordon, trajectories store displacement and velocity.  The complete
discrete defect includes both initial fields and every displacement/velocity step
of velocity Verlet.  Differentiating that map gives an exact lower-triangular
Jacobian action, and forward substitution closes the linear residual.  The task
metric is periodic physical $L^2$ on terminal displacement plus terminal velocity
with weight one.  Candidate fields are Fourier-lifted to the common task grid.

For the PDEBench two-field diffusion--reaction system, trajectories are
cell-centred on a $128\times128$ homogeneous-Neumann grid. The rollout starts at the
last of ten observed frames.  Each released U-Net maps the same ten-frame,
20-channel history to one two-field frame; all three checkpoints are rolled out
autoregressively to 101 frames by dropping the oldest frame and appending the
prediction at every step.  Their different one-step, autoregressive, and
push-forward-20 labels describe training strategies, not different evaluation
horizons.  The complete discrete defect contains the known
initial row and every subsequent trapezoidal-map row for
$u_t=u-u^3-k-v+D_u\Delta u$ and $v_t=u-v+D_v\Delta v$, with
$(D_u,D_v,k)=(0.001,0.005,0.005)$.  The Laplacian matches the released finite-volume
generator, including its one-neighbour boundary rows.  The exact block lower-
triangular Jacobian is solved by forward substitution; each step uses GMRES with a
state-dependent block-diagonal sparse-LU preconditioner, relative tolerance
$10^{-10}$, absolute tolerance $10^{-12}$, and a fail-closed 100-iteration limit.
The task metric is cell-area-weighted squared $L^2$ over both terminal fields.

The raw comparator is defined by the following implementation-level product
norms.  If $r_i^0$ is the initial-condition defect and $r_i^n$ the Burgers PDE
defect at time row $n$, then
\begin{equation}
 S_{i,\mathrm B}^{\mathrm{raw}}
 =\Delta x\sum_x|r_i^0(x)|^2
  +\Delta x\sum_n\Delta t_n\sum_x|r_i^n(x)|^2.
 \label{eq:s-raw-burgers}
\end{equation}
For stationary reaction--diffusion with mesh spacing $h$, the residual weights
are boundary line measure and interior area,
\begin{equation}
 S_{i,\mathrm{RD}}^{\mathrm{raw}}
 =h\sum_{x\in\partial\Omega_h}|r_i(x)|^2
  +h^2\sum_{x\in\Omega_h^\circ}|r_i(x)|^2.
 \label{eq:s-raw-rd}
\end{equation}
Writing $r_{u,i}^0,r_{v,i}^0$ for the two Sine--Gordon initial blocks and
$r_{u,i}^n,r_{v,i}^n$ for its velocity-Verlet map defects gives
\begin{equation}
 S_{i,\mathrm{SG}}^{\mathrm{raw}}
 =\Delta x\!\left(\sum_x(|r_{u,i}^0|^2+|r_{v,i}^0|^2)
 +\sum_{n,x}(|r_{u,i}^n|^2+|r_{v,i}^n|^2)\right).
 \label{eq:s-raw-sg}
\end{equation}
The map defects themselves contain the time-step factors; no additional
time weight is applied.  PDEBench analogously uses
\begin{equation}
 S_{i,\mathrm{PB}}^{\mathrm{raw}}
 =\Delta x\Delta y\left(\sum_{x,k}|r_{i,k}^0(x)|^2
 +\sum_{n,x,k}|r_{i,k}^n(x)|^2\right),
 \label{eq:s-raw-pdebench}
\end{equation}
where $k$ indexes its two fields and $r_i^n$ is the trapezoidal-map defect,
including its time-step factor.  Thus $\|\cdot\|_{Y,a}$ in the Methods is not an
unspecified continuous residual norm: it denotes the corresponding block,
time, and quadrature weighting above.  Every raw-residual comparison uses the same
candidates, inputs and reference states as propagated physics. It
compares two representations of the available physics signal; the computational
costs differ because propagation includes a linearized solve.
The benchmark equations are evaluated in their nondimensional forms.  These
product weights approximate the block integrals but are not optimized
stability or dual norms; sensitivity to alternative residual gauges remains a
separate question.

\subsection{Numerical-reference checks}

The eight-library two-PDE references used independent coarse/fine calculations:
Burgers used integrating-factor RK4 on 768- and 1,536-point Fourier grids, and
reaction--diffusion used 511- and 767-point interior grids.  The
reaction--diffusion Newton solver requested relative tolerance $10^{-11}$ and
absolute tolerance $10^{-13}$; its inner linear solve requested relative tolerance
$10^{-12}$.  Targeted second-method checks covered 12 Burgers and 12
reaction--diffusion inputs selected by fixed error/refinement rules and produced no
decision-changing discrepancy.

The 16-library references used Burgers grids 1,536 and 3,072 and
reaction--diffusion interior grids 767 and 1,023. The
sum-of-candidate-loss refinement rule found no indeterminate
pair among 107,520 and no indeterminate top-1 case among 3,840. Applied to the
eight-library comparison, the same rule flagged three of 53,760 pairs and one of
1,920 top-1 cases; all remained in the reported results.
Sine--Gordon used spectral
velocity-Verlet at 384 spatial points with 6,144 steps and at 768 points with
12,288 steps, together with residual, Jacobian, and closure checks.  Sixteen
development inputs were also compared with adaptive eighth-order Runge--Kutta
(DOP853).  For Sine--Gordon, none of the 53,760 fine-reference pair gaps was
exactly zero or within
$10^{-8}$ of zero.  Numerical uncertainty was assessed separately from this tie
check: a pair was unresolved if its absolute fine-reference gap did not exceed the
sum of the two candidates' absolute coarse-to-fine loss changes.  The top-1 check
applied the same construction to the fine-reference best and runner-up.  Ten pairs
and three top-1 cases were unresolved and were counted as errors.

For PDEBench, terminal states came directly from the published numerical
trajectories. Every extracted history, terminal state, spatial grid and time grid
matched the source. GMRES
used five or six iterations per step, the maximum relative step residual was
$9.999\times10^{-11}$, and the maximum complete linearized-constraint closure was
$2.310\times10^{-10}$.

\section{Candidate generation and evaluation sets}
\label{sec:s-evaluation-populations}

\subsection{Eight-library two-PDE comparison}

Eight cells crossed PDE, FNO/CNO architecture, and ensemble replicate.  Each cell
contained the epoch-20 and epoch-40 checkpoints from four training runs, for 64
candidates from 32 runs.  Admission required finite metrics, in-distribution
relative $L^2$ error at most one, exact hard initial or boundary conditions where
applicable, and complete eight-candidate libraries.  Within-cell in-distribution
maximum/minimum error ratios ranged from 1.48 to 4.13.

For each PDE, 240 evaluation inputs were balanced across three strata: 80 frequency
shifts activated basis modes absent from the training generator; 80 parameter
shifts used disjoint parameter intervals; and 80 combined shifts applied both.
For Burgers on the periodic domain $[0,2\pi]$ and time interval $[0,0.2]$, the
training generator used initial-condition modes 1--4, forcing modes 1--2 and
$\nu\in[0.025,0.075]$. Frequency shift added initial modes 5--10; parameter
shift used $\nu\in[0.008,0.018]\cup[0.085,0.110]$; combined shift applied both.
Initial Fourier coefficients had scales $0.55k^{-1.5}$ for modes 1--4 and
$0.30k^{-0.5}$ for added modes, while forcing coefficients had scale
$0.12k^{-1}$ with temporal modulation at most 0.35 over one or two cycles.

Reaction--diffusion used $L=1$, diffusion field
$a=0.25+\exp(g)$ and sine-basis forcing. The training generator activated the
first five of eight diffusion modes and first six of nine forcing modes, with
coefficient scales 0.12 and 0.80,
$\lambda\in[0.80,1.40]$ and $\gamma\in[0.40,0.90]$. Frequency shift activated
all eight and nine modes with scales 0.16 and 0.90. Parameter shift increased
the scales to 0.26 and 1.35 and used
$\lambda\in[0.30,0.60]\cup[1.80,2.20]$ and
$\gamma\in[0.10,0.25]\cup[1.10,1.40]$; combined shift applied the expanded
basis support and shifted parameters together.

These inputs were disjoint from training, validation and development. The hidden-reference
tie tolerance was $10^{-6}$.  Stored propagated labels used a
$1.01\times10^{-8}$ numerical score tolerance and raw-residual labels used
$10^{-12}$.  The propagated value is the fixed $10^{-10}$ absolute floor plus
a $10^{-8}$ relative floor at unit score scale. A matched-tolerance sensitivity
analysis applies $10^{-6}$ to
both score families.

\subsection{Internally trained candidate contracts}

The Burgers and reaction--diffusion candidate studies used 2,048 training and 256
in-distribution validation samples per run, AdamW for 40 epochs, batch size 32,
weight decay $10^{-4}$, and checkpoints at epochs 20 and 40.  The learning rate
was $2\times10^{-3}$ except for reaction--diffusion CNO, for which an
in-distribution-only pilot selected $10^{-3}$.  Burgers FNO used width 24, six
temporal and 12 spatial modes, four Fourier layers, and projection width 48;
Burgers CNO used three levels, two residual blocks, three neck blocks, circular
padding, and channel multiplier 37 at 48 native spatial points.  The
reaction--diffusion FNO used width 32, 12 modes per axis, four layers, and
projection width 64; its CNO used three levels, two residual blocks, three neck
blocks, zero padding, and channel multiplier 40 at 32 native interior points.
Initial or boundary conditions were enforced by construction.

Sine--Gordon used the same 2,048/256 split, 40 epochs, epoch-20/40 checkpoints,
AdamW learning rate $2\times10^{-3}$, and weight decay $10^{-4}$ on a
$48\times48$ space--time grid.  Its two-output FNO used the Burgers Fourier
width/modes/layers, while its two-output CNO used the Burgers CNO depth and channel
contract; both enforced the two initial phase channels exactly. GPU batch sizes
were 96 for FNO and 48 for CNO.

\subsection{Small-panel model selection}

The small-panel study reused the eight two-PDE libraries.  For each PDE it generated
eight balanced 24-input selection panels (eight inputs per stratum) and one disjoint
240-input deployment population.  Each panel had a fixed balanced ordering, so the
budgets 3, 6, 12, and 24 were nested.  Statistical inference conditions on the
eight libraries and resamples panels within PDE and deployment inputs within
PDE/shift stratum, preserving shared inputs across architecture and ensemble.
We evaluated mean normalized regret, exact and tolerant selection, and top-two
selection across the eight libraries.

\subsection{Sixteen-library comparison}

Four new library replicates were trained for each PDE/architecture cell, yielding
16 libraries and 128 candidates.  Training runs, seeds, checkpoint identities,
training samples, and validation samples were disjoint from the earlier candidate
population.  Primary diagnostics used the full-state geometric medoid; the
arithmetic mean was evaluated as a comparator.  The evaluation comprised 480
instance-wise inputs, 2,304 selection inputs, and 1,920 evaluation inputs. The
selection-panel size was 144.

\subsection{Candidate-composition stress test}

A separately generated set contained 16 libraries, 128 checkpoints, 480
instance-wise inputs, and 3,840 input--library cases; none of its candidates or
inputs belongs to the 16-library set above. With an
arithmetic-mean common state, propagated physics reached 97.9883\%
strict pairwise and 94.5833\% exact top-1 accuracy in aggregate, versus 71.7736\%
and 41.0677\% for raw residual.  One reaction--diffusion CNO library was markedly
less accurate (72.9762\% pairwise and 27.9167\% top-1) because two checkpoints
from one run lay far from the other six under shift.  Retaining the eight-candidate
mean while evaluating the six clustered candidates yielded 54.5278\% pairwise
accuracy; recomputing the shared state on those six yielded 3,600/3,600 pair and
240/240 top-1 decisions.  These results isolate a library-composition effect:
an outlying candidate can alter the decision coordinate even for comparisons in
which it does not participate.  On this same stress population, the geometric
medoid changed the affected CNO cell to 100\% on both metrics and matched the
near-perfect rankings of the other reaction--diffusion cells.  For
Burgers it changed no top-1 decision and reduced the correct-pair count by ten
among 53,760 comparisons.

\subsection{Sine--Gordon comparison}

The Sine--Gordon study used eight candidate libraries and generated 240 new
balanced inputs disjoint from training, validation, reference-development, and
earlier evaluation populations. The libraries were eight FNO/CNO ensembles.
Every indeterminate outcome counted as an error. The study used the geometric-medoid
common state and a fixed phase-space metric. Primary comparisons were strict; a
separate $10^{-8}$ tie check found no fine-reference ties. This
population therefore tests input generalization for a third PDE.

\subsection{Public PDEBench model library}

The candidate library contains the three released U-Net checkpoints for the
two-dimensional diffusion--reaction dataset: one-step, autoregressive and
push-forward-20 training variants. No model was fine-tuned. The U-Net architecture
is common to all candidates, while the training strategy changes. The reported
population comprises identifiers 0901--0999 from the official final split.

Official training code assigns identifiers 0900--0999 to its validation object,
so these inputs were absent from gradient training.

\section{Metrics, aggregation, and uncertainty}
\label{sec:s-metrics}

For an unordered canonical pair $i<j$, the stored gap is the loss of $j$ minus the
loss of $i$, so a positive gap prefers $i$. In the primary two-PDE study,
primary pair accuracy conditions on reference non-ties under absolute tolerance
$10^{-6}$; the all-pair three-way statistic retains the tie label.  Its top-1
metric accepts a selected checkpoint whose hidden loss is within $10^{-6}$ of the
minimum for that input and library; exact top-1 is reported separately. The additional-library study
uses zero-tolerance strict pairwise and exact top-1 metrics.  The Sine--Gordon
primary metrics compare the exact signs of all fine-reference pair gaps and the
exact minimizing checkpoint.  A separate $10^{-8}$ tolerant check found no
fine-reference ties.  Outcomes unresolved by the coarse-to-fine loss-change check
defined above are retained in the denominator and counted as errors.

The PDEBench study uses all three unordered candidate pairs per input under strict
zero-tolerance comparison and exact checkpoint-identity top-1.  Candidate pairs are
within-input decisions, not independent samples. Because every input produced the
same correctness vector, the bootstrap
interval is degenerate and is not interpreted as population uncertainty.  We
report exact descriptive counts for this one fixed library.

A matched-tolerance analysis applies the $10^{-6}$ hidden-reference tolerance to
both diagnostic gaps, separating ranking disagreement from different tie
conventions.

The primary study used 10,000 percentile-bootstrap replicates with seed
80173001 and the paired PDE/input identifier as the cluster.  The small-panel study
used 10,000 hierarchical replicates with seed 84100001, resampling panels within
PDE and deployment inputs within PDE/shift. The additional-library study used 10,000
hierarchical replicates with seed 93090001, preserving paired library replicate,
selection panel, deployment wave and input dependencies. The Sine--Gordon study used
10,000 replicates with seed 95090001 and the shared input identifier across
architectures, replicates and pairs as the cluster. The PDEBench bootstrap
used seed 97090001 and the input identifier as its unit.  Candidate pairs are never
treated as independent sampling units.

For panel-level selection, exact best is the deterministic minimum; tolerant best
uses absolute tolerance $10^{-6}$ in the additional-library study; and top-2 means the
selected checkpoint is among the two lowest deployment-reference risks.  NRegret
divides absolute regret by the larger of the within-library risk span and
$10^{-8}$.  Route-level point means and hierarchical-bootstrap means are reported
as distinct estimators.
% --- End inlined file: supplement_sections/s4_numerical_protocol.tex ---
% --- Begin inlined file: supplement_sections/s5_detailed_results.tex ---
\section{Cell-level and control results}
\label{sec:s-detailed-results}

\subsection{Evaluation sets}

% --- Begin inlined file: tables/table_populations.tex ---
\begin{table}[H]
\centering
\caption{Instance-wise study populations. An input--library case evaluates
one candidate library on one input; populations are analysed separately.}
\label{tab:populations}
\footnotesize
\setlength{\tabcolsep}{4pt}
\begin{tabularx}{\textwidth}{@{}l l r r >{\raggedright\arraybackslash}X@{}}
\toprule
Study & Candidates & Inputs / cases & Pair decisions & Primary rule \\
\midrule
Primary two-PDE & 8 trained libraries ($8$ each) & 480 / 1,920 & 53,760 & $10^{-6}$ reference non-ties; tolerant top-1 \\
Additional libraries & 16 libraries ($8$ each) & 480 / 3,840 & 107,520 & Strict pairs; exact top-1 \\
Sine--Gordon & 8 libraries ($8$ each) & 240 / 1,920 & 53,760 & Strict pairs and top-1; unresolved counted as errors \\
PDEBench & 1 public U-Net library ($3$ models) & 99 / 99 & 297 & Strict pairs; exact top-1 \\
\bottomrule
\end{tabularx}
\end{table}
% --- End inlined file: tables/table_populations.tex ---

\subsection{Pooled instance-wise results}

% --- Begin inlined file: tables/table_instancewise.tex ---
\begin{table}[H]
\centering
\caption{Instance-wise hidden-reference results. Values are estimates with 95\%
clustered-bootstrap intervals; pooled count ratios appear in parentheses when
they differ from the clustered estimator. The first two-PDE comparison uses the $10^{-6}$
tolerant top-1 rule; the 16-library, Sine--Gordon and PDEBench comparisons use exact
checkpoint identity. Exact primary top-1 was $1{,}899/1{,}920$ for propagation and
$688/1{,}920$ for raw residual (Supplementary Section~S5).  PDEBench is an exact
count for one public library. Populations are not pooled.}
\label{tab:instancewise}
\footnotesize
\begin{tabularx}{\textwidth}{@{}l l >{\raggedright\arraybackslash}X
  >{\raggedright\arraybackslash}X@{}}
\toprule
Setting & Method & Pairwise result & Top-1 result \\
\midrule
Two-PDE, eight libraries & Propagated &
$49{,}122/49{,}323=99.5925\%$ [99.5355, 99.6482] &
$1{,}901/1{,}920=99.0104\%$ [98.5938, 99.4271] \\
Two-PDE, eight libraries & Raw residual &
$33{,}457/49{,}323=67.8325\%$ [67.1886, 68.4727] &
$819/1{,}920=42.6563\%$ [40.7813, 44.4792] \\
\addlinespace
Two-PDE, 16 libraries & Propagated &
$99.7063\%$ [99.6187, 99.7870] ($107{,}205/107{,}520=99.7070\%$) &
$98.9549\%$ [98.3333, 99.4792] ($3{,}800/3{,}840=98.9583\%$) \\
Two-PDE, 16 libraries & Raw residual &
$70.4352\%$ [69.5879, 71.2891] ($75{,}733/107{,}520=70.4362\%$) &
$43.4816\%$ [38.9063, 48.0215] ($1{,}671/3{,}840=43.5156\%$) \\
\addlinespace
Sine--Gordon & Propagated &
$52{,}942/53{,}760=98.4784\%$ [98.2013, 98.7333] &
$1{,}835/1{,}920=95.5729\%$ [94.4792, 96.5625] \\
Sine--Gordon & Raw residual &
$27{,}837/53{,}760=51.7801\%$ [50.6622, 52.9018] &
$250/1{,}920=13.0208\%$ [11.4583, 14.6875] \\
\addlinespace
Public PDEBench library & Propagated &
$297/297=100\%$ &
$99/99=100\%$ \\
Public PDEBench library & Raw residual &
$0/297=0\%$ &
$0/99=0\%$ \\
\bottomrule
\end{tabularx}
\end{table}
% --- End inlined file: tables/table_instancewise.tex ---

The eight-library two-PDE columns use hidden-reference non-tie pairwise and tolerant top-1
rules; exact top-1 was 1,899/1,920 for propagation and 688/1,920 for raw
residual, and propagated all-pair three-way accuracy was 91.4732\%.
The 16-library comparison uses strict pairwise and exact top-1; tolerant propagated
top-1 was 3,804/3,840. Within this 16-library population,
propagated-minus-raw differences were
29.2711 percentage points [28.4673, 30.0791] pairwise and 55.4732 [51.0410,
59.9219] top-1.  Sine--Gordon uses conservative strict rules, with ten pair and
three top-1 refinement-indeterminate outcomes retained as errors.  The PDEBench
library has one uniform truth ranking across all 99 inputs.

\subsection{Metric and mechanism sensitivity}

The all-pair three-way accuracy was $49{,}176/53{,}760=91.4732\%$
for propagation and $33{,}457/53{,}760=62.2340\%$ for raw residual.  The
diagnostic tolerances differ because they encode method-specific numerical ties.
Applying the same $10^{-6}$ tolerance to both diagnostic gaps, propagation correctly labeled
$53{,}517/53{,}760=99.5480\%$ of all pairs and recovered 4,413 of the 4,437
reference ties; raw residual gave $33{,}456/53{,}760=62.2321\%$ and recovered
one reference tie.

\begin{table}[H]
\centering
\caption{Mechanism controls on the same 1,920 primary input--library cases.
Pair counts condition on the 49,323 hidden-reference non-ties; top-1 uses the
$10^{-6}$ tolerant winner set.  The wrong-input control cyclically reassigns $q$
within PDE and shift stratum, and the reversed control replaces $w+\delta_w$ by
$w-\delta_w$.}
\label{tab:s-mechanism-controls}
\footnotesize
\begin{tabularx}{\textwidth}{@{}lXrr@{}}
\toprule
Selector & Truth-free construction & Pair correct & Top-1 correct \\
\midrule
ID validation & Input-independent ordering by in-distribution validation loss
  & 24,814/49,323 & 420/1,920 \\
Mean centrality & Squared task-space distance to the candidate mean
  & 28,134/49,323 & 329/1,920 \\
Medoid centrality & Squared task-space distance to the candidate medoid
  & 27,689/49,323 & 335/1,920 \\
Wrong-input propagation & Propagated coordinate from another input in the same stratum
  & 33,885/49,323 & 884/1,920 \\
Reversed propagation & Input-aligned correction with its direction reversed
  & 5,080/49,323 & 59/1,920 \\
\bottomrule
\end{tabularx}
\end{table}

\begin{figure}[H]
  \centering
  \includegraphics[width=0.88\textwidth]{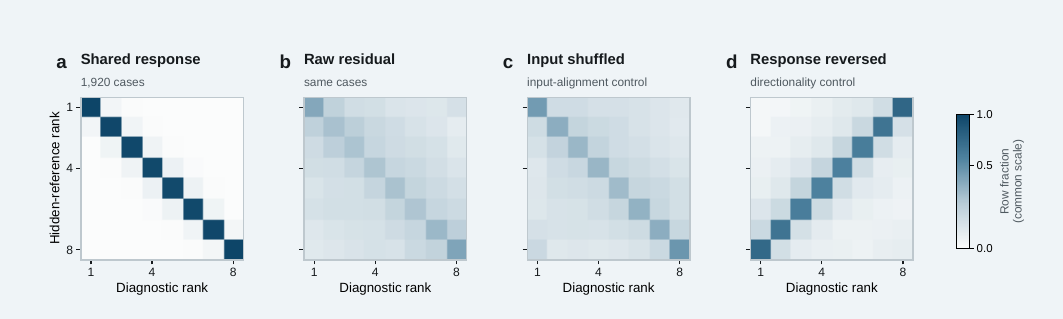}
  \caption{Population-level candidate-rank concordance in the primary
  population. Each tile cross-tabulates a candidate's hidden-reference rank and
  its diagnostic rank for the shared response, raw residual, input-shuffled
  response and reversed response, using one common colour scale. The diagonal
  concentration describes complete-order structure beyond the primary tolerant
  top-1 and non-tie pair endpoints; the tiles are dependent views of the same
  cases.}
  \label{fig:s-rank-concordance}
\end{figure}

For margin stratification, reference non-ties were ranked by absolute truth gap
separately within each of the eight cells and divided into ten equal-count deciles.
Propagated sign accuracy was $97.7913\%$ in the lowest decile and
$99.3917$--$100\%$ in deciles 2--10; raw accuracy ranged from $56.5755\%$ to
$77.3133\%$.  On the 1,920 reconstructed input--library cases, mean
span-normalized regret was $2.4048\times10^{-4}$ for propagation and 0.274746
for raw residual; their medians were 0 and 0.128197, and their maxima were
0.117419 and 1.

The hidden candidate losses also show that selection was not uniformly
large-margin.  In 336/1,920 cases the best--runner-up gap was at most the
primary $10^{-6}$ top-1 tolerance. Relative to the best-to-worst loss span,
the best--runner-up gap had 10th, median, and 90th percentiles 0.0260, 0.1616,
and 0.4777.
The task-loss scale also varied strongly by PDE.  For Burgers, the best-candidate
loss had median 0.10287 and 90th percentile 0.35209; for reaction--diffusion the
corresponding values were $2.3621\times10^{-4}$ and $6.1839\times10^{-3}$.
The median best-to-worst candidate spans were $3.1403\times10^{-3}$ and
$1.8387\times10^{-4}$.  Mean loss of the propagated coordinate $q$ was
$1.6427\times10^{-3}$ for Burgers and $7.1278\times10^{-9}$ for
reaction--diffusion, versus mean oracle-best candidate losses 0.14441 and
0.0021712, respectively.

\subsection{Eight-library two-PDE results}

\begin{table}[H]
\centering
\caption{Results for the eight two-PDE libraries. Pair columns condition on hidden-reference
non-ties; top-1 is tolerant.  Mean regret is the unnormalised excess hidden task
loss of the selected checkpoint over the best candidate, averaged across inputs.
Each cell has 240 inputs and at most 6,720 pair opportunities.}
\label{tab:s-primary-cells}
\small
\begin{tabular}{@{}lrrrrr@{}}
\toprule
Cell & Prop. pair & Raw pair & Prop. top-1 & Raw top-1 & Prop. mean regret \\
\midrule
Burgers CNO e1 & 98.7928 & 71.7288 & 96.6667 & 48.7500 & $2.37\times10^{-6}$ \\
Burgers CNO e2 & 98.8074 & 72.4806 & 96.6667 & 42.5000 & $3.24\times10^{-6}$ \\
Burgers FNO e1 & 99.7283 & 86.8246 & 99.5833 & 71.6667 & $1.63\times10^{-7}$ \\
Burgers FNO e2 & 99.6694 & 89.5116 & 99.1667 & 67.9167 & $6.39\times10^{-6}$ \\
RD CNO e1 & 100.0000 & 53.7785 & 100.0000 & 10.8333 & 0 \\
RD CNO e2 & 100.0000 & 33.4425 & 100.0000 & 0.0000 & 0 \\
RD FNO e1 & 100.0000 & 65.7546 & 100.0000 & 48.7500 & 0 \\
RD FNO e2 & 100.0000 & 69.3808 & 100.0000 & 50.8333 & $3.21\times10^{-12}$ \\
\bottomrule
\end{tabular}
\end{table}

The reaction--diffusion FNO cells contain most reference ties: their non-tie
denominators are 4,532 and 4,667 rather than 6,720.  The other six cells contribute
196 additional ties, giving the pooled non-tie denominator 49,323.  The
all-pair three-way sensitivity is reported above and in the main paper.

\subsection{Sixteen-library two-PDE results}

\begin{table}[H]
\centering
\caption{Results for the 16 two-PDE libraries. Mean regret is
the unnormalised excess hidden task loss of the selected checkpoint over the best
candidate, averaged across inputs.  Each cell has 240 inputs and 6,720 strict pair
opportunities.  Strict top-1 is exact checkpoint agreement.}
\label{tab:s-fresh-library-cells}
\footnotesize
\begin{tabular}{@{}lrrrrr@{}}
\toprule
Cell & Prop. pair & Raw pair & Prop. top-1 & Raw top-1 & Prop. mean regret \\
\midrule
Burgers CNO r1 & 99.4345 & 78.5565 & 97.5000 & 54.1667 & $2.66\times10^{-6}$ \\
Burgers CNO r2 & 99.1815 & 78.7351 & 97.9167 & 57.9167 & $1.26\times10^{-6}$ \\
Burgers CNO r3 & 99.2857 & 77.2917 & 98.7500 & 65.4167 & $4.16\times10^{-7}$ \\
Burgers CNO r4 & 98.8690 & 73.5268 & 95.4167 & 38.7500 & $3.42\times10^{-6}$ \\
Burgers FNO r1 & 99.8512 & 91.5476 & 99.5833 & 65.4167 & $7.86\times10^{-7}$ \\
Burgers FNO r2 & 99.6875 & 88.0952 & 98.3333 & 76.2500 & $6.39\times10^{-6}$ \\
Burgers FNO r3 & 99.7024 & 88.7500 & 99.1667 & 67.0833 & $2.45\times10^{-7}$ \\
Burgers FNO r4 & 99.6131 & 89.8661 & 98.3333 & 72.9167 & $1.50\times10^{-6}$ \\
RD CNO r1 & 100.0000 & 51.3244 & 100.0000 & 23.3333 & 0 \\
RD CNO r2 & 100.0000 & 53.7649 & 100.0000 & 9.5833 & 0 \\
RD CNO r3 & 100.0000 & 55.2827 & 100.0000 & 25.8333 & 0 \\
RD CNO r4 & 100.0000 & 51.5476 & 100.0000 & 40.8333 & 0 \\
RD FNO r1 & 99.8958 & 64.0179 & 99.1667 & 27.9167 & $5.94\times10^{-11}$ \\
RD FNO r2 & 99.9107 & 60.8780 & 99.5833 & 18.7500 & $9.61\times10^{-12}$ \\
RD FNO r3 & 99.9702 & 60.5357 & 100.0000 & 24.1667 & 0 \\
RD FNO r4 & 99.9107 & 63.2589 & 99.5833 & 27.9167 & $7.01\times10^{-13}$ \\
\bottomrule
\end{tabular}
\end{table}

The pooled estimators are clustered-bootstrap means rather than ratios of the cell
totals. Pooled count ratios are 107,205/107,520 pairwise and 3,800/3,840
top-1; the bootstrap means are 99.7063\% and 98.9549\%.

\subsection{Sine--Gordon cells}

Two auxiliary input populations quantify reference-grid sensitivity.  At the
initial reference resolution, propagation reached 97.9018\% pairwise and
95.0521\% top-1 accuracy versus 51.9773\% and 12.2917\% for raw defect energy;
166 pair and 13 top-1 comparisons were numerically indeterminate.  Reference-grid
refinement reduced these counts and yielded 98.3612\% and 96.4063\% versus
51.6648\% and 13.0729\%, with four indeterminate top-1 cases.  These auxiliary
populations are not pooled with the disjoint population below, whose ten pair and
three top-1 indeterminate outcomes are conservatively counted as errors.

\begin{table}[H]
\centering
\caption{Sine--Gordon conservative cell results.  Reference-indeterminate events
are retained as errors.  Each cell has 240 inputs and 6,720 pair opportunities.}
\label{tab:s-sine-gordon-cells}
\small
\begin{tabular}{@{}lrrrrrr@{}}
\toprule
Cell & Prop. pair & Raw pair & Prop. top-1 & Raw top-1 & Pair indet. & Top indet. \\
\midrule
CNO r1 & 99.1667 & 44.3006 & 97.0833 & 20.0000 & 0 & 0 \\
CNO r2 & 98.8542 & 51.0268 & 95.8333 & 8.3333 & 1 & 0 \\
CNO r3 & 98.7798 & 51.6071 & 97.0833 & 7.9167 & 1 & 0 \\
CNO r4 & 99.3899 & 47.8869 & 98.3333 & 13.7500 & 2 & 0 \\
FNO r1 & 97.9167 & 57.4405 & 94.5833 & 15.4167 & 0 & 0 \\
FNO r2 & 97.7530 & 48.0804 & 92.5000 & 11.2500 & 2 & 0 \\
FNO r3 & 97.6488 & 57.9613 & 93.7500 & 12.9167 & 2 & 2 \\
FNO r4 & 98.3185 & 55.9375 & 95.4167 & 14.5833 & 2 & 1 \\
\bottomrule
\end{tabular}
\end{table}

\subsection{Public PDEBench model library}

\begin{table}[H]
\centering
\caption{Descriptive results for the fixed public PDEBench U-Net library on
inputs 0901--0999.  Each input contributes three strict pair comparisons.  No
population interval is attached because candidate ordering and method correctness
are constant across inputs.  Mean regret is the unnormalised excess terminal task
loss over the best checkpoint, averaged across inputs.}
\label{tab:s-pdebench-public}
\small
\begin{tabular}{@{}lrrrl@{}}
\toprule
Method & Strict pair & Exact top-1 & Mean regret & Selected checkpoint \\
\midrule
Propagated physics & $297/297$ & $99/99$ & 0 & Autoregressive \\
Raw residual & $0/297$ & $0/99$ & 0.282935 & One-step \\
Full-trajectory-medoid centrality & $198/297$ & $0/99$ & 0.141394 & Push-forward-20 \\
\bottomrule
\end{tabular}
\end{table}

The true loss ordering was autoregressive, push-forward-20, then one-step on all
99 inputs; the smallest best--runner-up loss gap was 0.108595.  The propagated
scores had the same ordering and the raw scores had the reverse ordering.  The
propagated terminal estimate had mean truth loss 0.004332,
versus 0.010511 for the oracle-best checkpoint, and was strictly better on every
input.  These values use cell-area-weighted squared $L^2$ over both terminal
fields.  The corresponding excluded development input 0900 had the same candidate
ordering, so a constant ordering copied from development would also be perfect on
this population; the experiment supports the physics score on this fixed external
library, not input-adaptive routing.

\subsection{Selection-panel size}

\begin{table}[H]
\centering
\caption{Route-level results across nested selection-panel sizes. The main
comparison uses $B=24$; the full curve shows the observed panel-size dependence.}
\label{tab:s-panel-size}
\small
\begin{tabular}{@{}rrrrrrrrr@{}}
\toprule
$B$ & \multicolumn{4}{c}{Propagated physics} & \multicolumn{4}{c}{Raw residual} \\
\cmidrule(lr){2-5}\cmidrule(lr){6-9}
& NRegret & Exact & Tol. & Top-2 & NRegret & Exact & Tol. & Top-2 \\
\midrule
3  & 0.1269 & 37 & 40 & 46 & 0.2660 & 30 & 32 & 33 \\
6  & 0.1335 & 35 & 39 & 42 & 0.2506 & 27 & 30 & 31 \\
12 & 0.0481 & 46 & 50 & 54 & 0.2148 & 29 & 33 & 37 \\
24 & 0.0351 & 49 & 52 & 55 & 0.1866 & 31 & 37 & 39 \\
\bottomrule
\end{tabular}
\end{table}

Performance varied strongly by library. Burgers CNO ensemble 2 exceeded raw
defect energy by 0.07289 mean normalized regret and selected a top-two model in
only one of eight panels. At group level, propagated versus raw mean normalized regret was
0.007526 versus 0 for Burgers FNO and 0.070986 versus 0.070411 for Burgers CNO,
showing no improvement in either Burgers architecture group.

Combinations of the 24-input panels were also examined at
$B=48,72,96,144,$ and 192. Propagated mean NRegret
decreased from 0.0351 at $B=24$ to 0.0049 at $B=144$ and 0.0030 at $B=192$;
exact-best frequency was not monotone, changing from 0.7656 to 0.9062 and 0.8750,
respectively. These recombinations show the observed panel-size dependence in
this population.

\subsection{Candidate-specific correction control}

A subset used the first eight input identifiers in each PDE/shift stratum,
producing 48 unique inputs and 192 input--library cases. For pair comparisons,
the propagated method recomputed an anchor and response after holding
out both candidates in each pair; it is therefore a pair-cross-fit estimator,
whereas top-1 used
the ordinary one-response shared construction on the same cases. Both were
compared with a control that solves one linearized problem per candidate. Each
candidate-specific solve produces that candidate's corrected task estimate, and
the control score is its squared task-metric distance from the original candidate
output:
$S_i^{\mathrm{cand}}=\lVert Q_a(v_i)+DQ_a(v_i)\delta_i-y_i\rVert_M^2$.

\begin{table}[H]
\centering
\caption{Correction control on the deterministic subset. For the propagated
row (${}^\ast$), the pair entry is pair-cross-fit whereas the top-1 entry uses the ordinary shared
response; the two columns therefore have different propagated constructions.
The candidate-specific pair interval is a 95\% input-cluster bootstrap interval.}
\label{tab:s-candidate-specific}
\small
\begin{tabular}{@{}lcc@{}}
\toprule
Method & Non-tie pairwise & Tolerant top-1 \\
\midrule
Propagated${}^\ast$ & $5{,}007/5{,}031=99.5230\%$ & $190/192=98.9583\%$ \\
Candidate-specific & $5{,}010/5{,}031=99.5826\%$ [99.3993, 99.7616] & $190/192=98.9583\%$ \\
\bottomrule
\end{tabular}
\end{table}

The ordinary shared and candidate-specific methods were top-1 correct on the
same number of cases, while the pair-cross-fit and candidate-specific pair
accuracies differed by three decisions. The pair result is not a one-response
versus eight-response cost comparison. At both $B=12$ and $B=24$, the ordinary
shared and candidate-specific selectors had identical aggregate exact-best,
tolerant-best, top-2 and mean-regret results across the 64 routes.

A separate control used 16-candidate libraries. For each PDE and architecture,
the two eight-candidate ensembles were interleaved; the first 16 input identifiers
in each of the three strata yielded 192 input--library cases. Thus the population contains
the first-eight-per-stratum subset used above and adds the next eight identifiers
per stratum.  The shared construction used the arithmetic mean of all 16
candidates. The same candidates and inputs were used for shared,
candidate-specific and raw-residual scoring.

\begin{table}[H]
\centering
\caption{Pooled-16 accuracy control.  Pairwise accuracy conditions on
hidden-reference non-ties; top-1 uses the $10^{-6}$ tolerant winner set.  The
raw-residual row uses the same candidates and inputs.}
\label{tab:s-pooled16-accuracy}
\small
\begin{tabular}{@{}lcc@{}}
\toprule
Method & Non-tie pairwise & Tolerant top-1 \\
\midrule
Shared propagated & $21{,}378/21{,}458=99.6272\%$ & $189/192=98.4375\%$ \\
Candidate-specific & $21{,}370/21{,}458=99.5899\%$ & $190/192=98.9583\%$ \\
Raw residual & $14{,}235/21{,}458=66.3389\%$ & $67/192=34.8958\%$ \\
\bottomrule
\end{tabular}
\end{table}

Shared and candidate-specific correction selected the same checkpoint in
191/192 cases.  Their 89 pair-label disagreements favoured the shared method by
48 correct decisions to 40; the remaining disagreement was a hidden-reference
tie and is excluded from the non-tie denominator.
Reaction--diffusion was perfect for both propagated methods; their remaining
errors occurred in the two Burgers pools.

\subsection{Truth-covariant flips and direct use of \texorpdfstring{$q$}{q}}

\begin{table}[H]
\centering
\caption{Checkpoint-re-inference flip decomposition over 53,760 unordered pair
opportunities in the primary two-PDE study.}
\label{tab:s-flips}
\begin{tabular}{@{}lr@{}}
\toprule
Category & Count \\
\midrule
No flip & 37,782 \\
Diagnostic and truth flip & 15,534 \\
Diagnostic-only flip & 256 \\
Truth-only flip & 188 \\
\bottomrule
\end{tabular}
\end{table}

Among 15,790 strict diagnostic flips, 98.3787\% were accompanied by a truth flip.
In the primary population, the shared
corrected output $q$ also matched or improved on the best checkpoint's task loss
in all 1,920 input--library cases. This observation concerns
the reported task coordinate; panel-level deployment instead selects a reusable
checkpoint.

\subsection{Component cost accounting}

The shared diagnostic uses one correction solve per input/library, whereas the
candidate-specific construction uses one solve per candidate. The pooled-16
accuracy control is a direct one-versus-16 comparison; the earlier $N=8$ pair
control is pair-cross-fit and is not used for this cost claim. A timing study
formed one 16-candidate library for each PDE and architecture by interleaving the
two eight-candidate ensembles. The first input in each of three strata was
measured ten times after two warm-ups,
giving 60 observations per PDE, method, and candidate count.

\begin{table}[H]
\centering
\caption{Median measured physics-diagnostic component time in seconds for the
pooled-library scaling study. S denotes one shared solve and C one solve per
candidate.  These measurements are separate from the eight-candidate accuracy
control and use fewer inputs than the pooled-16 accuracy control.}
\label{tab:s-pooled16-timing}
\small
\begin{tabular}{@{}rcccc@{}}
\toprule
& \multicolumn{2}{c}{Burgers} & \multicolumn{2}{c}{Reaction--diffusion} \\
\cmidrule(lr){2-3}\cmidrule(lr){4-5}
$N$ & S & C & S & C \\
\midrule
2  & 0.653 & 1.293 & 0.0472 & 0.0930 \\
4  & 0.551 & 2.164 & 0.0475 & 0.1887 \\
8  & 0.540 & 4.337 & 0.0483 & 0.3764 \\
16 & 0.541 & 8.664 & 0.0486 & 0.7455 \\
\bottomrule
\end{tabular}
\end{table}

These measurements characterize within-run scaling of the physics components
after candidate inference. End-to-end latency additionally depends on model
inference, data movement, scheduling and the surrounding serving system. The
archived environment records Python
3.12.13, NumPy 2.5.1, SciPy 1.18.0, and one-thread OpenMP, MKL, OpenBLAS, and
NumExpr settings.

\begin{figure}[H]
  \centering
  \includegraphics[width=0.96\textwidth]{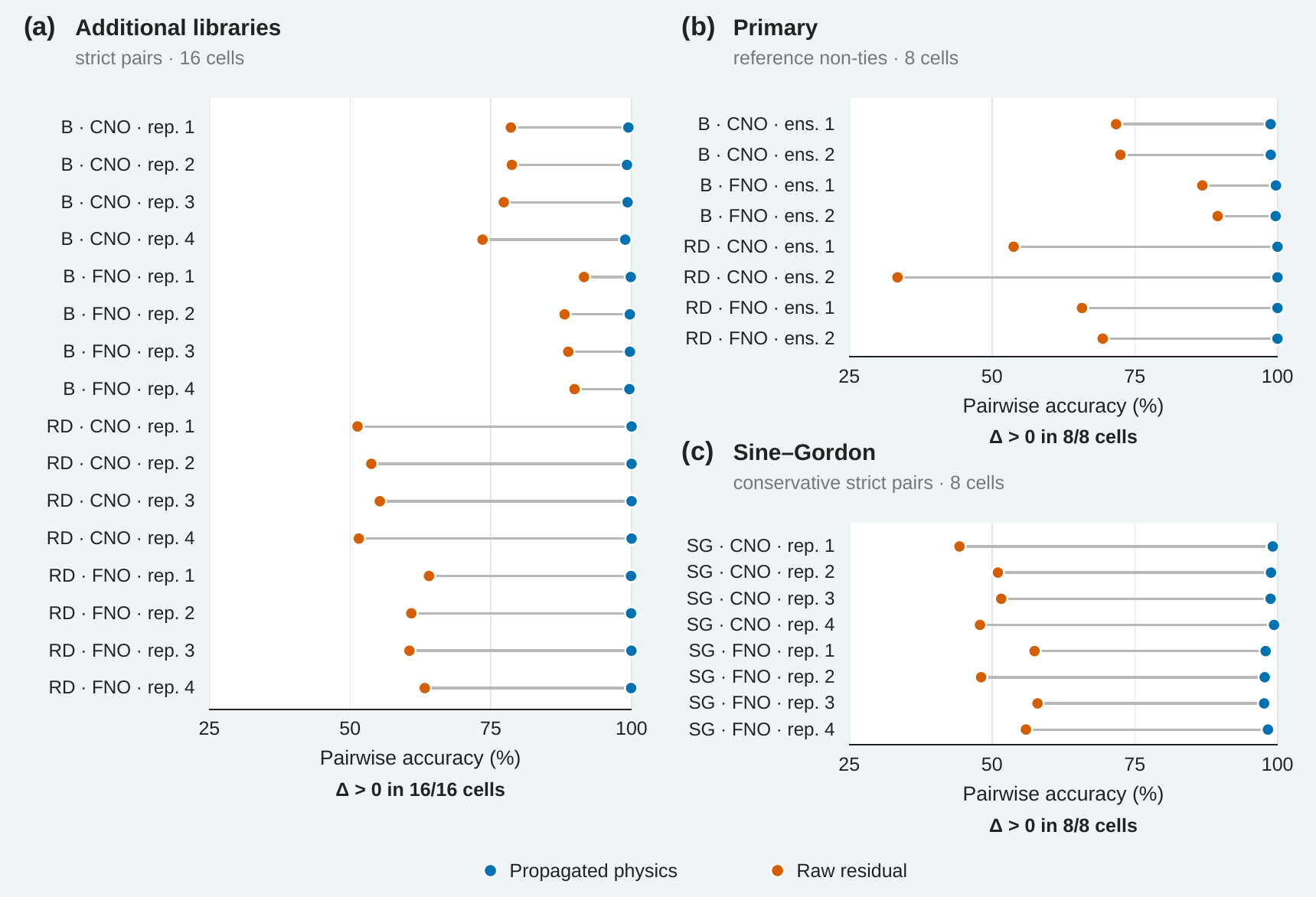}
  \caption{Complete within-cell pairwise contrasts on matched candidates and
  inputs. Panels show 16 additional-library cells under strict pairs, eight
  primary cells under reference non-ties, and eight Sine--Gordon cells under
  conservative strict pairs.  B, RD and SG denote the three PDEs; ens. and rep.
  identify ensembles and library replicates. Shared inputs induce
  dependence, so no cellwise intervals are shown and the three population
  estimands remain separate.}
  \label{fig:s-cellwise-advantage}
\end{figure}
% --- End inlined file: supplement_sections/s5_detailed_results.tex ---
% --- Begin inlined file: supplement_sections/s6_cross_resolution.tex ---
\section{Cross-resolution and locality analyses}
\label{sec:s-cross-resolution}

\subsection{Fixed candidates}

\begin{table}[H]
\centering
\caption{Fixed-candidate grid transitions.  The pair column counts comparisons
that are non-ties on at least one endpoint grid; pairs tied on both grids are
excluded from strict-flip rates.  Strict flips compare positive/negative ordering,
whereas tolerance transitions also include movement into or out of the
pre-specified tie region.}
\label{tab:s-fixed-grid}
\small
\begin{tabular}{@{}lrrrrr@{}}
\toprule
Transition & Cells & Pairs & Strict flips & Tol. transitions & Top-1 changes \\
\midrule
Burgers 192$\to$384 & 4 & 26,874 & 0 & 0 & 0 \\
RD 64$\to$128 & 4 & 26,835 & 10 & 63 & 2 \\
RD 128$\to$256 & 4 & 26,818 & 0 & 25 & 0 \\
RD 256$\to$512, new inputs & 4 & 26,793 & 0 & 6 & 0 \\
\bottomrule
\end{tabular}
\end{table}

For the new-input reaction--diffusion 256-to-512 transition, the four cells
contained 0, 1, 4, and 1 tolerance-label changes, respectively, and no cell changed
top-1.  Every changed pair lay within the bottom 0.70\% of its cell's margin
distribution.  These results describe the tested finite sequence; a
continuum rate would require additional resolution levels and an asymptotic
argument.

\subsection{Exact remainder populations}

All coarse and fine outputs in this analysis are first lifted to the same Hilbert
task space with the fixed metric $M$.  Write $q_h,y_{i,h},y_{j,h}$ for the
coarse-grid diagnostic and candidate outputs and $q_H,y_{i,H},y_{j,H}$ for
their fine-grid counterparts.  For $g\in\{h,H\}$, define the diagnostic gap
and its pair-centred factors by
\begin{align}
  G_{ij,g}
  &=\|q_g-y_{j,g}\|_M^2-\|q_g-y_{i,g}\|_M^2,
  \label{eq:s-cross-grid-gap}\\
  c_{ij,h}&=q_h-\frac{y_{i,h}+y_{j,h}}{2},
  &p_{ij,h}&=y_{i,h}-y_{j,h}.
  \label{eq:s-cross-grid-centre-separation}
\end{align}
Thus $G_{ij,h}=2\langle c_{ij,h},p_{ij,h}\rangle_M$.  Introduce the observed
grid-to-grid drifts
\begin{equation}
  d_q=q_H-q_h,
  \qquad d_i=y_{i,H}-y_{i,h},
  \qquad d_j=y_{j,H}-y_{j,h},
  \label{eq:s-cross-grid-drifts}
\end{equation}
and set
\begin{equation}
  \delta c_{ij}=d_q-\frac{d_i+d_j}{2},
  \qquad
  \delta p_{ij}=d_i-d_j.
  \label{eq:s-cross-grid-centred-drifts}
\end{equation}
The term linear in these drifts and the remaining bilinear term are
\begin{align}
  \widehat G_{ij,H}^{(1)}
  &=G_{ij,h}
    +2\langle c_{ij,h},\delta p_{ij}\rangle_M
    +2\langle\delta c_{ij},p_{ij,h}\rangle_M,
  \label{eq:s-cross-grid-first-gap}\\
  R_{ij}^{(2)}
  &=2\langle\delta c_{ij},\delta p_{ij}\rangle_M.
  \label{eq:s-cross-grid-remainder}
\end{align}
Since the fine-grid centred factors are $c_{ij,h}+\delta c_{ij}$ and
$p_{ij,h}+\delta p_{ij}$, direct expansion gives the exact identity
\begin{equation}
  G_{ij,H}=\widehat G_{ij,H}^{(1)}+R_{ij}^{(2)}.
  \label{eq:s-cross-grid-exact-identity}
\end{equation}
Here ``first order'' and ``second order'' refer to algebraic degree in the
observed drifts; Eq.~\eqref{eq:s-cross-grid-exact-identity} makes no small-drift
approximation.

For completeness, let $\eta_{\rm num}\geq0$ denote the numerical
closure allowance and put
\begin{equation}
  I_{ij,H}=
  \bigl[\widehat G_{ij,H}^{(1)}-|R_{ij}^{(2)}|-\eta_{\rm num},
        \widehat G_{ij,H}^{(1)}+|R_{ij}^{(2)}|+\eta_{\rm num}\bigr].
  \label{eq:s-cross-grid-interval}
\end{equation}
The strict pair screen passes when $0\notin I_{ij,H}$, equivalently
$|\widehat G_{ij,H}^{(1)}|>|R_{ij}^{(2)}|+\eta_{\rm num}$; its interval side
then fixes the sign of $G_{ij,H}$.  With a pair tolerance $\tau_{ij}\geq0$,
the screen returns the positive, negative or tied label only when
$I_{ij,H}$ lies wholly in $(\tau_{ij},\infty)$,
$(-\infty,-\tau_{ij})$ or $[-\tau_{ij},\tau_{ij}]$, respectively.  A strict
top-1 decision passes when one first-order winner has a strict certified pair
interval on its side against every competitor.  All remaining pair and top-1
decisions abstain.  The certified object is therefore the corresponding
fine-grid diagnostic decision $G_{ij,H}$.

The first population comprised eight two-PDE libraries on 480 inputs; the second
reused the same inputs and candidates after checkpoint re-inference at both grids. These
populations each contain 53,760 pairs.  The exact identity closes on every row, with maximum
errors $9.1593\times10^{-16}$ and $9.0206\times10^{-16}$.  The candidate-only
expression that omits shared-diagnostic drift fails on 5,729 and 5,819 rows,
respectively.  Pooled pair-screen coverages are 65.4464\% and 65.0800\%;
the architecture-resolved values are reported in
Supplementary Table~\ref{tab:s-certificate} and summarized graphically in
Supplementary Fig.~\ref{fig:s-reliability-screen}.  Exact
interval separation determines every covered diagnostic decision, so coverage is
the empirical quantity of interest.

% --- Begin inlined file: tables/table_reliability.tex ---
\begin{table}[H]
\centering
\caption{Exact cross-resolution diagnostic stability. Coverage is the fraction of
fine-grid diagnostic decisions passing the exact stability screen; all uncovered
decisions abstain.  Agreement of every covered decision follows algebraically and
does not measure hidden-reference ranking.}
\label{tab:s-certificate}
\small
\begin{tabular}{@{}l l r r r r@{}}
\toprule
Population & Group & Pairs & Pair coverage & Top-1 coverage & Max closure error \\
\midrule
Earlier & FNO & 26,880 & 98.1734\% & 93.8542\% & $4.72\times10^{-16}$ \\
Earlier & CNO & 26,880 & 32.7195\% & 6.0417\% & $9.16\times10^{-16}$ \\
Checkpoint re-inference & FNO & 26,880 & 98.0394\% & 93.4375\% & $9.02\times10^{-16}$ \\
Checkpoint re-inference & CNO & 26,880 & 32.1205\% & 5.6250\% & $9.02\times10^{-16}$ \\
\bottomrule
\end{tabular}
\end{table}
% --- End inlined file: tables/table_reliability.tex ---

\begin{figure}[H]
  \centering
  \includegraphics[width=0.72\textwidth]{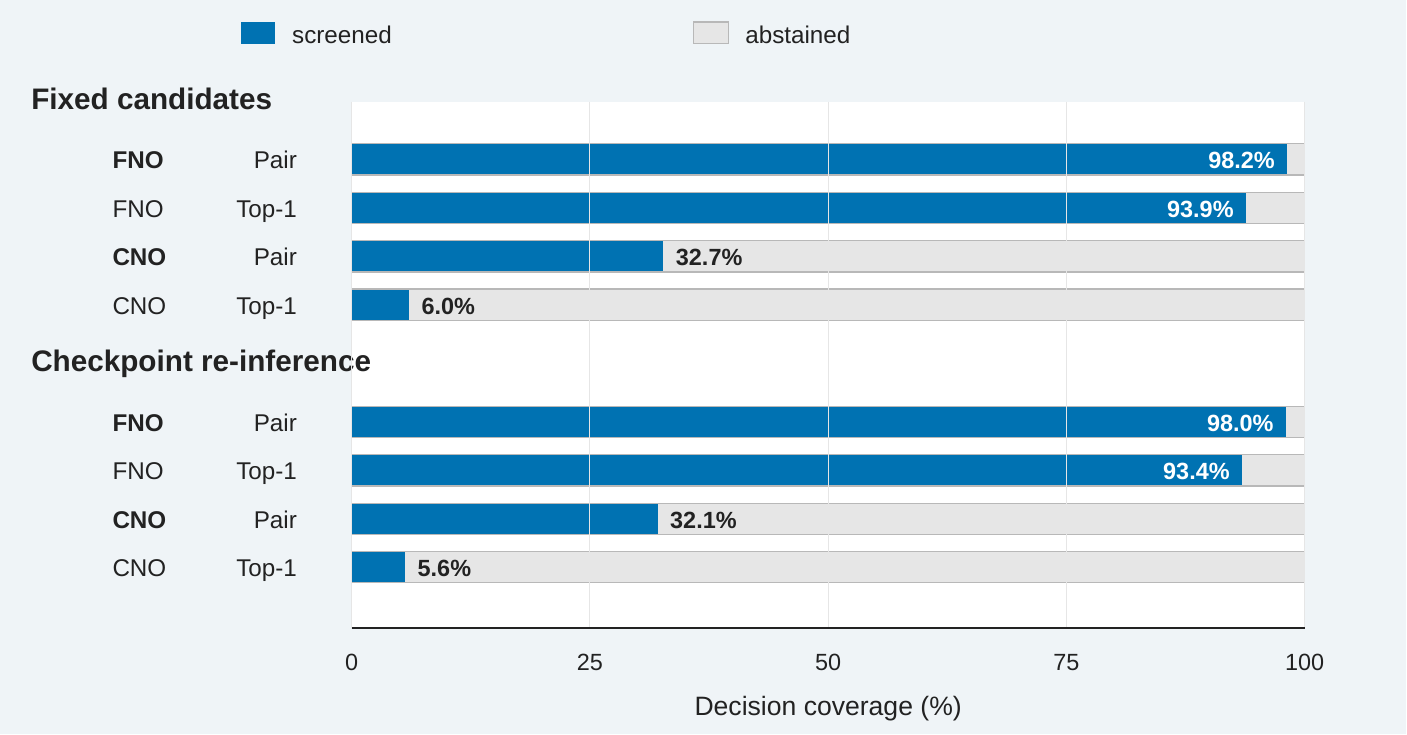}
  \caption{Coverage of the exact cross-resolution rank-or-abstain stability
  screen in two distinct cross-resolution populations.  Each contains 26,880 pair
  decisions and 960 top-1 cases per architecture.  Blue decisions pass the
  screen and grey decisions abstain.  Coverage is high for FNO and low for CNO,
  especially for top-1. The screen measures diagnostic stability between two
  computed grids.}
  \label{fig:s-reliability-screen}
\end{figure}

\subsection{Locality and fragility}

An exploratory first-order gap-change predictor, formed from the three terms that
are linear in the shared and candidate drifts, was compared with a margin-only flip
predictor on the checkpoint-re-inference population.  Pooled FNO AUROC improved
by 0.0662, with 95\% CI $[0.0548,0.0781]$, and AUPRC increased from 0.3639 to
0.5060.  For CNO, AUROC changed by $-0.0342$ (95\% CI
$[-0.0401,-0.0283]$) and AUPRC fell from 0.5724 to 0.5455.  All four FNO cell
increment intervals were positive, while both reaction--diffusion CNO intervals
were negative.  Across eight cells, the Spearman association between centered
locality ratio and AUROC increment was $-0.9524$, with locality confounded by
architecture.

Median centered locality ratios ranged from 0.010--0.200 for FNO, 0.995--0.998 for
Burgers CNO, and 3.31--4.46 for reaction--diffusion CNO.  Error-relative medians
ranged from 0.075--0.326, 1.44--1.53, and 20.8--22.8, respectively.  The opposite
FNO and CNO trends preclude a common empirical threshold in these data.  The exact
pair-specific sign control remains
$|R_{ij}^{(2)}|/|\widehat G_{ij,H}^{(1)}|$.

\subsection{Other auxiliary analyses}

Principal angles between the candidate-difference subspace and its cross-grid
drift subspace did not by themselves explain pair flips.  An SVD of the same
task-difference coordinates placed 95.5\% of energy in low modes, but ordinary
and near-tie cases split across subsequent modes, so this did not establish an
effective rank.  A conservative envelope obtained from pairwise task-adjoint
residual estimates had a median envelope-to-gap ratio of approximately
$1.98\times10^5$, far too loose for useful decisions.  A sequential pilot that
added candidate-comparison directions until a decision stabilized retained
accuracy only after computing every expensive probe and the full solve, so it
produced no cost reduction.  These auxiliary results motivate direct pair-specific
intervals rather than a universal locality, rank, or adjoint surrogate.
% --- End inlined file: supplement_sections/s6_cross_resolution.tex ---
% --- Begin inlined file: supplement_sections/s7_rd_certificate.tex ---
\section{Operator-aligned selective certification for cubic reaction--diffusion}
\label{sec:s-rd-certificate}

This section gives the complete exact-arithmetic theorem for the strongly monotone
reaction--diffusion discretization and evaluates its decision margins
numerically. The computation includes algebraic solve-defect inflation but does
not enclose floating-point roundoff.

\subsection{Typed discrete setting}

Let a uniform grid of spacing $h=L/(m+1)$ cover $(0,L)^2$.  The state space
$V_h=\mathbb R^{m\times m}$ contains only the interior degrees of freedom; each
$v\in V_h$ is extended by zero to the boundary.  Define
\begin{equation}
  \langle v,z\rangle_h=h^2\sum_{k\in\mathcal I_h}v_kz_k,
  \qquad \norm{v}_h=\langle v,v\rangle_h^{1/2}.
  \label{eq:s-rd-inner-product}
\end{equation}
Let $K_h:V_h\to V_h$ be the conservative variable-coefficient Dirichlet
diffusion operator assembled with positive harmonic face coefficients.  Its
edge-sum representation gives
\begin{equation}
  \langle K_hv,z\rangle_h=\langle v,K_hz\rangle_h,
  \qquad \langle K_hv,v\rangle_h>0\quad(v\ne0).
  \label{eq:s-rd-diffusion-spd}
\end{equation}
For $\lambda>0$, $\gamma>0$, and interior forcing $f_h\in V_h$, set
\begin{equation}
  A_h(v)=K_hv+\lambda v+\gamma v^{\odot3}-f_h.
  \label{eq:s-rd-operator}
\end{equation}
The implementation also stores a complete residual with boundary rows.  The
results below use only its interior restriction after verifying exact zero
Dirichlet values.  They do not apply to the nonsymmetric complete matrix that
contains identity boundary rows and interior-to-boundary couplings.

Let $(Y,\langle\cdot,\cdot\rangle_Y)$ be a finite-dimensional real Hilbert task
space and let $Q_h:V_h\to Y$ be linear.  Its weighted adjoint is defined by
\begin{equation}
  \langle Q_hv,z\rangle_Y=\langle v,Q_h^*z\rangle_h
  \quad\text{for every }v\in V_h,\ z\in Y.
  \label{eq:s-rd-weighted-adjoint}
\end{equation}
Thus $Q_h^*$ need not be an unweighted Euclidean transpose. In the computation,
$Y=V_h$ with the same $h^2$-weighted inner product and $Q_h=Q_h^*=I$.

\subsection{Global operator-aligned enclosure}

\begin{lemma}[Cubic secant lower bound]
\label{lem:s-rd-cubic}
For every $a,b\in\mathbb R$,
\begin{equation}
  (a^3-b^3)(a-b)\ge\frac34a^2(a-b)^2.
  \label{eq:s-rd-cubic}
\end{equation}
The constant $3/4$ is sharp.
\end{lemma}

\begin{proof}
If $a=b$, the statement is immediate.  Otherwise divide by $(a-b)^2$ and use
\[
  a^2+ab+b^2=\left(b+\frac a2\right)^2+\frac34a^2.
\]
Equality occurs at $b=-a/2$.
\end{proof}

\begin{theorem}[Global residual enclosure]
\label{thm:s-rd-enclosure}
The equation $A_h(u_h)=0$ has a unique solution $u_h\in V_h$.  For any
$s\in V_h$, define
\begin{equation}
  M_s=K_h+\lambda I+\frac{3\gamma}{4}
  \operatorname{diag}(s^{\odot2}).
  \label{eq:s-rd-metric}
\end{equation}
Then $M_s$ is symmetric positive definite and
\begin{equation}
  \norm{s-u_h}_{M_s}
  \le \norm{A_h(s)}_{M_s^{-1}},
  \label{eq:s-rd-state-bound}
\end{equation}
where
$\norm{v}_{M_s}^2=\langle M_sv,v\rangle_h$ and
$\norm{r}_{M_s^{-1}}^2=\langle r,M_s^{-1}r\rangle_h$.
\end{theorem}

\begin{proof}
The map $A_h$ is the gradient, in the $h^2$-weighted inner product, of the
coercive strictly convex finite-dimensional functional
\[
  \mathcal E(v)=\frac12\langle K_hv,v\rangle_h
  +\frac\lambda2\langle v,v\rangle_h
  +\frac\gamma4h^2\sum_kv_k^4-\langle f_h,v\rangle_h.
\]
It therefore has exactly one zero.  Let $e=s-u_h$.  Symmetry and positivity of
$K_h$, together with Lemma~\ref{lem:s-rd-cubic} applied componentwise, yield
\[
  \langle A_h(s)-A_h(u_h),e\rangle_h
  \ge \langle M_se,e\rangle_h=\norm{e}_{M_s}^2.
\]
Because $A_h(u_h)=0$, Cauchy--Schwarz in the dual pair generated by $M_s$ gives
\[
  \norm{e}_{M_s}^2
  \le\langle A_h(s),e\rangle_h
  \le\norm{A_h(s)}_{M_s^{-1}}\norm{e}_{M_s}.
\]
The result follows after treating $e=0$ separately and otherwise dividing by
$\norm{e}_{M_s}$.
\end{proof}

If $a_{\min}>0$ bounds all diffusion face coefficients from below, then
\begin{equation}
 \lambda_{\min}(M_s)\ge
 \mu_s^{\mathrm{analytic}}
 :=\lambda+a_{\min}\frac{8}{h^2}
 \sin^2\!\left(\frac{\pi}{2(m+1)}\right)
 +\frac{3\gamma}{4}\min_ks_k^2>0.
 \label{eq:s-rd-coercivity}
\end{equation}
The eigenvalues here refer to the coordinate matrix; the scalar $h^2$ in the
inner product does not alter them.

Indeed, the conservative edge-sum representation and the face bound
$a_f\geq a_{\min}$ give, for every $v\in V_h$,
\begin{equation}
  \langle K_hv,v\rangle_h
  \geq a_{\min}\langle-\Delta_hv,v\rangle_h.
\end{equation}
The two-dimensional five-point zero-Dirichlet Laplacian has eigenvalues
\begin{equation}
  \frac{4}{h^2}\left[
  \sin^2\!\left(\frac{p\pi}{2(m+1)}\right)
  +\sin^2\!\left(\frac{q\pi}{2(m+1)}\right)
  \right],
  \qquad 1\leq p,q\leq m.
\end{equation}
Its smallest eigenvalue is the $p=q=1$ value in
Eq.~\eqref{eq:s-rd-coercivity}.  Adding the Rayleigh lower bounds for
$\lambda I$ and
$(3\gamma/4)\operatorname{diag}(s^{\odot2})$ proves the displayed analytic
coercivity constant.

\subsection{Pairwise and top-1 decisions}

Fix task outputs $y_1,\ldots,y_N\in Y$ and define
$R_i(z)=\norm{z-y_i}_Y^2$ and
$G_{ij}(z)=R_j(z)-R_i(z)$.  Let $d_{ij}=y_i-y_j$,
$t=Q_hu_h$, and $q=Q_hs$.

\begin{theorem}[Operator-aligned pair bound]
\label{thm:s-rd-pair}
For every candidate pair,
\begin{align}
 |G_{ij}(t)-G_{ij}(q)|
 &=2\left|\langle u_h-s,Q_h^*d_{ij}\rangle_h\right|,
 \label{eq:s-rd-pair-identity}\\
 &\le
 2\norm{A_h(s)}_{M_s^{-1}}
   \norm{Q_h^*d_{ij}}_{M_s^{-1}}
 =:B_{ij}.
 \label{eq:s-rd-pair-bound}
\end{align}
Consequently, if $|G_{ij}(q)|>B_{ij}$, then $G_{ij}(t)\neq 0$ and has the
same sign as $G_{ij}(q)$.
\end{theorem}

\begin{proof}
Expanding the two squared-distance differences gives
\[
 G_{ij}(t)-G_{ij}(q)
 =2\langle t-q,d_{ij}\rangle_Y
 =2\langle u_h-s,Q_h^*d_{ij}\rangle_h.
\]
Apply Cauchy--Schwarz in the $M_s$ metric and then
Theorem~\ref{thm:s-rd-enclosure}.
\end{proof}

Let $k$ be the lowest-index proxy minimizer and let $\tau\ge0$.

\begin{corollary}[Strict and tolerant top-1 certificates]
\label{cor:s-rd-top1}
Candidate $k$ is the unique exact same-grid discrete-solution minimizer if
\begin{equation}
  R_j(q)-R_k(q)>B_{kj}\qquad\text{for every }j\ne k.
  \label{eq:s-rd-strict-top1}
\end{equation}
It is guaranteed to lie within $\tau$ of the best exact same-grid
discrete-solution loss if
\begin{equation}
  R_k(q)-R_j(q)+B_{kj}\le\tau
  \qquad\text{for every }j\ne k.
  \label{eq:s-rd-tolerant-top1}
\end{equation}
Failure of either condition returns abstention.
\end{corollary}

\begin{proof}
For strict top-1, Theorem~\ref{thm:s-rd-pair} implies
$R_j(t)-R_k(t)\ge R_j(q)-R_k(q)-B_{kj}>0$ for every competitor.  For the
tolerant statement, it gives
$R_k(t)-R_j(t)\le R_k(q)-R_j(q)+B_{kj}\le\tau$ for every $j$, including a
exact same-grid discrete-solution minimizer.
\end{proof}

To evaluate all pair directions, define
\begin{equation}
  \Dcal=\operatorname{span}\{y_i-y_1:2\leq i\leq N\},
  \qquad
  \mathcal B_h=Q_h^*\Dcal,
  \qquad
  r=\dim\mathcal B_h\leq N-1.
\end{equation}
Choose a basis $b_1,\ldots,b_r$ of $\mathcal B_h$ and define
\begin{equation}
  \Gamma_{k\ell}=\langle b_k,M_s^{-1}b_\ell\rangle_h.
  \label{eq:s-rd-gram}
\end{equation}
For each pair, let $c_{ij}\in\Rnum^r$ be the unique coefficient vector with
$Q_h^*d_{ij}=\sum_{\ell=1}^r(c_{ij})_\ell b_\ell$. Then
\begin{equation}
  \norm{Q_h^*d_{ij}}_{M_s^{-1}}^2
  =c_{ij}^{\mathsf T}\Gamma c_{ij}.
  \label{eq:s-rd-gram-pair}
\end{equation}
Thus one residual inverse action and $r$ comparison-basis inverse actions with
the common matrix $M_s$ evaluate all $N(N-1)/2$ bounds. No candidate-specific
correction is required, but the certificate cost is comparison-rank dependent
rather than library-size independent.

\subsection{Transfer from operator truth to dataset truth}

The preceding results certify $t=Q_hu_h$, the exact zero of the named same-grid
operator. Dataset truth requires an additional discrepancy enclosure.
Throughout this subsection, the task metric is the $Y$ metric:
$\metricip{x}{y}=\langle x,y\rangle_Y$, and $P_{\Dcal}^{M}$ is the
corresponding orthogonal projection.

\begin{proposition}[Projected operator--dataset discrepancy]
\label{prop:s-rd-dataset-discrepancy}
Suppose $t_{\rm data}=t+\xi$ and a compact set
$\mathcal U_{\rm disc}\subset\Dcal$ satisfies
$P_{\Dcal}^{M}\xi\in\mathcal U_{\rm disc}$. With
$\sigma_{\mathcal U}(d)=\sup_{z\in\mathcal U}\metricip{z}{d}$,
\begin{align}
  G_{ij}(t_{\rm data})\in\bigl[&G_{ij}(q)-B_{ij}
  -2\sigma_{\mathcal U_{\rm disc}}(-d_{ij}),\notag\\
  &G_{ij}(q)+B_{ij}
  +2\sigma_{\mathcal U_{\rm disc}}(d_{ij})\bigr].
  \label{eq:s-rd-dataset-discrepancy}
\end{align}
If only $\metricnorm{P_{\Dcal}^{M}\xi}\leq\Delta_{\Dcal}$ is known, it
suffices to add $2\Delta_{\Dcal}\metricnorm{d_{ij}}$ to $B_{ij}$.
Discrepancy in $\Dcal^{\perp_M}$ does not change a squared-loss decision.
\end{proposition}

\begin{proof}
The exact gap identity gives
$G_{ij}(t_{\rm data})=G_{ij}(t)+2\metricip{\xi}{d_{ij}}$. Theorem
\ref{thm:s-rd-pair} encloses the first term, while the minimum and maximum of
the second over $\mathcal U_{\rm disc}$ are
$-2\sigma_{\mathcal U_{\rm disc}}(-d_{ij})$ and
$2\sigma_{\mathcal U_{\rm disc}}(d_{ij})$. The norm-ball statement follows
from Cauchy--Schwarz.
\end{proof}

No reference-free $\mathcal U_{\rm disc}$ or $\Delta_{\Dcal}$ was available
for PDEBench or the exploratory flow datasets. The proposition formalizes the
extra information required for transfer; it does not provide that information.

\subsection{Floating-point solve-defect control}

\begin{lemma}[Residual-inflated inverse norm]
\label{lem:s-rd-solve-defect}
Suppose $M:V_h\to V_h$ is self-adjoint and positive definite in
$\langle\cdot,\cdot\rangle_h$, and $\lambda_{\min}(M)\ge\mu>0$.  For any
approximate solution $\widehat z$ of $Mz=b$, with algebraic defect
$\rho=b-M\widehat z$,
\begin{equation}
  \norm{b}_{M^{-1}}
  \le \norm{\widehat z}_M+\frac{\norm{\rho}_h}{\sqrt\mu}.
  \label{eq:s-rd-solve-defect}
\end{equation}
\end{lemma}

\begin{proof}
Write $z=\widehat z+M^{-1}\rho$.  The triangle inequality in the $M$ norm and
$\norm{\rho}_{M^{-1}}\le\norm{\rho}_h/\sqrt\mu$ prove the claim.
\end{proof}

Equation~\eqref{eq:s-rd-solve-defect} was applied to the residual right-hand side
and every comparison-direction combination.  The implementation then added a
small multiple of machine epsilon to the float64 quantities. This does not bound
roundoff; a strict floating-point certificate would require outward rounding or
another validated arithmetic implementation.

\subsection{Numerical certificate margins and same-grid comparison}

The calculation used four reaction--diffusion cells: FNO and CNO, each with
ensembles e1 and e2. Each cell contained eight candidates, and the four cells
reused the same six inputs, giving 24 library--input cases.
The identity task uses full-field squared $h^2$-weighted loss; exact zero boundary
values make this identical to the published tensor-trapezoid loss.

\begin{table}[H]
\centering
\caption{Float64 evaluation of the reaction--diffusion sufficient conditions.
A failed sufficient condition returns abstention.}
\label{tab:s-rd-pretruth}
\small
\begin{tabular}{@{}lr@{}}
\toprule
Quantity & Result \\
\midrule
Library--input cases & 24 \\
Candidate pairs & 672 \\
Pair inequalities satisfied & 672/672 \\
Strict top-1 conditions satisfied & 24/24 \\
Tolerant top-1 conditions satisfied & 24/24 \\
Largest bound/proxy-margin ratio & 0.09494 \\
Median bound/proxy-margin ratio & $3.45\times10^{-4}$ \\
Smallest decision slack & $1.41\times10^{-8}$ \\
\bottomrule
\end{tabular}
\end{table}

For comparison, we computed a high-accuracy numerical approximation to the zero
of the same $128\times128$ interior-grid operator for each input. Its residual ranged from
$6.27\times10^{-14}$ to
$2.26\times10^{-12}$.  This solver uncertainty was recorded but not propagated
through the following float64 comparisons.

\begin{table}[H]
\centering
\caption{Comparison with a high-accuracy numerical approximation to the
same-grid operator zero.}
\label{tab:s-rd-evaluation}
\small
\begin{tabular}{@{}lr@{}}
\toprule
Comparison & Result \\
\midrule
State bounds exceeding approximate-zero difference & 24/24 \\
State bound/observed-error ratio & 1.000018--1.000211 \\
Pair bounds exceeding approximate-zero drift & 672/672 \\
Disagreements among selected pair signs & 0/672 \\
Disagreements among strict top-1 selections & 0/24 \\
Disagreements among tolerant top-1 selections & 0/24 \\
\bottomrule
\end{tabular}
\end{table}

The theorem targets the exact discrete zero, whereas the comparison above uses a
numerical approximation whose uncertainty is not propagated. The projected
fine-grid reference is also not an exact zero of the 128-grid operator:
its operator defect was $6.28\times10^{-4}$--$2.11\times10^{-3}$, while its
discrepancy from the numerical same-grid zero was
$4.13\times10^{-6}$--$9.97\times10^{-6}$ in discrete $L^2$. Transfer from the
discrete operator zero to a dataset or another discretization therefore requires
the discrepancy bounds derived above.
% --- End inlined file: supplement_sections/s7_rd_certificate.tex ---
% --- Begin inlined file: supplement_sections/s8_multimetric_mechanism.tex ---
\section{Multi-objective transfer and proxy mechanism}
\label{sec:s-multimetric}

This section examines multi-objective rescoring and the partial-correction path
using PDEBench inputs 0901--0999, the three released U-Net checkpoints and all 91
predicted future frames. The same candidate predictions, common states and
corrections are reused throughout, so no additional linearized solve is required.

\subsection{One corrected trajectory, multiple deployment objectives}

For each objective $d_k$, metric-matched rescoring uses
\begin{equation}
  S_{i,k}^{\rm matched}=d_k\!\left(v_i,s\right),
  \qquad s=w+\delta_w,
\end{equation}
whereas fixed-score transfer retains the original terminal squared-$L^2$ score.
The same $s$ is used throughout, so changing $k$ requires candidate
rescoring but no additional physics solve.  The two squared-$L^2$ objectives are
exact squared-Hilbert tasks.  Roots, maxima, normalization and the boundary or
spectral reductions fall outside that exact identity and are empirical transfer
tests.

For reproducibility, let $N_t=91$ and
$e_{t,y,x,c}=v_{i,t,y,x,c}-z_{t,y,x,c}$ for the future frames, two channels and
an $N_y\times N_x$ uniform grid, where $z$ is
either the released reference or the proxy used for scoring.  With cell area
$h_xh_y$, the terminal and trajectory squared-$L^2$ metrics are
\begin{align}
 d_{\rm term}&=h_xh_y\sum_{y,x,c}e_{T,y,x,c}^2,\\
 d_{\rm traj}&=\frac{h_xh_y}{N_t}\sum_{t,y,x,c}e_{t,y,x,c}^2.
\end{align}
Define
\begin{equation}
 R_{t,c}=\left(\frac{1}{N_xN_y}\sum_{x,y}e_{t,y,x,c}^2\right)^{1/2},
 \qquad
 Z_{t,c}=\left(\frac{1}{N_xN_y}\sum_{x,y}z_{t,y,x,c}^2\right)^{1/2}.
\end{equation}
The reported RMSE and normalized plug-in are
$\operatorname{mean}_{t,c}R_{t,c}$ and
$\operatorname{mean}_{t,c}(R_{t,c}/Z_{t,c})$.  Thus $Z$ is formed from truth
for hidden loss and from $s$ for proxy scoring.  Conservation and
maximum errors are
\begin{equation}
 \operatorname{mean}_{t,c}\left|\operatorname{mean}_{x,y}e_{t,y,x,c}\right|,
 \qquad
 \operatorname{mean}_{t,c}\max_{x,y}|e_{t,y,x,c}|.
\end{equation}
Boundary RMSE first averages $e^2$ over the top, bottom, left and right grid
rows, with corners counted twice and denominator $2N_x+2N_y$, then takes a root
and averages over $(t,c)$.

For the spectral metrics, an unnormalized two-dimensional discrete Fourier
transform is applied over $(y,x)$.  The executed implementation retains indices
$0\leq k_y<N_y/2$ and $0\leq k_x<N_x/2$, assigns each coefficient to
$r=\lfloor(k_y^2+k_x^2)^{1/2}\rfloor$, and keeps
$r<K=\min(N_y/2,N_x/2)$.  For each shell it computes
\begin{equation}
 R^{\rm F}_{t,r,c}=\frac{L_xL_y}{N_xN_y}
 \left(\sum_{\lfloor|k|\rfloor=r}|\widehat e_{t,k,c}|^2\right)^{1/2},
\end{equation}
then averages over $(t,c,r)$ within the bands $[0,4)$, $[4,12)$ and
$[12,K)$ for the low-, mid- and high-band restricted-quadrant spectral scores.
These are the exact executed reductions. They omit the mixed-sign quadrants and
are therefore not a full radial Fourier norm; no value-level parity with every
version of the official PDEBench metric implementation is claimed.

% --- Begin inlined file: tables/table_pdebench_metric_transfer.tex ---
\begin{table}[H]
\centering
\caption{Secondary task transfer on the fixed public PDEBench library. Pair and top-1 counts use three candidates and 99 inputs. ``Fixed'' reuses the original terminal squared-$L^2$ score; ``matched'' reuses the same corrected trajectory but changes the candidate-to-proxy distance. The ten metrics are correlated objectives, not independent replications.}
\label{tab:s-pdebench-multimetric}
\footnotesize
\begin{tabularx}{\textwidth}{@{}p{0.21\textwidth}Xcc@{}}
\toprule
Metric & Theory scope & Fixed pair / top-1 & Matched pair / top-1 / order \\
\midrule
Terminal squared $L^2$ & Exact squared-Hilbert & 297/297; 99/99 & 297/297; 99/99; 99/99 \\
Trajectory squared $L^2$ & Exact squared-Hilbert & 297/297; 99/99 & 297/297; 99/99; 99/99 \\
RMSE & Nonlinear $L^2$ reduction & 297/297; 99/99 & 297/297; 99/99; 99/99 \\
Normalized RMSE plug-in & Truth-dependent plug-in & 296/297; 99/99 & 297/297; 99/99; 99/99 \\
Conservation error & Linear summary; nonlinear reduction & 297/297; 99/99 & 297/297; 99/99; 99/99 \\
Maximum error & $L^\infty$; outside theorem & 297/297; 99/99 & 293/297; 99/99; 95/99 \\
Boundary RMSE & Linear restriction; nonlinear reduction & 200/297; 99/99 & 296/297; 99/99; 98/99 \\
Low-band restricted spectral score & Fixed restricted-quadrant projection; nonlinear reduction & 297/297; 99/99 & 297/297; 99/99; 99/99 \\
Mid-band restricted spectral score & Fixed restricted-quadrant projection; nonlinear reduction & 289/297; 99/99 & 287/297; 99/99; 89/99 \\
High-band restricted spectral score & Fixed restricted-quadrant projection; nonlinear reduction & 234/297; 69/99 & 280/297; 88/99; 83/99 \\
\midrule
Total & Correlated ten-metric panel & 2,801/2,970; 960/990 & 2,938/2,970; 979/990; 959/990 \\
\bottomrule
\end{tabularx}
\end{table}
% --- End inlined file: tables/table_pdebench_metric_transfer.tex ---

The high-band restricted-quadrant objective supplies the clearest
input-adaptive stress. Its
truth winners were autoregressive on 69 inputs, push-forward-20 on 23 and
one-step on seven, spanning five complete candidate orders.  Metric-matched
rescoring selected the true winner on $88/99$ inputs, compared with $69/99$ for
the fixed terminal score, $23/99$ for objective-matched common-state rescoring
and $7/99$ for raw defect energy.  Across all ten correlated objectives, matched
rescoring gave $2{,}938/2{,}970$ correct pairs, $979/990$ top-1 decisions and
$959/990$ complete orders.  The corrected trajectory itself had no larger
truth error than the best public candidate in all $990$ metric--input cells.
These counts describe one fixed public library; they do not establish ten
independent replications or population-wide dominance of the correction.

The normalized-RMSE row is an explicitly named plug-in stress.  Candidate truth
errors use the truth RMS denominator, while candidate-to-proxy scores use the
proxy RMS denominator.  It is therefore not one fixed truth-independent metric
on both sides, so the Hilbert identity does not apply.

\subsection{Partial correction separates proxy and comparison mechanisms}

The partial-correction analysis evaluates
\begin{equation}
  s_\alpha=w+\alpha\delta_w,
  \qquad
  \alpha\in\{0,0.25,0.5,0.75,1,1.25\},
\end{equation}
using the same $\delta_w$ at every $\alpha$.  Because the executed task maps are
the identity or a fixed terminal trace,
\begin{equation}
 q_{{\rm lin},\alpha}=Q(w)+\alpha DQ(w)\delta_w
 =q_{{\rm eval},\alpha}=Q(s_\alpha)
\end{equation}
algebraically.  This equality is not an independent numerical validation.
Figure~\ref{fig:s-alpha-path} separates the full state, task and
candidate-difference errors from ranking accuracy.

\begin{figure}[H]
  \centering
  \includegraphics[width=0.82\textwidth]{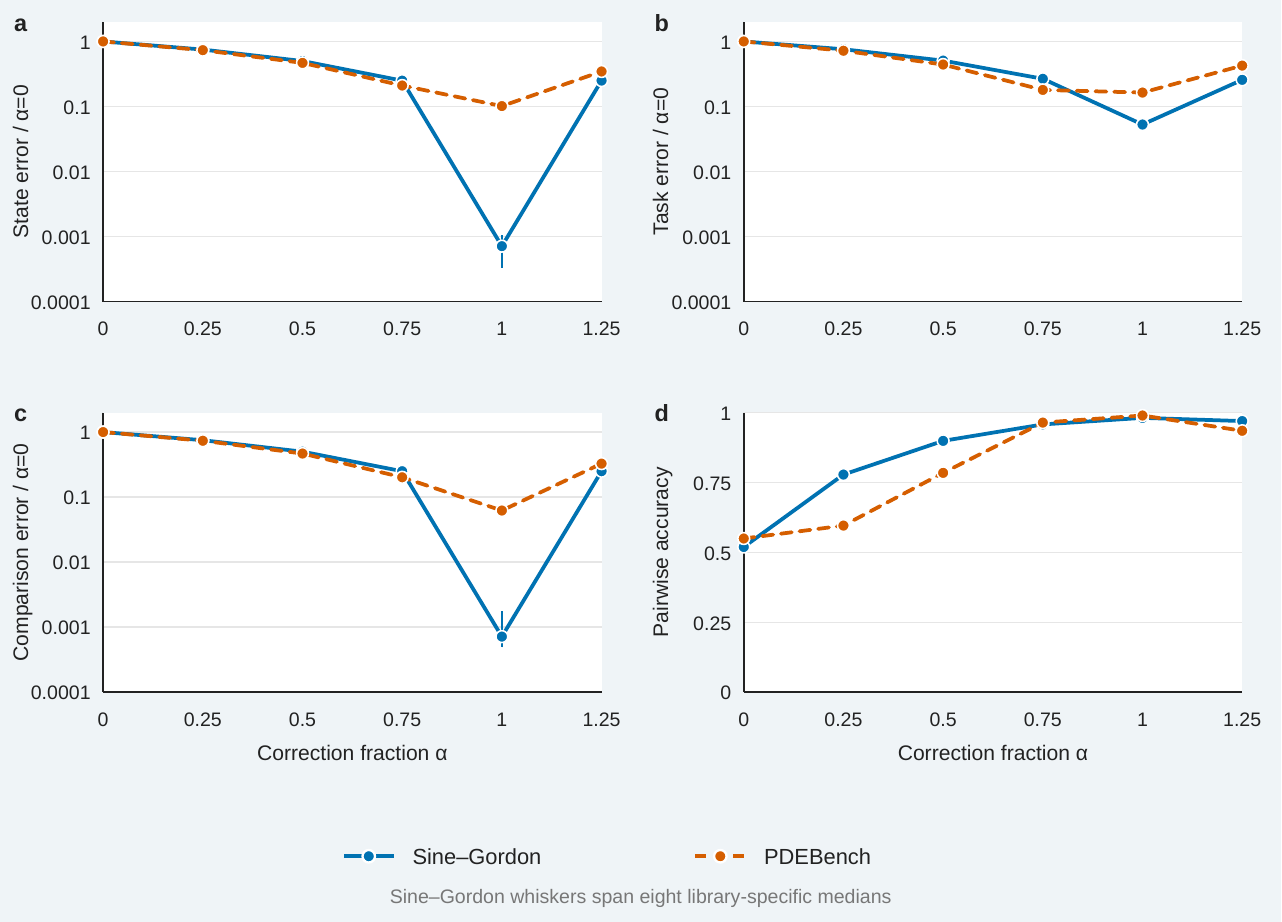}
  \caption{Partial-correction path. Panels a--c normalize each
  library's median error by its value at $\alpha=0$; Sine--Gordon points are the
  medians of eight library-specific medians and whiskers span their range, while
  PDEBench contains one public library.  Comparison error is the orthogonal
  projection of truth--proxy error onto the trajectory candidate-difference
  subspace.  Panel d aggregates exact stored pair counts: one terminal
  squared-$L^2$ task across eight Sine--Gordon libraries (53,760 pairs) and ten
  correlated objectives in PDEBench (2,970 pairs). The path reuses one
  correction.}
  \label{fig:s-alpha-path}
\end{figure}

\begin{figure}[H]
  \centering
  \includegraphics[width=0.92\textwidth]{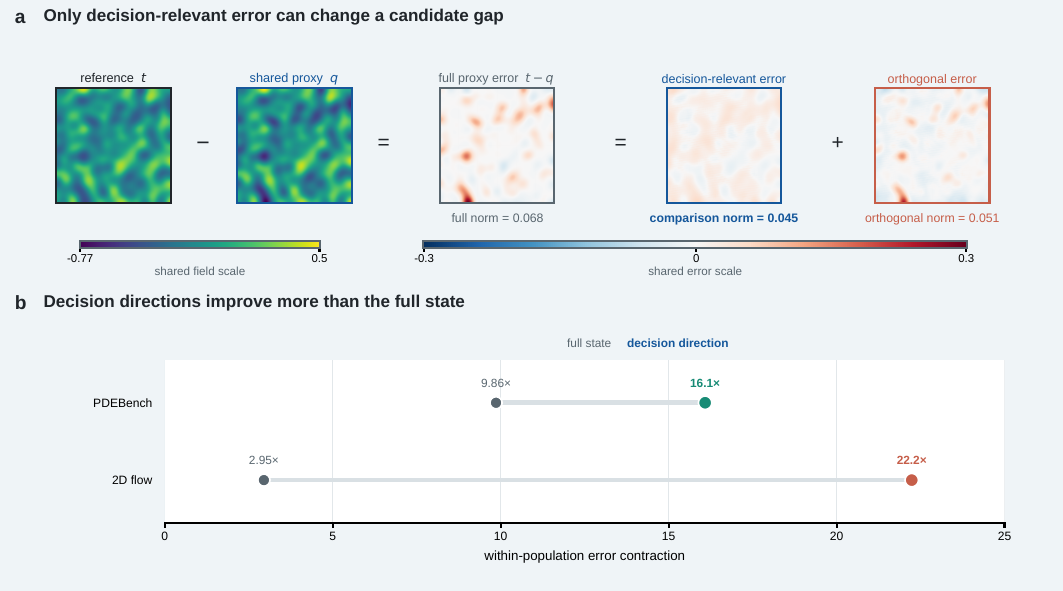}
  \caption{Comparison and reconstruction regimes. \textbf{a},
  Deterministic PDEBench terminal sample 0950 shows the reference, shared proxy
  and exact decomposition of their error. Only
  $e_D=P_{\Dcal}^{M}(t-q)$ changes fixed squared-loss gaps along candidate
  directions $d_{ij}=y_i-y_j$; maps show field 0, while the projection and norms
  use both task fields. \textbf{b}, Paired within-population contraction factors
  show preferential improvement in the decision subspace for PDEBench
  (9.86-fold full-state versus 16.1-fold decision-direction contraction) and for
  the two-model flow study (2.95-fold versus 22.2-fold). In the
  operator-mismatched flow population, comparison was correct in 10/10 cases
  whereas direct use improved on the best candidate in 4/10. PDEBench uses one
  public library and 99 inputs; flow factors are medians of ten casewise ratios.
  Absolute error scales are not compared across PDEs.}
  \label{fig:s-comparison-geometry}
\end{figure}

For Sine--Gordon, the full-trajectory quantities in panels a and c use a
48-point, 48-step trajectory reference on the candidate discretization. The
terminal task and panel-d decisions instead lift candidate and proxy terminal
states to the 192-point terminal reference and use
ordinary strict comparisons.  At $\alpha=1$ this produces
$52{,}728/53{,}760$ correct pairs and $1{,}833/1{,}920$ exact top-1 decisions.
Against that regenerated full-trajectory truth, the corrected proxy had strictly
lower trajectory error than the oracle-best candidate in all $1{,}920$
input--library cases.
These differ from the conservative grid-refinement result of
$52{,}942/53{,}760$ and $1{,}835/1{,}920$ because the path uses the regenerated
48-point trajectory truth and ordinary lifted terminal comparisons, whereas the
grid-refinement result uses a 96-to-192 rule that also marks
numerically indeterminate outcomes as errors.  Neither numerator is expected to
bound the other, and neither set substitutes for the other.

In pooled case-level summaries at $\alpha=1$, median full-state error contracted
by factors of 1,835.7 in Sine--Gordon and 9.86 in PDEBench relative to the
common state.  Panels a--c instead use a library-balanced median of
within-library ratios, so the Sine--Gordon plotted ordinate is not the reciprocal
of the pooled 1,835.7 factor.  Aggregate
pairwise accuracy rose from $51.9\%$ to $98.1\%$ and from $54.9\%$ to $98.9\%$,
respectively.  Overshoot to $\alpha=1.25$ reduced those accuracies to $96.9\%$
and $93.5\%$.  The common candidate library and hidden truth margins are fixed
along this path, so margin size alone cannot explain the improvement.

The two systems nevertheless support different mechanism balances.  In
PDEBench, trajectory error along the candidate-difference subspace contracted by
16.1, compared with 9.86 for the full trajectory, providing additional
comparison-conditioned accuracy.  In Sine--Gordon, the comparison-to-full error
fraction did not decrease for the full trajectory, so its success is primarily a
full-proxy effect.  Exact pair-gap drift identities closed to
$2.11\times10^{-14}$ in the squared-Hilbert tasks.  At $\alpha=1$, each of the
eight Sine--Gordon libraries had a lower median truth margin among flipped pairs
than among correct comparisons, linking the remaining errors to small decision
margins.

The PDEBench contractions use the released trajectory data. The nonlinear
reductions above are metric-specific and fall outside the exact squared-Hilbert
identity.
% --- End inlined file: supplement_sections/s8_multimetric_mechanism.tex ---
% --- Begin inlined file: supplement_sections/s9_variational_residual_baseline.tex ---
\section{Candidate-wise variational residual baselines}
\label{sec:s-rd-variational}

These controls ask whether the weakness of raw residual ranking is primarily a
norm-selection artifact. They use 240 inputs, four
eight-candidate libraries, and two state scenarios (transferred fixed candidates
and checkpoint re-inference).  The resulting 1,920 cases contain 15,360
candidate rows and 53,760 unordered pairs; 42,761 pairs are hidden-reference
non-ties under the stated absolute tolerance. No case or candidate was filtered.

For a candidate state $v_i$, let $A_h(v_i)$ be the discrete residual and define
$H_h=K_h+\lambda I$.  The raw score is
\begin{equation}
  R_{\mathrm{raw},i}^2=\norm{A_h(v_i)}_h^2.
\end{equation}
The common-energy variational score instead uses the stability-compatible dual
norm
\begin{equation}
  R_{H,i}^2=\norm{A_h(v_i)}_{H_h^{-1}}^2,
  \qquad
  \norm{v_i-u_h}_{H_h}\le R_{H,i}.
  \label{eq:s-common-energy-estimator}
\end{equation}
Because every candidate in a case is bounded in the same energy norm, the
resulting upper bounds are directly comparable.  As a sensitivity analysis, a
candidate-adapted construction sets
\begin{equation}
  M_i=H_h+\frac{3\gamma}{4}\operatorname{diag}(v_i^{\odot2}),
  \qquad
  U_i^2=\frac{\norm{A_h(v_i)}_{M_i^{-1}}^2}{\mu_i},
  \label{eq:s-adapted-l2-estimator}
\end{equation}
where the analytic coercivity lower bound $\mu_i$ converts the
candidate-dependent energy enclosure to the common discrete $L^2$ norm.  The
shared comparator remains
\begin{equation}
  S_i=\norm{v_i-(w+\delta_w)}_h^2,
  \qquad DA_h(w)\delta_w=-A_h(w).
\end{equation}

% --- Begin inlined file: tables/table_rd_variational_baseline.tex ---
\begin{table}[H]
\centering
\caption{Candidate-wise residual estimators on the reaction--diffusion
population. Pairwise accuracy conditions on 42,761
hidden-reference non-ties; top-1 denominators contain all 1,920 cases.  Inverse
actions count the linear inverse actions used by the diagnostic after candidate
outputs are available.  Times are median post-inference components, not end-to-end
latencies.}
\label{tab:s-rd-variational}
\footnotesize
\setlength{\tabcolsep}{4pt}
\begin{tabular}{@{}lrrr@{}}
\toprule
Method & Non-tie pairwise & Exact top-1 & Tolerant top-1 \\
\midrule
Raw product residual
  & 25,553/42,761 (59.758\%) & 451/1,920 (23.490\%)
  & 780/1,920 (40.625\%) \\
Common-energy variational
  & 35,371/42,761 (82.718\%) & 1,111/1,920 (57.865\%)
  & 1,430/1,920 (74.479\%) \\
Candidate-adapted $L^2$ majorant
  & 35,369/42,761 (82.713\%) & 1,111/1,920 (57.865\%)
  & 1,430/1,920 (74.479\%) \\
Shared corrected proxy
  & 42,761/42,761 (100\%) & 1,917/1,920 (99.844\%)
  & 1,920/1,920 (100\%) \\
\bottomrule
\end{tabular}

\vspace{0.65em}

\begin{tabular}{@{}lrrr@{}}
\toprule
Method & Inverse actions & Mean regret & Median component time (s) \\
\midrule
Raw product residual & 0 & $4.699\times10^{-4}$ & 0.00339 \\
Common-energy variational & 8 & $6.274\times10^{-6}$ & 0.09246 \\
Candidate-adapted $L^2$ majorant & 8 & $6.274\times10^{-6}$ & 0.17310 \\
Shared corrected proxy & 1 & $3.397\times10^{-12}$ & 0.04057 \\
\bottomrule
\end{tabular}
\end{table}
% --- End inlined file: tables/table_rd_variational_baseline.tex ---

Using a stability-compatible residual norm materially improved selection:
relative to raw residual, the common-energy estimator gained 22.96 percentage
points pairwise and 34.38 points in exact top-1 accuracy, while reducing mean
regret by a factor of 74.9.  The shared corrected proxy nevertheless recovered
all 42,761 non-tie pair decisions and all 1,920 tolerance-optimal choices.  Its
three exact top-1 misses had a maximum regret of $4.19\times10^{-9}$.

The common-energy and candidate-adapted estimators selected the same top-1
candidate in all 1,920 cases and differed on only two of 53,760 pair
preferences.  The adapted construction therefore added matrix setup and almost
doubled its median component time without a decision improvement in this
population.  Each variational comparator requires eight residual inverse
actions per case; the shared comparator uses one library-level linearized
correction.

The variational scores are individual error upper bounds rather than pairwise
certificates.  Without competing lower bounds, selecting the smallest majorant
does not prove that its candidate has the smallest true error.  The comparison
concerns the stationary identity-task reaction--diffusion population. The
reported timings measure post-inference diagnostic components, beginning after
candidate outputs are available.
% --- End inlined file: supplement_sections/s9_variational_residual_baseline.tex ---
% --- Begin inlined file: supplement_sections/s10_ns2d_probe.tex ---
\section{Two-dimensional compressible-flow boundary probe}
\label{sec:s-ns2d}

We examined ten held-out trajectories from one public PDEBench periodic two-dimensional
compressible-flow file at Mach number $0.1$ and viscosities
$\eta=\zeta=0.01$.  The candidate library contains one public ten-frame FNO and one
public ten-frame push-forward-20 U-Net.  Their outputs were evaluated at
$64\times64$ resolution over eleven autoregressive future frames after a shared
ten-frame observation window.

The primary common state was the arithmetic mean after conversion from primitive
variables to density, two momenta and total energy.  The comparison operator used
periodic Rusanov fluxes, density-weighted viscosity and SSPRK2 integration.  Its
finite-difference Jacobian action used relative step $10^{-4}$.  Development
repeats at $10^{-3}$, $10^{-4}$ and $10^{-5}$ gave median full-error improvement
factors 1.970364, 1.970351 and 1.970351, respectively.  The public trajectory is
not a zero of this operator: its median defect relative to the common-state
defect was 0.412. The comparison therefore probes behaviour under operator--data
mismatch.

For each evaluation input, let $e_{\rm full}=\lVert t-q\rVert_M$ and let
$e_{\Dcal}=\lVert P_{\Dcal}^{M}(t-q)\rVert_M$.  Because the library contains two
candidates, $\Dcal$ has rank one. Table~\ref{tab:s-ns2d} reports all ten cases.
Correction improved both errors in every case, but contraction
was much stronger in the candidate-difference direction.  The perpendicular
component improved in only four cases.  The squared-gap identity closed to a
maximum absolute discrepancy of $4.05\times10^{-11}$.

\begin{table}[H]
  \centering
  \caption{Two-dimensional flow comparison and reconstruction. Improvement
  factors are common-state error divided by corrected-proxy error at
  $\alpha=1$.  Ranking contains one binary pair per input.}
  \label{tab:s-ns2d}
  \begin{tabular}{lrr}
    \toprule
    Quantity & Improved/correct & Median factor \\
    \midrule
    Full trajectory error & $10/10$ & 2.951 \\
    Candidate-difference error & $10/10$ & 22.238 \\
    Perpendicular error & $4/10$ & 0.894 \\
    Operator defect & $10/10$ & 32.767 \\
    Common-state top-1 & $6/10$ & --- \\
    Corrected-proxy top-1 & $10/10$ & --- \\
    Corrected proxy better than fixed oracle FNO & $4/10$ & --- \\
    \bottomrule
  \end{tabular}
\end{table}

The admissibility and defect rule selected $\alpha=1$ in every case.
The full-state error was minimized by $\alpha=1$ in eight cases and
$\alpha=0.75$ in two; the exploratory overshoot $\alpha=1.25$ worsened full
error relative to $\alpha=1$ in all ten.  Thus the correction direction and
scale were informative, although direct return of the corrected trajectory was
not uniformly preferable to returning the FNO.

The geometric-medoid sensitivity is also informative. With two
candidates, both medoid objectives are exactly tied, and the tie rule
selects the first candidate, the FNO.  Correction from this center improved full
error in $4/10$ cases (median factor 0.808) and comparison error in $5/10$
(median factor 1.667). The positive primary result is therefore conditional on
the arithmetic-mean common state. Because the FNO is also the hidden-reference
winner on all ten inputs, the result reflects comparison within this two-model
library rather than input-adaptive routing.
% --- End inlined file: supplement_sections/s10_ns2d_probe.tex ---
% --- Begin inlined file: supplement_sections/s11_ns1d_probe.tex ---
\section{One-dimensional compressible-flow validity probe}
\label{sec:s-ns1d}

This study used 35 samples from the official PDEBench held-out split, three
public checkpoints and one public one-dimensional compressible-flow file. The
shared correction used a fixed-substep Rusanov--SSPRK2 comparison operator rather
than the adaptive HLLC data generator, so the experiment also probes
operator--data mismatch.

The arithmetic-mean common state in conservative variables and full correction
$\delta_w$ were computed once per input.
Admissibility required positive density and pressure throughout the rollout.  A
truth-free backtracking rule considered the fixed multipliers
$\{0,0.25,0.5,0.75,1\}$ and selected the largest admissible value whose operator
defect did not exceed the common-state defect.  The selector received only the
grid, common state, correction, positivity limits and operator defect.
Every inadmissible full step remains an error in the strict decision counts.

\begin{table}[H]
  \centering
  \caption{One-dimensional shock admissibility. Error factors are
  common-state full error divided by selected-proxy full error.  The
  $\alpha=1$ row counts every inadmissible unit step as an incorrect decision.}
  \label{tab:s-ns1d}
  \begin{tabular}{lrrrr}
    \toprule
    Readout & Valid & Pairwise & Top-1 & Median error factor \\
    \midrule
    Common state ($\alpha=0$) & $35/35$ & $74.3\%$ & $40.0\%$ & 1.00 \\
    Unit correction ($\alpha=1$) & $22/35$ & $61.9\%$ & $62.9\%$ & 2.47 \\
    Truth-free backtracking & $35/35$ & $94.3\%$ & $91.4\%$ & 2.04 \\
    \bottomrule
  \end{tabular}
\end{table}

The selected multipliers were zero for one case, 0.25 for five, 0.5 for three,
0.75 for four and one for 22.  Backtracking improved full-state error in $33/35$
cases and comparison-subspace error in $34/35$, with median factors 2.04 and
5.20.  Comparison contraction exceeded full-state contraction in $31/35$ cases.
The exact finite-library gap identity closed to maximum relative discrepancy
$2.14\times10^{-15}$, and every observed ranking flip crossed its proxy margin.

One checkpoint was the winner on $33/35$ inputs, and the best fixed selector
reached $94.3\%$ top-1,
slightly above backtracking's $91.4\%$.  The median released-truth defect was
$0.102$ times the common-state operator defect, quantifying rather than removing
the operator mismatch.  Strong-shock conclusions therefore require an
operator-aligned generator study beyond this paper.
% --- End inlined file: supplement_sections/s11_ns1d_probe.tex ---